\documentclass[11pt]{article}
\usepackage{marginnote}
\usepackage{graphbox} 
\usepackage{textcomp}

\usepackage{amsmath,amssymb, enumerate, fullpage}
\usepackage{graphicx,color}
\usepackage[T1]{fontenc}
\usepackage[utf8]{inputenc}
\usepackage{caption}
\usepackage[tight,footnotesize]{subfigure}
\usepackage[dvipsnames]{xcolor}
\usepackage{amsmath}
\usepackage{amsthm}
\usepackage{multirow}
\newtheorem{theorem}{Theorem}[section]

\usepackage{verbatim}
\usepackage{epstopdf}
\usepackage{epsfig} 
\usepackage{bm}
\usepackage{rotating}
\usepackage{array, makecell, multirow}
\usepackage{tabularx,booktabs}
\usepackage{url} 
\usepackage[colorinlistoftodos]{todonotes}

\usepackage{lineno}

\usepackage{amsfonts, amsthm}
\usepackage{blindtext, appendix}
\usepackage[ruled,vlined]{algorithm2e}
\usepackage{xcolor, soul}
\usepackage{algorithmic}

\usepackage{graphicx,booktabs}

\usepackage{array,longtable,calc}

\begin{document}
\title{Counting Objects by Diffused Index: \\
geometry-free and training-free approach}

\author{
Mengyi Tang
\thanks{tangmengyi@gatech.edu,
School of Mathematics, Georgia Institute of Technology
686 Cherry Street, Atlanta, GA 30332 USA.}
\and 
 Maryam Yashtini 
\thanks{my496@georgetown.edu,
Department of Mathematics and Statistics, Georgetown University,
327A St. Mary's Hall, 37th and O Streets, N.W., Washington D.C. 20057 . (The corresponding Author)}
\and Sung Ha Kang 
\thanks{kang@math.gatech.edu,
School of Mathematics, Georgia Institute of Technology
686 Cherry Street, Atlanta, GA 30332 USA. Kang's research was supported in part by Simons Foundation Collaboration Grants 584960.}}

\maketitle

\begin{abstract}
Counting objects is a fundamental but challenging problem.  In this paper, we propose diffusion-based, geometry-free, and learning-free methodologies to count the number of objects in images.  The main idea is to represent each object by a unique index value regardless of its intensity or size, and to simply count the number of index values.    
First, we place different vectors, refer to as  seed vectors, uniformly throughout the mask image. The mask image has boundary information of the objects to be counted.  Secondly, the seeds are diffused using an edge-weighted harmonic variational optimization model within each object. We propose an efficient algorithm based on an operator splitting approach and alternating direction minimization method, and theoretical analysis of this algorithm is given. 
An optimal solution of the model is obtained when the distributed seeds are completely diffused such that there is a unique intensity within each object, which we refer to as an index.   For computational efficiency, we stop the diffusion process before a full convergence, and propose to cluster these diffused index values.  We refer to this approach as Counting Objects by Diffused Index (CODI).
We explore scalar and multi-dimensional seed vectors.  For Scalar seeds, we use Gaussian fitting in histogram to count, while for vector seeds, we exploit a high-dimensional clustering method for the final step of counting via clustering.  
The proposed method is flexible even if the boundary of the object is not clear nor fully enclosed.  We present counting results in various applications such as biological cells, agriculture, concert crowd, and transportation.  Some comparisons with existing methods are presented.  
\end{abstract}

\textbf{NOTE:} Technical details and codes can be found at  https://github.gatech.edu/skang66/CODI

\section{Introduction}
Counting object is an important problem in various applications such as  biological cells \cite{chen1999automatic, drury2011endometrial, kolhatkar2016detection,lu2018class,marsden2018people}, production line items \cite{baygin2018image}, vehicles \cite{lu2018class},  plant organs \cite{ayalew2020unsupervised}, 
animals \cite{marsden2018people},  crowd  \cite{marsden2018people} counting and others. 
In literature, various different approaches have been explored. 
Studies such as watershed \cite{tulsani2013segmentation, chourasiya2014automatic} and floodfill \cite{guo2013method, chen1999automatic} consider cases where the objects to be counted have uniform intensity, similar shapes and sizes, and are disconnected from each other by distinct background color.  
For these classical techniques, the counting results are highly dependent on the quality of segmentation result of a given image.  Utilizing geometrical features of objects can be useful in such cases. Hough Transform is often implemented to segment objects with a similar circular shape \cite{baygin2018image, maitra2012detection,  venkatalakshmi2013automatic}, and aid the segmentation stage.
If the objects have overlapping boundaries, more  preprocessing is required.
For instance, in \cite{berge2011improved}, the authors split the blood cell clumps by finding the maximum curvature on object boundaries and use Delaunay triangulation.
In \cite{kothari2009automated}, the authors first detect concavity at the edge of a cluster to find the points of overlaps between two nuclei, then use the ellipse-fitting technique for segmentation. 
There are other detection oriented segmentation methods, such as, integrating representative \cite{FLCS_count05},  hough transform technique in detection \cite{ GL_count09, MM_count09, detect-hough} and principle component analysis combined with histogram processing \cite{loukas2003image}.

In some recent studies, a density map from an image patch is learned by extracting global features such as texture, gradient, edge features, or local features, then regression technique is used for counting \cite{crowd_dense_est13,   onoro2016towards, marsden2018people, xie2018microscopy, wang2016fast}. 
Based on extraction of these meaningful features,  integration is done on the density function over any subdomain of image, based on their dependence among neighboring patches or on the whole image to give an estimate 
count \cite{interactivecount,  CNNboosting, MCNN,ranjan2018iterative}.
The performance of density-based methods highly depends on the types of features used.
Many methods use manually crafted features to improve the segmentation of objects and background together with a learning a density map \cite{countingwild, cholakkal2019object}.  
There are network methods that implement regression directly on a given image without retrieving a density map to give a count.  In \cite{ding2020classification}, the author formulates the counting task as an image classification problem and takes the counts as class labels.  In \cite{deeppeaple}, a convolution network regression model is learned on extremely dense crowds to directly give a count for a crowd sample.  
These methods are able to handle a large quantity of various objects, but a corresponding database with ground truth and a training process is required.

In this paper, we propose a diffusion-based geometry-free and training-free counting method.  The main idea is to give a unique index to each object regardless of its intensity or size, and to simply count the number of indexes.    
First, we place different-value vectors, i.e. seed vectors, uniformly through out the given image.   
The seed values are independent from requiring precise prior knowledge about the image and objects to be counted. 
Secondly, these seed vectors are diffused using an edge-weighted harmonic variational optimization model to give a unique index to each object.  Our edge-weighted harmonic variational optimization model  is motivated by \cite{kang2007variational,yashtini2019efficient} which was used for color image inpainting \cite{MaryamKangSIIM}. 
Inspired by recent developments on solving structured optimization models 
\cite{DistribOpt, chy13, hnyz15, hestenes1969multiplier, WT10,YangZhangYin2010, YKssvm15, MaryamKangSIIM,yashtini2019efficient}, we exploit variable splitting, alternating direction method of multipliers, as well as periodic boundary condition to develop a fast algorithm to solve this optimization model.  We refer this part as Diffusion Algorithm.
 An optimal solution of the model is reached when the uniformly distributed seeds are diffused and reached different gray-level intensities. At this point, each object in the image has a unique index.
 For efficiency and more flexibility, we cluster the index values of each pixel before the Diffusion Algorithm is fully converged. We investigate both scalar and multi-dimensional seed vectors.  For scalar seed vectors,  we count the number of peaks in the Gaussian fitted curve of the histogram.  For multi-dimensional seed vectors, we use high dimensional density based clustering algorithm. 
 The main contribution of this paper is outlined below:
 \begin{enumerate}
 \item{We introduce new simple geometry-free and training-free counting methodologies.}
 \item{We propose a fast diffusion algorithm and  establish its theoretical analysis. }
 \item{We provide numerical experiments and comparison on various applications.  We present results for counting objects without clear or closed boundaries, and propose a simple extension to counting different size objects separately.  }
 \end{enumerate}

The rest of this paper is organized as follows.   In Section \ref{sec:method}, we present the proposed methodologies.  In Section  \ref{sec:property} and \ref{sec:experiment}, we provide more insight into the method and present various numerical examples and comparisons.
Some concluding results and remarks are given in Section \ref{sec:concluding}.

\section{Counting Objects by Diffused Index}\label{sec:method}
Let us consider a given image in which there are objects to count. We aim to give a unique index to each object regardless of its intensity, shape or size, then we simply count the number of indexes to provide the quantity of the objects.  There are three simple steps to this method: 
\begin{itemize}
\item{\textbf{[Step 1]} Place different gray-value seeds uniformly though out the given image;}
\item{\textbf{[Step 2]} Diffuse the seeds to obtain different index values within each object; }
\item{\textbf{[Step 3]} Counting the different indexes to obtain the number of objects. We can further cluster objects based on their size. }
\end{itemize}
Outline of the proposed method is presented in Figure \ref{f:outline}.
Based on the given image,  in [Step 1]  we put uniformly distributed seed onto a corresponding mask image.  We choose the seeds to be all different from each other.   In [Step 2], the seeds, whether scalar or multi-dimensional, are diffused within each object.  The diffusion process is done by an iterative algorithm where after the decay rate reaches to a certain level,  each object is reached to a different gray-intensity value. 
 This is shown in [Step 3] via histogram of the diffused image. 
 Each object is associated to a peak in the histogram.  In [Step 3], we provide two counting methods for scalar and vectorial seeds. 

 \begin{figure}
\begin{center}
\begin{tabular}{c} 
\includegraphics[width = 0.9\textwidth]{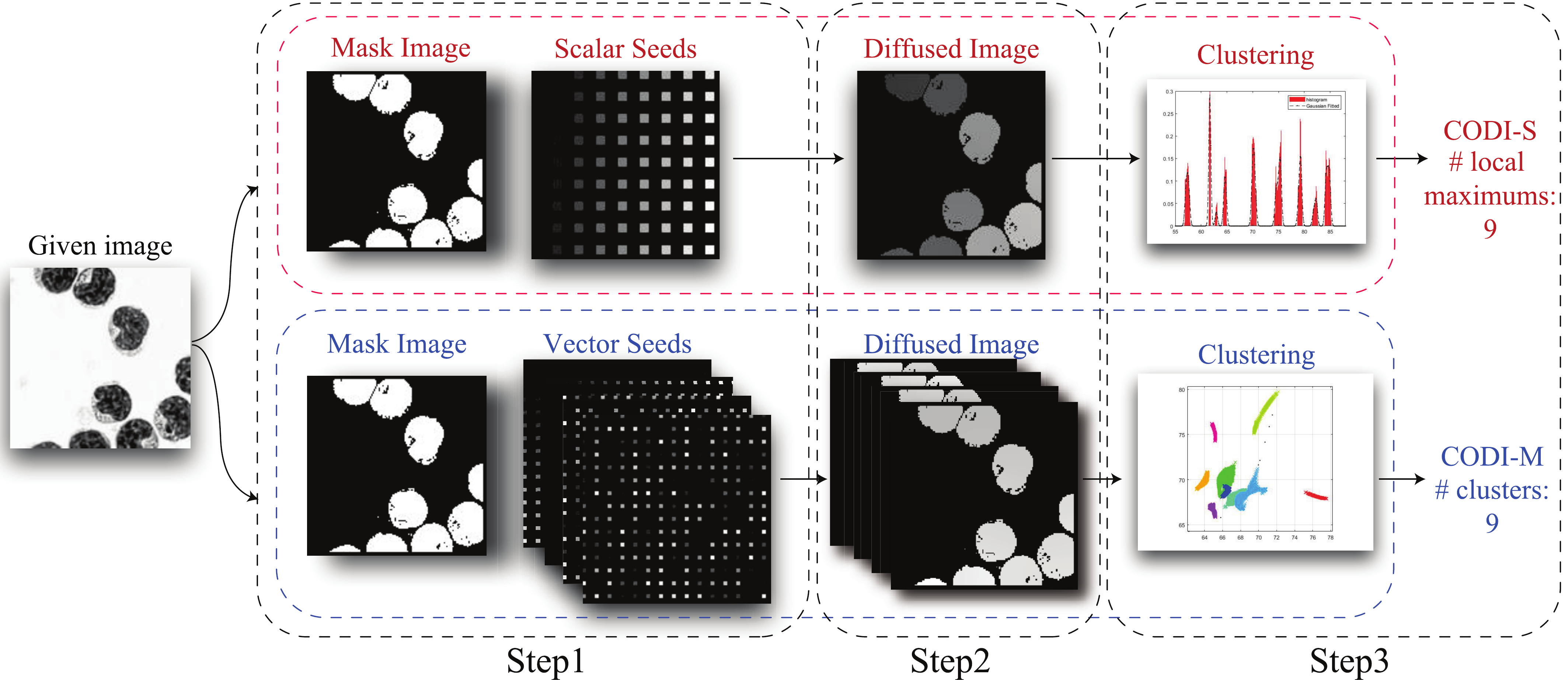}
\end{tabular}
\end{center}
 \caption{Outline of two  proposed counting methods (scalar or vectorial seeds). Given image with 9 cells.    
 \textbf{[Step 1]} Uniformly distributed seed (Scalar or multi-dimensional seeds).
 \textbf{[Step 2]} Diffusion of seeds to find unique index for each object.  
\textbf{[Step 3]} Counting stage: the number of indexes is counted using clustering methods. 
Both methods give 9 objects. }
\label{f:outline}
\end{figure}

\subsection{Ingredients: seed, mask, and edge images [Step 1]}
\label{seed}

Let $\Omega\subset \mathbb R^2$ be the image domain with Lipschitz boundary
and $\Phi_0:\Omega\to\mathbb R$ be the given image.  We place different gray-value seeds through out the given image.  
Let $U_0:\Omega\to [0,255]$ denote the  seed image with $M$ different seeds
$s_{i,j}: \Omega_{i,j} \to v_{i,j}$, where 
$\Omega_{i,j}\subset \Omega$, $i=1,2,\dots,n_1$, $j = 1,2, \dots, n_2$, $ M = n_1n_2$,
and $n_1,n_2\in\mathbb N$. 
We explore both scalar value seed as $ v_{i,j} \in (0,255]$ and multi-dimensional seed as  $ v_{i,j} \in \mathrm{R}^N$.
For multi-dimension seeds, we use superscript to represent each dimension, e.g., $U_0^j$ with $j=1,2,\dots,N$.   
Different seeds are placed on a small region $\Omega_{i,j}\subset \Omega$ such that $D=\cup_{i=1}^{n_1} \cup_{j=1}^{n_2} \Omega_{i,j}\subset \Omega$, and $D^c=\Omega\backslash D\subset \Omega$.  
In practice, $\Omega_{i,j}$ are considered to be square shape, all with the same size, and dimension $d \times d$, and the distance between two adjacent seeds to be $l$.
Outside of the seeded region $D^c$, $U_0$ is set to be zero. 
For scalar seeds, we set a constant gray-scale values $v_{i,j} \in (0,255]$ on each seed domain $\Omega_{i,j}$. Thus, the seed image $U_0$ has $M+1$ gray values $\{0,v_{1,1},\dots, v_{n_1,n_2}\}$ such that for any $x\in\Omega$, ${U}_0(x)= v_{i,j}$ if $x \in \Omega_{i,j}$, $i = 1,2,\cdots, n_1$, $j = 1,2,\cdots, n_2$, and ${U}_0(x)= 0$ otherwise.  Typically, we picked  $v_{i,j} = \frac{255}{M}[(i-1)n_2+j]$ for $i = 1,2,\cdots, n_1, j = 1,2,\cdots, n_2$ as uniformly distributed value in $(0,255]$.

Depending on the image and the objects, to stabilize the small separation between objects and to avoid having the same index for different objects, we also utilize multi-dimensional seeds.
Figure  \ref{F: seed distribution} shows a multi-dimensional seed,
where  each seed dimension is shown separately in (a)-(d).
In the first dimension, we increase the seed values in $x$-direction (horizontally) then $y-$direction (vertically) such that the lowest value is located on the upper-left corner and highest value is on the bottom-right corner. 
This is identical to the scalar seeds, i.e., $U_0^1 = U_0$.   In the second dimension, we start with the bottom-left corner, increase the values in $y$-direction first then increase in $x$-direction, where the lowest gray-value is on the bottom-left corner while the highest gray value  is on the uppper-left corner:  ${U}_0^2(x)= v_{i,j}$ if $x \in \Omega_{i,j}, \; i=1,\dots,M$ where seeds are assigned in the same logic: $v_{i,j} = \frac{255}{M}[n_1n_2 -in_2 + j ]$ for $i = 1,2,\cdots, n_1, j = 1,2,\cdots, n_2$. 
We add two additional dimensions with random seeds given by  a random permutation of the set $\{ v_1,...,v_M\}$.   The values $v_i$ of seeds in different dimensions are identical, but the order of placement is different in each dimension.  
We recommend $p\geq 3$ for multi-dimensional seeds, where $p$ denotes the number of seed dimension, i.e. 
$U_0=[U_0^1,\dots, U_0^p]\in\mathbb R^p$. Through out this paper, we consider $p=4$.

\begin{figure}
\centering
\begin{tabular}{cccc} 
\centering
(a)  &  (b) & (c)  & (d) \\ 
\includegraphics[align = c, height=0.9in]{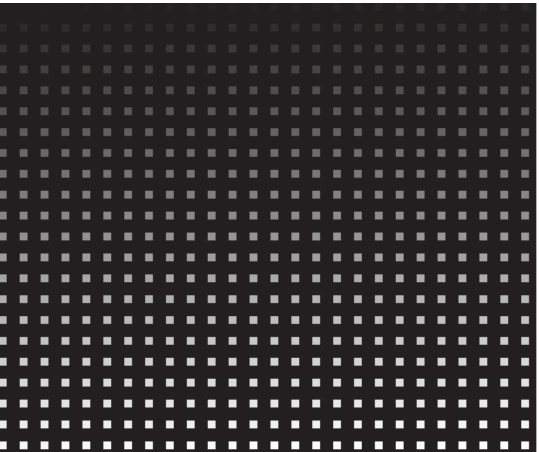}&
\includegraphics[align = c, height=0.9in]{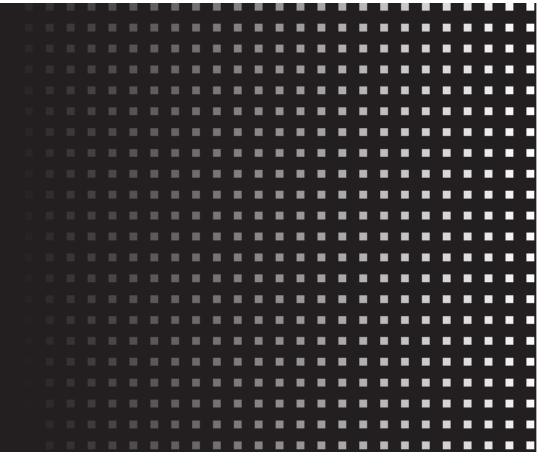}& 
\includegraphics[align = c, height=0.9in]{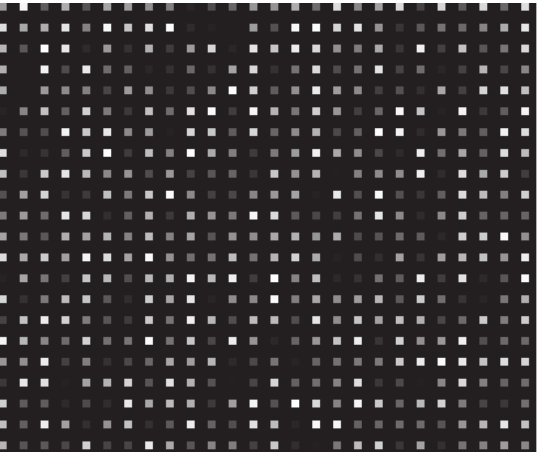} &
\includegraphics[align = c, height=0.9in]{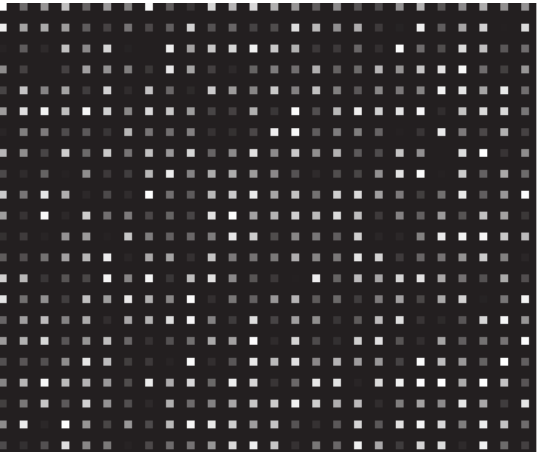}\\
\hspace{0.1in}
\end{tabular}\\
\begin{tabular}{cccccc}
\centering
(e) & (f) $\bar{g}$ & (g) & (h)$\mathcal{M}$ & (i) & (j) $\bar{g} \cdot \mathcal{M}$\\
\includegraphics[align = c,height=1in]{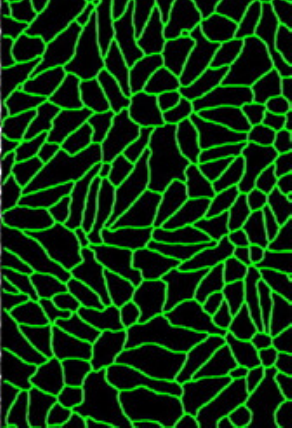}
&\includegraphics[align = c,height=1in]{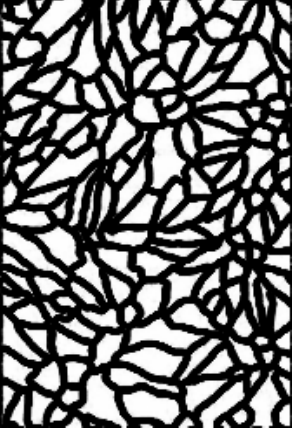}
&\includegraphics[align = c,height=1in]{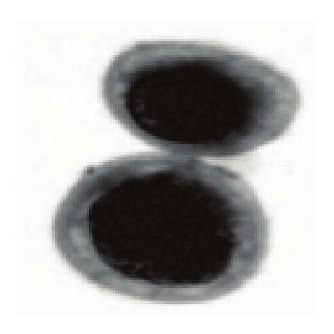}
&\includegraphics[align = c,height=1in]{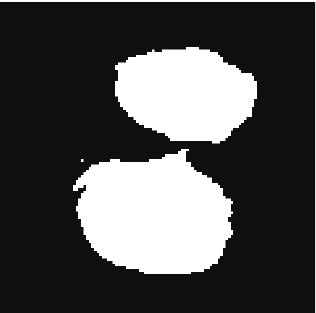} 
&\includegraphics[align = c,height=1in]{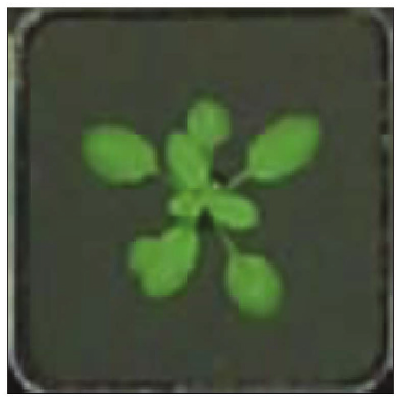} 
& \includegraphics[align = c,height=1in]{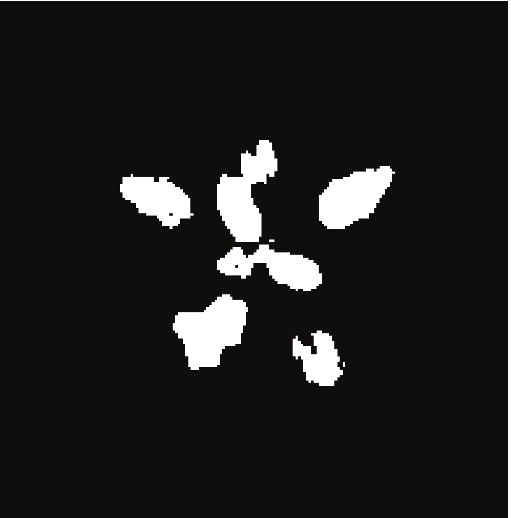} 
\\
\end{tabular}
\caption{[4 dimensional seed, edge function and mask image] (a)-(d) shows an example of multi-dimensional seed. 
(a) $U_0^1$ (horizontal), (b)  $U_0^2$ (vertical), (c) $U_0^3$ (random) (d) $U_0^4$ (random).   (e), (g) and (i) are three given images, (f), (h) and (j) show the corresponding edge function $\bar{g}$, a mask image $\mathcal M$ and a mask and edge function $\mathcal M  \cdot \bar{g}$ used for each image respectively.}
\label{F: seed distribution}
\end{figure}

The most important geometric features of any image are the edges.
When objects in a given image $\Phi_0$ are separated by edges, 
we define a continuous monotone decreasing function $\hat{g}(t): \mathbb R\to [0,1]$ 
such that $\hat{g}(|\nabla \Phi_0|)$ gives the edge information. 
Here $\nabla$ denotes the gradient operator and $|\cdot|$ represents the $\ell_2$ norm. 
Some examples of $\hat{g}$ includes 
\begin{equation}
\tilde{g}(t) = e^{-\tau t^2},\;\;\text{ and } \bar{g}(t) = (1+\tau t^2)^{-1}
\label{e:g}
\end{equation}
with $\tau>0$. 
With the monotone decreasing property of $\hat{g}(t)$ with respect to $t$, the diffusion process stops close to the object boundaries. To eliminate the coarse boundary features and remove noise,  we consider $g(\Phi_0) = \hat g(|\nabla (G_{\sigma}* \Phi_0)|)$,
where $G_\sigma$ denotes the two dimensional Gaussian function with the variance $\sigma$.

If  objects are separated with a different  background color,
we utilize a mask image $\mathcal M:\Omega\to [0,1]$ that is a binary image having zero values on the background and one on the objects. In this case, we set $g(\Phi_0)=\mathcal M$ in the model (\ref{obj:EWH}). Both  $\hat{g}$ and $\mathcal{M}$ can be considered at the same time when it is required to distinguish the objects from background as well as edges from objects. Figure \ref{F: seed distribution} gives three examples of using edge function and mask images, where  $\hat{g}$, $\mathcal M$ and $\mathcal M  \cdot \hat{g}$ are suggested for image (e), (g) and (i) respectively.

\subsection{Diffusion Phase [Step 2]}\label{ssec: diffusion process}

The weighted harmonic variational diffusion model is given by
\begin{eqnarray}\label{obj:EWH}
&\displaystyle{\min_U\;\big\{ F_{\eta}[U] \big| \;\; U\in BV(\Omega;\mathbb R^2) \quad a\le U(x)\le b \big\}} , \text{ where  }&\\[.1in]
&\displaystyle{F_{\eta}[U]=\int_{\Omega} g(\Phi_0)|\nabla U|^2 dx+\frac{\eta}{2}
\int_{D^c \cap \mathcal{M}} |U-U_0|^2 dx}. & \nonumber
\end{eqnarray}
The first term is the regularization term and the second term is the data fidelity term,
$\eta>0$ is the fidelity parameter, $0<a<b<255$, $\nabla U$ denotes the gradient 
of the image $U$ defined by $\nabla U:=(\partial_x U,\partial_y U)$,
where $\partial_x$ and $\partial_y$ are the partial derivatives  along the horizontal and vertical directions, and the function $g$ is described near (\ref{e:g}) in subsection \ref{seed}. 
The parameter $\eta$ in the fidelity term enforces the solution to stay close
to the seed image $U_0$ on the regions $\Omega_{i,j},$ $i=1,\dots,n_1$, $j=1,\dots,n_2$.
To obtain the unique index in each object, $\eta$ should not chosen too large.

We propose the diffusion algorithm to solve (\ref{obj:EWH}) efficiently,
by exploiting  variable splitting and alternating direction method of multiplier 
\cite{DistribOpt, hestenes1969multiplier, WT10,YangZhangYin2010, MaryamKangSIIM,yashtini2019efficient}.
We let  the auxiliary variable ${V}\in L^2(\Omega;\mathbb R)$ and 
we define $\Gamma :=\{ V\in L^2(\Omega;\mathbb R)| a\le V(x)\le b\}$.
We rewrite (\ref{obj:EWH}) equivalently as the following constrained optimization problem
\begin{align}\label{obj:cp}
&\displaystyle{
\min_{{ U}, {V}} \Big\{
\int_{\Omega}  g(\Phi_0)|\nabla{ U}|^2 \; dx
+\frac{\eta_D}{2} \int_{\Omega}  |{V}-{U}_0|^2} dx\Big\}& \\
&\textrm{subject to} \;\;\;  { V}={ U} \;\;{\rm and}\;\;  V \in\Gamma,&\nonumber
\end{align}
where $\eta_D:\Omega\to (0,+\infty)$ is given by  
\begin{eqnarray}
\eta_{D}(x) =
\left\{
\begin{array}{cc}
0, &x\in D\\
\eta, &\;\;\;x\in D^c \cap \mathcal{M}.
\end{array}
\right.
\end{eqnarray}
We let  ${\lambda}$ be the Lagrange multiplier associated with the 
linear constraint $V-U=0$. The augmented Lagrangian functional associated to (\ref{obj:cp}) is given
by
\begin{eqnarray*}
\mathcal L_{\mu}({U}, {V}, {\lambda})=
\displaystyle{
\int_{\Omega}
\Big\{
g(\Phi_0)|\nabla{U}|^2 
+\frac{\eta_D}{2} |{ V}-{U}_0|^2 }
+  \langle {\lambda},  {V-U}\rangle +
\frac{\mu}{2}  |{V-U}|^2 
+ \chi_{\Gamma}({V}) 
\Big\}dx, 
\end{eqnarray*}
where $\mu>0$ penalty parameter, $\chi_{\Gamma}( V)$
is the indicator function given by 
\[
\chi_{\Gamma}({ V}):=
\left\{
\begin{array}{cc}
0, & \quad {V}\in\Gamma\\
+\infty,& {\rm otherwise}.
\end{array}
\right.
\]

The algorithm to solve (\ref{obj:cp}) is given as follows. 
We initially set $k=0$, and let ${V}^{(0)}=0$ and ${\lambda}^{(0)}=0$ be the initial values.
For any $k\ge 1$, given ${V}^{(k)}$ and ${\lambda}^{(k)}$, we compute ${ U}^{(k+1)}$
by solving 
\[
{U}^{(k+1)}=\arg\min_{ U} \mathcal L({U}, { V}^{(k)}, {\lambda}^{(k)}).
\]
More precisely, we compute ${U}^{(k+1)}$ by solving 
\begin{eqnarray}\label{usub-orig}
\min_{U}
\int_{\Omega} \Big\{ g(\Phi_0) |\nabla{ U}|^2 + 
\langle {\lambda},  { V}^{(k)}-{U}\rangle+\frac{\mu}{2} |{ V}^{(k)}-{ U}|^2\Big\} dx.
\end{eqnarray}
To find the close-form solution and to encourage a fast diffusion,
we modify this energy functional as follows.
We let $G_0=\max\big\{g(\Phi_0(x))\;\big| {x}\in\Omega\big\}$,
$\mathcal H({ U})=\big(g(\Phi_0(x))-G_0\big) |\nabla { U}|^2$, and 
write $g(\Phi_0) |\nabla { U}|^2 =
G_0  |\nabla {U}|^2 + \mathcal H({ U}). $
We exploit the second order Taylor polynomial approximation of  $\mathcal H(U)$ about  
${ U}^{(k)}$ to get
\begin{eqnarray*}
\mathcal H({U})&\approx& 
\mathcal H({U}^{(k)})+
\Big\langle \nabla \mathcal H({U}^{(k)}),{U}-{U}^{(k)} \Big\rangle
+\frac{\theta}{2}|{U}-{U}^{(k)}|^2\\
&=&\Big(g(\Phi_0)-G_0\Big) |\nabla {U}^{(k)}|^2
+\Big\langle 2 \nabla \cdot \Big((G_0-g(\Phi_0) \nabla {U}^{(k)}\Big),
{U}-{ U}^{(k)}\Big\rangle
+ \frac{\theta}{2}|{ U}-{U}^{(k)}|^2,
\end{eqnarray*}
where $\theta>0$ is a scalar.
With this approximation, the $U$-minimization subproblem (\ref{usub-orig}) becomes
\begin{eqnarray*}
{U}^{(k+1)} &=& 
\displaystyle{
 \arg\min_{U} \;\; 
 \int_{\Omega}
G_0  |\nabla {U}|^2 \;dx 
+ \int_{\Omega} \Big\langle 2 \nabla \cdot \Big(G_0-g(\Phi_0)\nabla {U}^{(k)}\Big) ,{U} \Big\rangle \;dx}
 \\
&&
\displaystyle{+ \frac{\theta}{2} \int_{\Omega} |{U}-{U}^{(k)} |^2 dx
+ \frac{\mu}{2} \int_{\Omega} |{U}-{V}^{(k)} - \mu^{-1}{\lambda}^{(k)}|^2 dx }.\nonumber
\end{eqnarray*}
The first-order optimality conditions of this problem is  given by
\begin{eqnarray*}\label{Eq:uu2}
\big((\theta+\mu)\mathcal I-2G_0\Delta\big) {U}^{(k+1)}
=\theta {U}^{(k)}+
2 \nabla \cdot \Big((g(\Phi_0)-G_0) \nabla {U}^{(k)}\Big)
+\mu{V}^{(k)} + {\lambda}^{(k)}.
\end{eqnarray*}
We exploit the Fast Fourier Transform (FFT) to obtain the closed form solution of this problem.
Since $\mathcal F\mathcal F^{-1}=\mathcal I$ we then obtain 
\begin{eqnarray*}\label{Eq:uu22}
{U}^{(k+1)}
=\mathcal F^{-1}\Big[
\mathcal F\Big(\theta {U}^{(k)}+
2 \nabla \cdot \Big((g(\Phi_0)-G_0) \nabla {U}^{(k)})
+\mu{V}^{(k)} + {\lambda}^{(k)}\Big)/\mathcal D\Big],
\end{eqnarray*}
where $\mathcal D=(\theta +\mu)\mathcal I-2G_0\mathcal F(\Delta)\mathcal F^{-1}$.

%
Next, we compute ${V}^{({k+1})}$ given ${U}^{(k+1)}$ and ${\lambda}^{(k)}$  by solving 
\[
{V}^{(k+1)}=\arg\min_{V}
\mathcal L({U}^{(k+1)}, {V}, {\lambda}^{(k)}),\;\;\;{\rm where }
\]
\[
\mathcal L({U}^{(k+1)}, {V}, {\lambda}^{(k)})
=\int_{\Omega} \Big\{\frac{\eta_D}{2}  |{ V}-{U}_0|^2 
+  \langle {\lambda}^{(k)},  {V-U^{(k+1)}}\rangle +
\frac{\mu}{2}  |{V-U^{(k+1)}}|^2 
+ \chi_{\Gamma}({V}) 
\Big\}dx.
\]
The objective function is the sum of a qudratic functional and an indicator function,
so the solution is given in a closed form in form of projection as follows
\[
{V}^{(k+1)}(x)={\rm Proj}_{\Gamma} (\gamma),
\quad {\rm where} \quad \gamma=
\frac{\eta_D{ U}_0(x)+ \mu{U}^{(k+1)}(x)-{ \lambda}^{(k)}(x) }{\eta_D (x)+ \mu}.
\]
By the definition of $\Gamma$, then we have
\begin{eqnarray*}
{V}^{(k+1)}(x)
=\left\{
\begin{array}{cc}
a &  \;\;  \gamma \le a\\
\gamma & a\le \gamma \le b\\
b & \gamma \ge b
\end{array}.
\right.
\end{eqnarray*}

Finally, we update the multiplier ${\bf \lambda}^{(k+1)}$ by 
\[
{\bf \lambda}^{(k+1)}={\bf \lambda}^{(k)}+\mu ({V}^{(k+1)} - {U}^{(k+1)}).
\]
This algorithm is referred as Diffusion Algorithm, summarized in Algorithm \ref{alg:diff}.

\begin{algorithm}
\begin{algorithmic}
\STATE {\bf Data}: A digital image $\Phi_0$, mask image $\mathcal M$, seed image $U_0$ 

\STATE {\bf Output:} Diffused Image $U_*$ 

\STATE {\bf Initialization:} $k=0$; $V^{(0)}={\bf 0}$, $\lambda^{(0)}={\bf 0}$

\STATE {\bf Parameters:} $\mu>0, \theta>0, \eta>0$

\STATE Set $\mathcal D=(\theta +\mu)\mathcal I-2G_0\mathcal F(\Delta)\mathcal F^{-1}$

\STATE For $k=1,2,\dots$  do
\begin{eqnarray*}
\begin{array}{l}
{U}^{(k+1)}=\mathcal F^{-1}\Big(
\mathcal F\big(\theta {U}^{(k)}+
2 \nabla \cdot \big((g(\Phi_0)-G_0) \nabla {U}^{(k)})
+\mu{V}^{(k)} + {\lambda}^{(k)}\big)/\mathcal D\Big);\\
{V}^{(k+1)}(x)={\rm Proj}_{\Gamma} (\gamma(x)),
\;\;\;\gamma(x)= (\eta_D{U}_0(x)+ \mu{U}^{(k+1)}(x)-{ \lambda}^{(k)}(x) )/(\eta_D (x)+ \mu);\\
\lambda^{(k+1)}={\bf \lambda}^{(k)}+\mu ({V}^{(k+1)} - {U}^{(k+1)});
\end{array}
\end{eqnarray*}

\STATE If stopping criteria satisfied, set $U_*=U^{(k+1)}$. \\
\end{algorithmic}
\caption{The Diffison Algorithm}\label{alg:diff}
\end{algorithm}

\begin{theorem}
Suppose that $\{(U^{k},  V^{(k)}, {\bf \lambda}^{(k)})\}_k\in \mathbb N$ be the sequence generated by the proposed method. 
Let us define the error sequences
$
 U_{e}^{(k)}=U^{(k)}-{U^*}, \;\;
V_{e}^{(k)}=V^{(k)}-{V^*},\;\;
\lambda_{e}^{(k)}=\lambda^{(k)}-{\lambda^*},\;\;
$
where $(U^*, V^*, \lambda^*)$ satisfies
the first-order optimality conditions (\ref{obj:cp}),
that is,
\begin{eqnarray}\label{kkt}
-2\nabla \cdot (g\nabla U^*)-\lambda^*=0,
\quad
\langle \eta_D(V^*-U_0)+\lambda^*, V-V^* \rangle \ge 0,
\quad
V^*-U^*=0.
\end{eqnarray}
The following statements hold:
\begin{itemize}
\item [a.]
The quantity
\[
E_k = \frac{\theta}{2}\|U_e^{(k)}\|^2 +\frac{\mu}{2} \|V_e^{(k)}\|^2+ \frac{1}{2\mu}\|\lambda_e^{(k)}\|^2
\]
is a monotone decreasing function of all $k\in\mathbb N$, and $\mu>0$ and $\theta>0$.
\item[b.] 
$\lim_{k \to \infty}  \|U^{(k+1)}-U^{(k)}\| =
\lim_{k \to \infty}   \|V^{(k+1)}-V^{(k)}\|  =
\lim_{k \to \infty}   \|\lambda^{(k+1)}-\lambda^{(k)}\|  = 0.$
\item[c.] Any limit point of the sequence $(U^{(k)}, V^{(k)}, \lambda^{(k)})$
is an stationary point.
\item [d.] The sequence $\{(U^{(k)}, V^{(k)}, \lambda^{(k)})\}_{k\in\mathbb N}$ is convergent.
\end{itemize}
where $\|f\|=\int_{\Omega} |f|^2 dx$.
\label{thm 2}
\end{theorem}
\begin{proof} The proof is given in Appendix \ref{proof of thm}.
\end{proof}

For multi-dimensional seeds, the minimization problem (\ref{obj:EWH}) is solved for each seed dimension separately and can be performed in parallel for efficiency.  As a result, the diffused image is also multi-dimensional.   

The converged image $U_*$ has a diffused value of nearby seeds.   The parameter $\eta$ in (\ref{obj:EWH}) controls how close the value in the domain $\Omega_{i,j}$ to the given seed $s_{i,j}$, and it is not necessary to keep $\eta$ very large.   Since the edge function $g$ gives information about the boundary of the objects, the diffusion will stop or slow down near the boundary, and give unique index to each object.  The value of the diffused image $U$ gives unique index $d_i$ to each object for $i=1,\dots,K$, i.e., we refer to this as  the diffused index image.

\subsection{Clustering and Counting [Step 3]}
\label{ssec:countingprocess}

During the diffusion process, seeds within each object start to converge to an unique index $d_i$, for $i=1,\dots,J$.  
Considering the histogram $H(U)$ of image $U$, the number of local maximum $J$ in $H(U)$ starts from the total number of seed plus zero value on $D^c$, i.e., $M+1$, and converges to $K$, the number of objects.  Since different values of seeds are placed uniformly, especially when multidimensional seeds are used, it is highly unlikely for  two objects in different locations to converge to the same index.   Local seeds converge to one unique index $d_i$ as long as they are within one object.  

In \textbf{[Step 3]}, we count the number of local maximum 
by clustering the histogram $H(U)$ of $U$.  Each local maximum represents $d_i$, a unique index for an object, and the number of such local maximums $K$ is the number of the objects in the image and big$\Phi_0$.

For one-dimensional scalar seed image,
we consider Gaussian fitting on $H(U)$  and we refer to it as Counting Objects by Diffused Index - Scalar seed clustering (CODI-S). 
The histogram can be described as a discrete function $h(r_k)=n_k/N$ where $r_k$ is the $k$th gray level intensity within the range $[0,255]$ in $U$, $n_k$ is the number of pixels having the intensity value $r_k$,
and $N$ is the total number of pixels in the image.  A discrete Gaussian filter 
$p(i)=\frac{1}{\sigma \sqrt{2 \pi}} e^{-i^2/2\sigma^2}$ for $i = -r, \cdots, r$, (where $\sigma>0$ denotes the variance) is considered onto $H(U)$ to obtain a smooth fitting curve. 
The number $K$ of local maximums is counted by implementing binary search recursively.  A larger $r$ and bigger $\sigma$ results in smoother curve which gives a fewer number of local minimums.  A smaller $r$ and smaller $\epsilon$ involved more details from the labeled pixels that it can give larger number $K$.   CODI-S has a large stable range of optimal parameters, due to smoothing $H(U)$ with the Gaussian convolution.  We used $\sigma \in [0.05,1.2]$, and $r = 5$ though out this paper. 
Figure \ref{Geometry} demonstrates the result obtained by the CODI-S on the 
cell image shown in (a). The mask image  described in [Step 1] is 
shown in (b). Notice the open boundaries between the three cells. (c) demonstrates the histogram and Gaussian fitted curve 
after the diffusion process [Step 2]; we observe three peaks where each peak corresponds to
each cell. The visualization of the clustering by CODI-S is also shown in (d), where the histograms in (c) are split into 3 sets at the two minimum values between the local maximums.

\begin{figure}
\centering
\begin{tabular}{cccccc}
(a) & (b) & (c) & (d) & (e) & (f) \\
{\includegraphics[height=.13\textwidth]{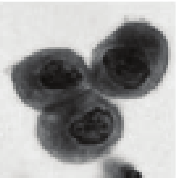}}&
{\includegraphics[height=.13\textwidth]{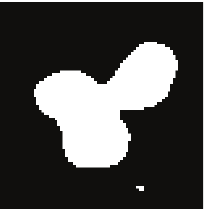}} &
{\includegraphics[height=.13\textwidth]{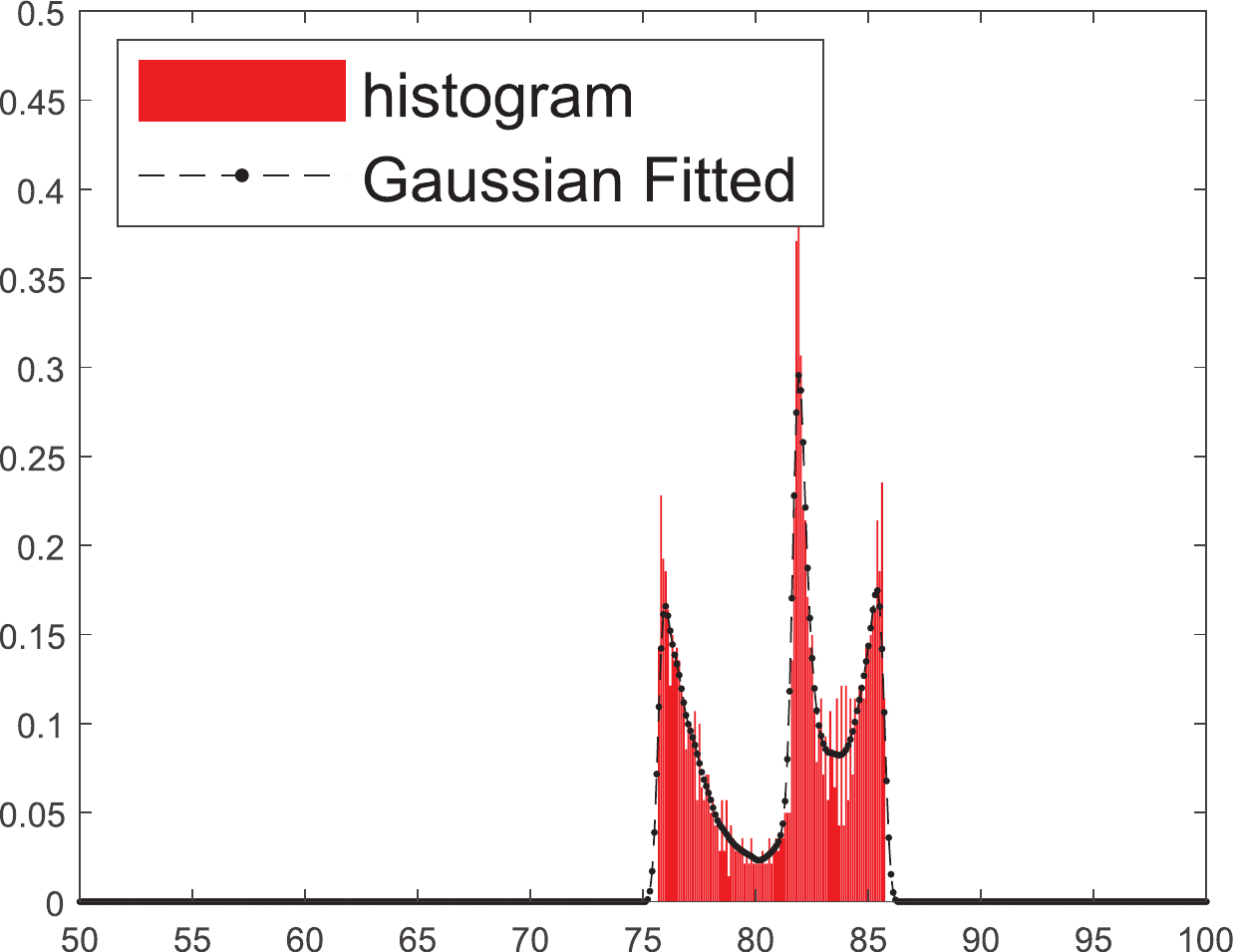}}&
{\includegraphics[height=.13\textwidth]{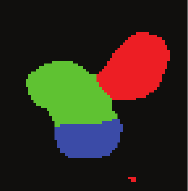}}&
{\includegraphics[height=.13\textwidth]{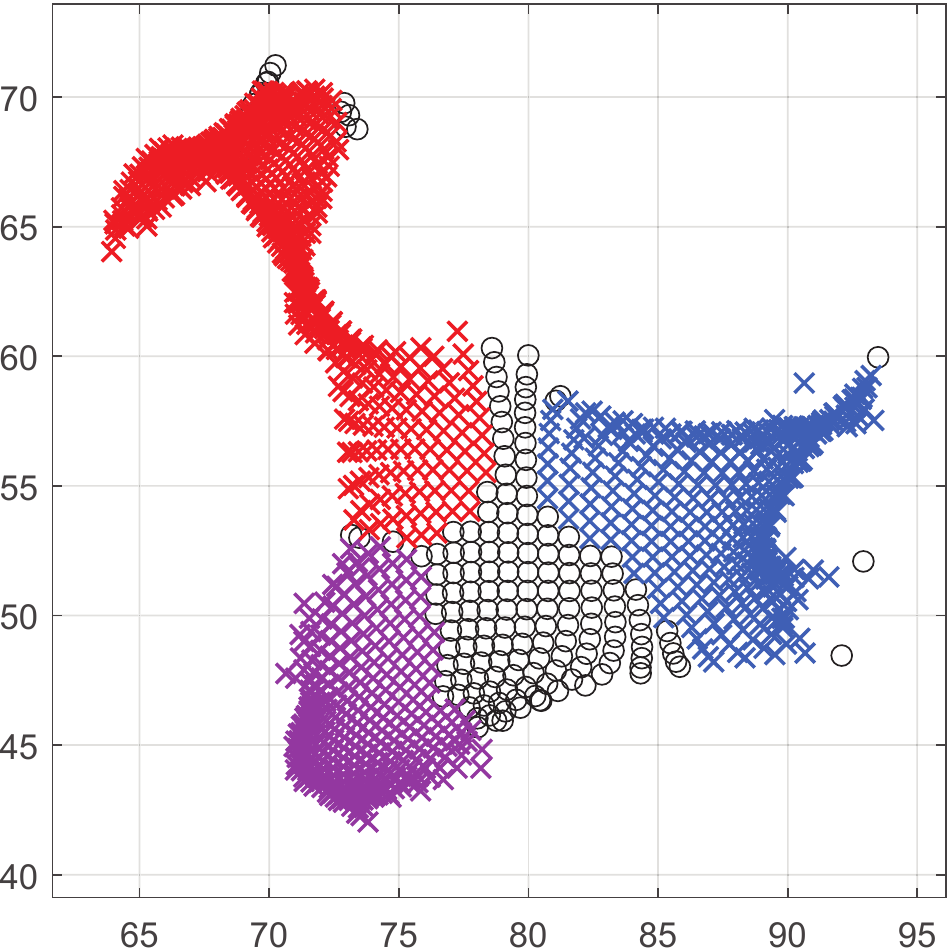}}&
{\includegraphics[height=.13\textwidth]{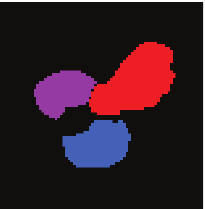}}
\end{tabular}
\caption{[CODI-S and CODI-M] 
(a) Given image with three cells
(b) The mask image showing open boundaries between the objects.
(c) The histogram and Gaussian fitted curve of  CODI-S.
(d) The visualization of CODI-S clustering in image domain. 
(e) The clustering results of CODI-M, projected onto two dimension for visualization.
(f) The visualization of CODI-M clustering in image domain. Both methods counts three cells.
}\label{Geometry}

\end{figure}

For multi-dimensional seeds, we use  high dimensional density method, such as DBSCAN, and refer to as Counting Objects by Diffused Index - Multi-dimensional seed clustering (CODI-M).  
Using DBCAN  \cite{ester1996density},  the seed vector similarity is tracked by the density, defined by $\epsilon$ and {\rm MinPts} via $l_2$ Euclidian distance norm.  
Here $\epsilon$ defines  $\epsilon$-neighborhood,  $N_{\epsilon}(\boldsymbol{x}) = \{\boldsymbol{y} \in \mathbb{R}^p : dist(\boldsymbol{x}, \boldsymbol{y}) \leq \epsilon \}$, and {\rm MinPts} is the minimum number of points required within
 $\epsilon$-neighborhood  to be connected as one cluster.  
 This property is called \textit{direct density reachability} of $\boldsymbol{x}$ from $\boldsymbol{y}$.  For points $\boldsymbol{x}$ that does not have density reachable points in its $\epsilon$-neighborhood, they are classified as noise.  
To find a cluster of 4-dimensional histogram $\boldsymbol{H(U)}$ of $\boldsymbol{U}$, we start with an arbitrary point $\boldsymbol{U}(x)$ and retrieves all  density-reachable points from $\boldsymbol{U}(x)$  with respect to given $\epsilon>0$ and  ${\rm MinPts}>0$. 
The reaching procedure ends when all points in a cluster has been visited and all the neighboring points in distance $\epsilon$ from any of the point in this cluster have been included in this cluster. The next step is to move onto the next unvisited point.
The accuracy of the method depends on the two parameters $\epsilon$  and {\rm MinPts}.
A relatively small {\rm MinPts} and large $\epsilon$ gives fewer and bigger clusters, while a relatively large {\rm MinPts} and small $\epsilon$ leads to more and smaller clusters. 
In this paper, we use $\epsilon \in [1, 1.2],  \text{MinPts} \in [ 12, 18 ] $ as the optimal range.

In CODI-M with DBCAN, clusters $\{ C_i | i = 1, \dots, K \}$ are formed, where the centroid is the unique index $\boldsymbol{d_i}$ for each object.   In Figure \ref{Geometry} (e) is a projection in two directions for visualization.  Each object is visualized in multi-dimensional space with a different color.   The number of different colors accurately gives the number of cells $K$. 
For each data $x \in \Omega$,  it is associated with the diffused index value and a cluster 
$$\{(x, \boldsymbol{U}(x), C_i ) :  \; \; U(x) \in C_i, \;\; i = 0,1,\dots, K\},$$
where $C_i$ denotes the $i$-th cluster in the high dimensional histogram domain, and $K$ denotes the count. Let $C_0$ stores $x$ which is considered as noise, and in \ref{Geometry} (e), $C_0$ is marked as black circles.  (f) shows a visualization of each cluster $C_i$ in $\Omega$ for $i=1,2$ and 3. 

%

\section{Properties of the proposed methods} \label{sec:property}

In this section, we focus on a few interesting properties of the propose methods.  Since the proposed CODI counts the diffused index  before the full convergence of diffusion algorithm,  we utilize this aspect to count objects which has open boundaries and explore this aspect.  Secondly, since we use clustering methods to count the diffused index, we can further extend this idea, and count different sized objects separately.  By using regularized $k$-means \cite{kang2011regularized}, we show how different size objects can be separately counted just by one simple additional step.

\subsection{Open boundary and counting accuracy}
\label{subsec: open boundary}

The proposed CODI counts the diffused index  before the full convergence of diffusion algorithm has reached.  These method can handle not fully closed boundaries in the objects, and we present the effect of such cases. 
In Figure \ref{F: Open boundaries MI-40iter},  three synthetic images  with different open boundaries are presented: (a) thick and narrow boundary opening, (b) thin and narrow boundary opening, and (c) thin and wide boundary opening.   The given image size is  $47 \times 91$ with the sizes of gaps as (a) $9\times 15$, (b) $9\times 3$  and (c) $21\times 3$.   Identical seeds distribution $U^1_0$ is used for CODI-S and the first dimension of CODI-M.  For CODI-M, we use  $U^2_0$ for second dimension, and two random seeds similar to the idea in Figure \ref{F: seed distribution} for third and forth dimensions.  The third and forth columns demonstrate CODI-S and the fifth and sixth columns demonstrate CODI-M  after $40$ and $80$ iterations of the diffusion process respectively.
 
We observe that even after short iterations for images (a) and (b), both CODI-S and CODI-M find two objects clearly, even with partially opened boundary.  (a1)- (a4) and (b1)-(b4) all finds two objects.   When the boundary opening is large and separation between the objects are not very clear like image (c), it is better for CODI-S to have smaller number of iteration while CODI-M needs a larger number of iterations.

\begin{figure}
\begin{center}
\begin{tabular}{cccccc}
\multicolumn{2}{c}{Seeds \& Images } & CODI-S-40 iter &  CODI-S-80 iter &CODI-M-40 iter& CODI-M-80 iter \\
(a) & $U^1_0$  & (a1)  & (a2) & (a3) & (a4) \\
\includegraphics[align=c,height= 0.9in]{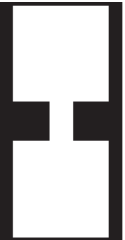} & 
\includegraphics[align=c,height= 0.9in]{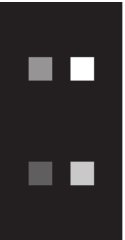} &
\includegraphics[align=c,height= 0.9in]{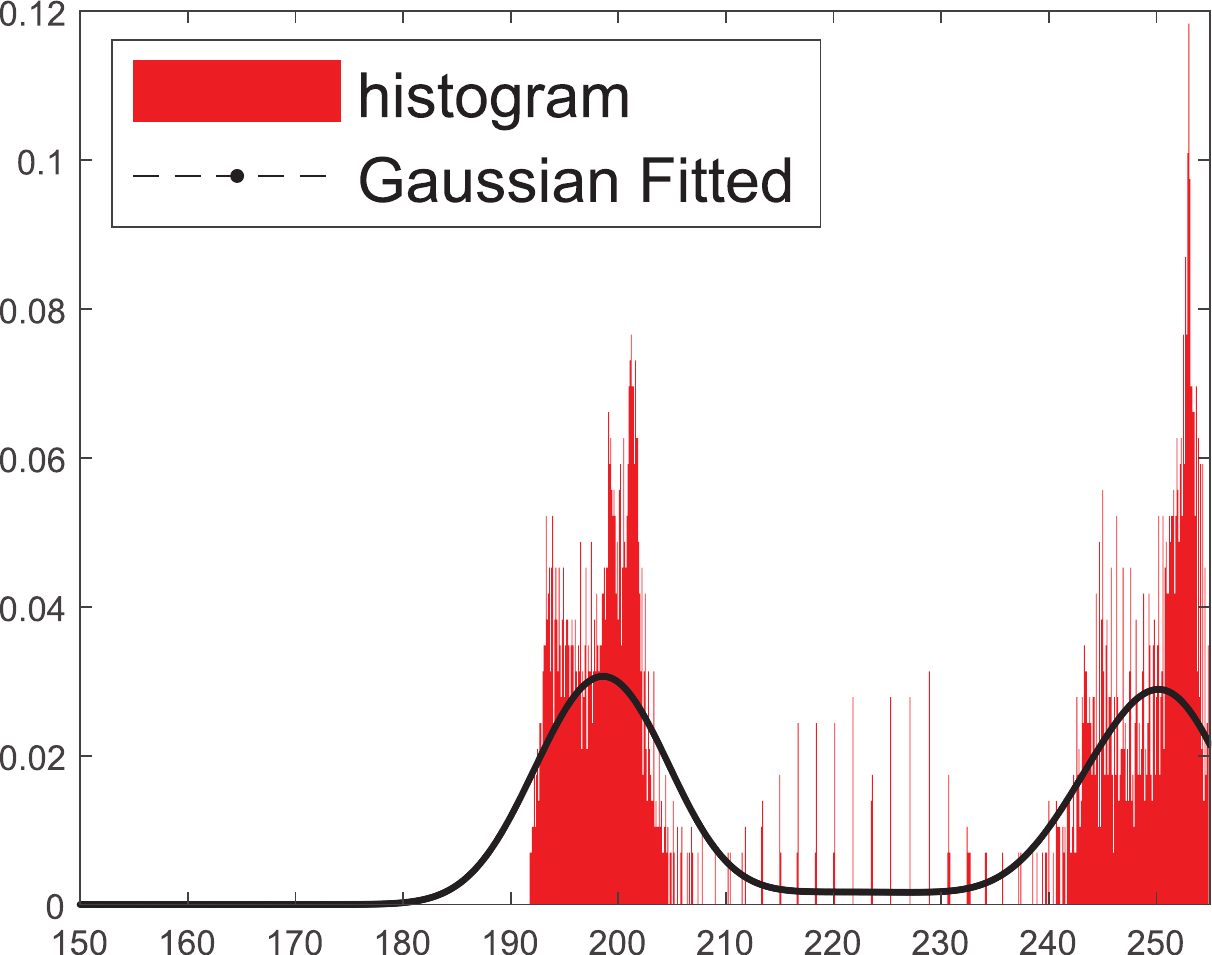} & 
\includegraphics[align=c,height= 0.9in]{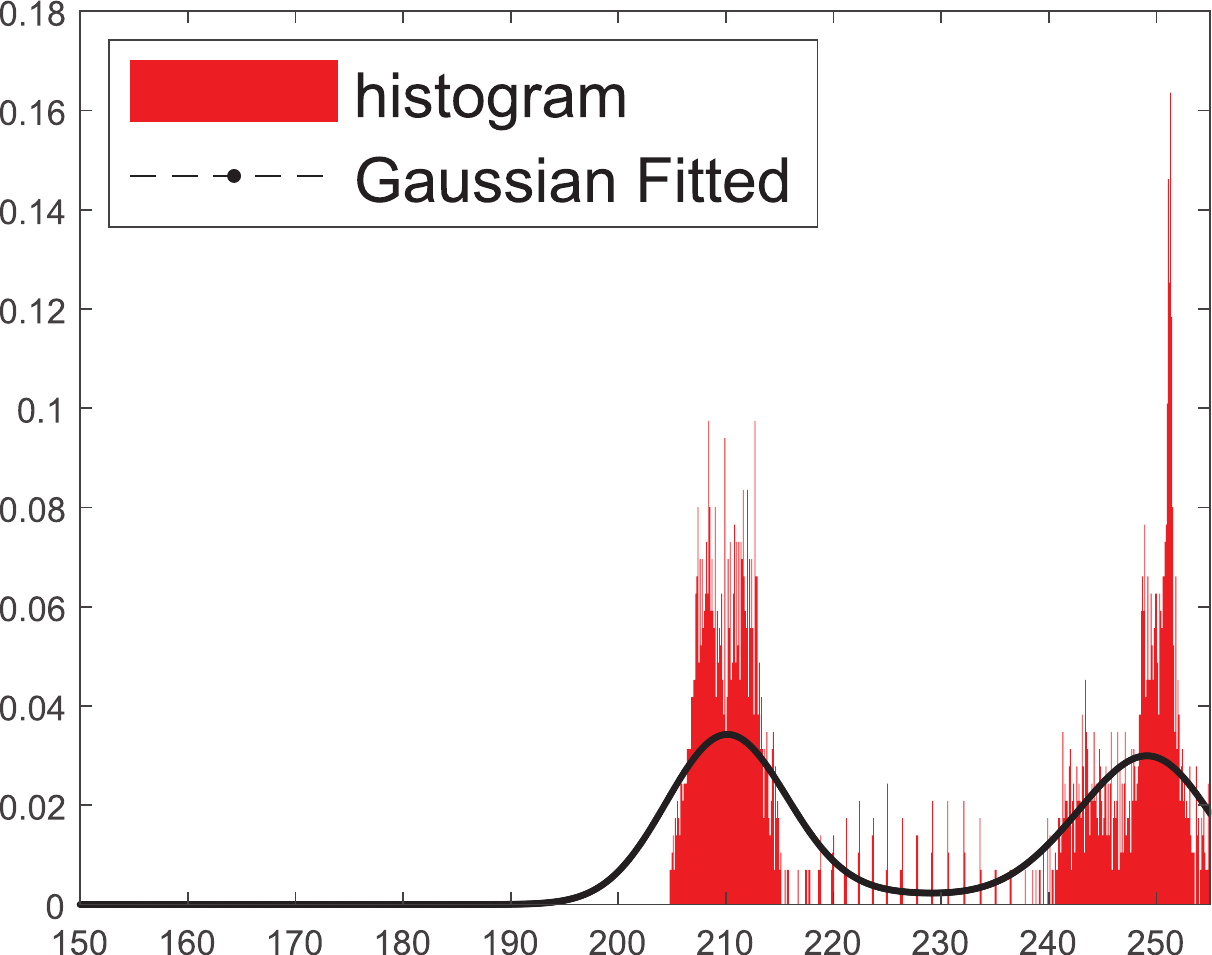}&
\includegraphics[align = c,height =  0.9in]{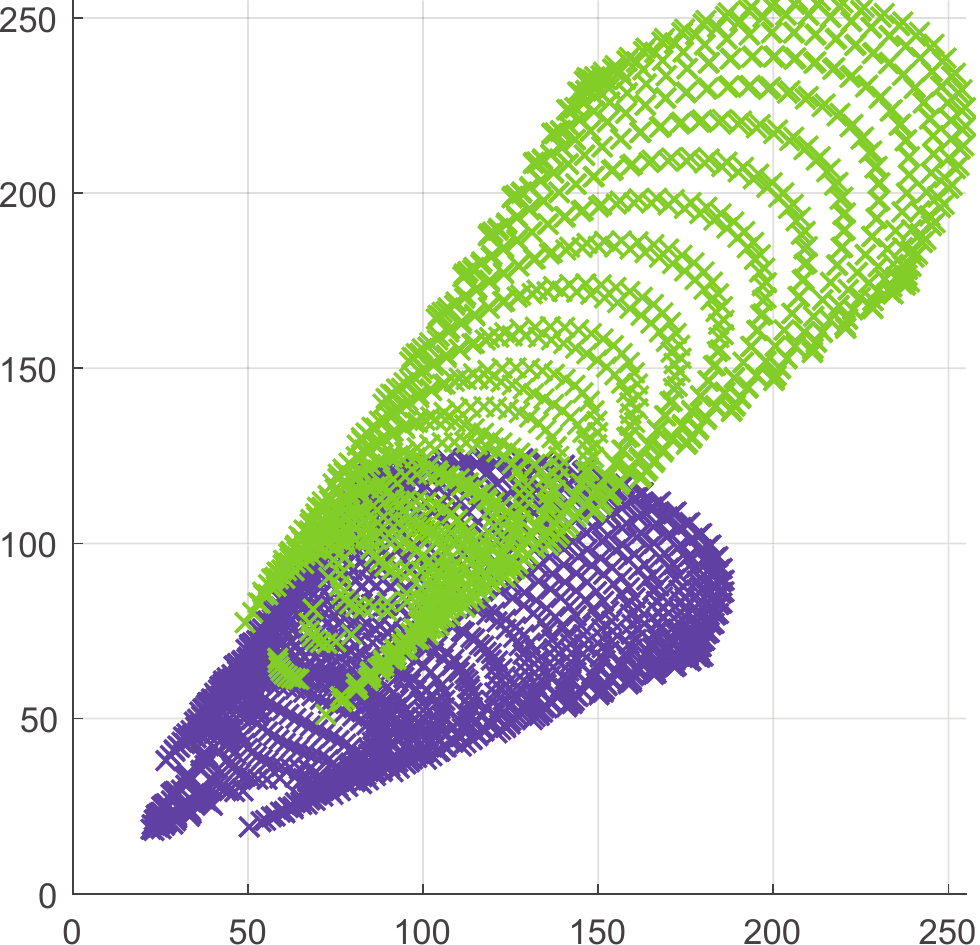} &
\includegraphics[align = c,height = 0.9in]{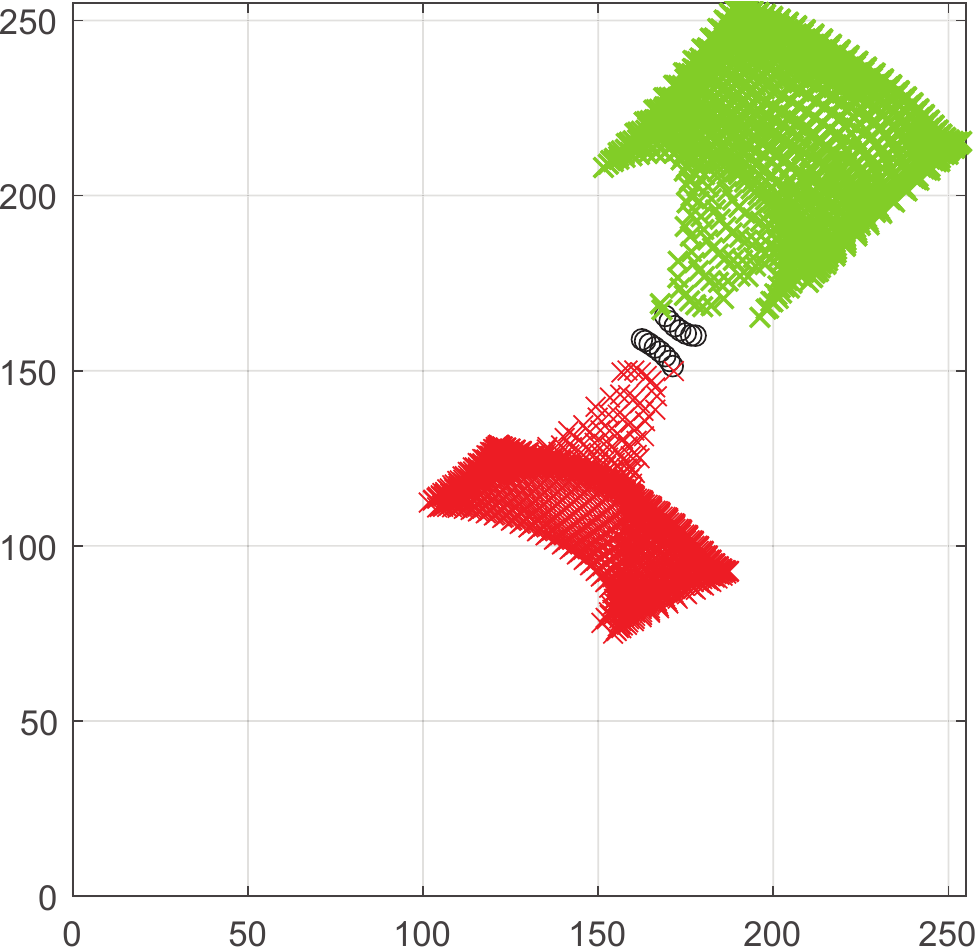} 
 \\
\hspace{0.2cm}
(b) &  $U^2_0$   & (b1)  & (b2) & (b3) & (b4)  \\
  \includegraphics[align=c,height= 0.9in]{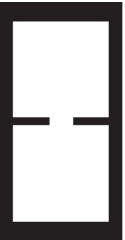} &
  \includegraphics[align=c,height= 0.9in]{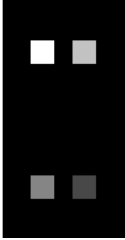} & 
 \includegraphics[align=c,height= 0.9in]{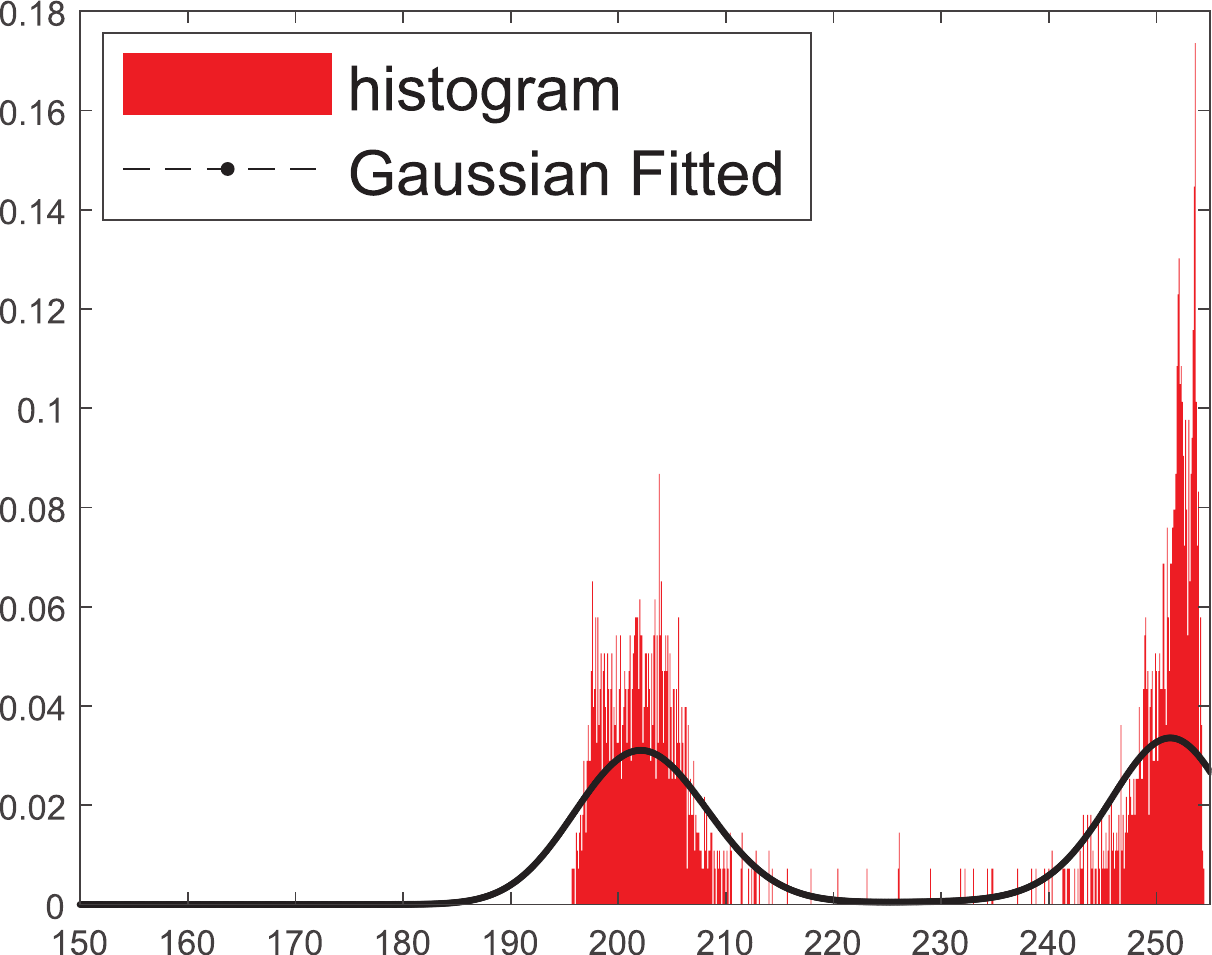}&
 \includegraphics[align=c,height= 0.9in]{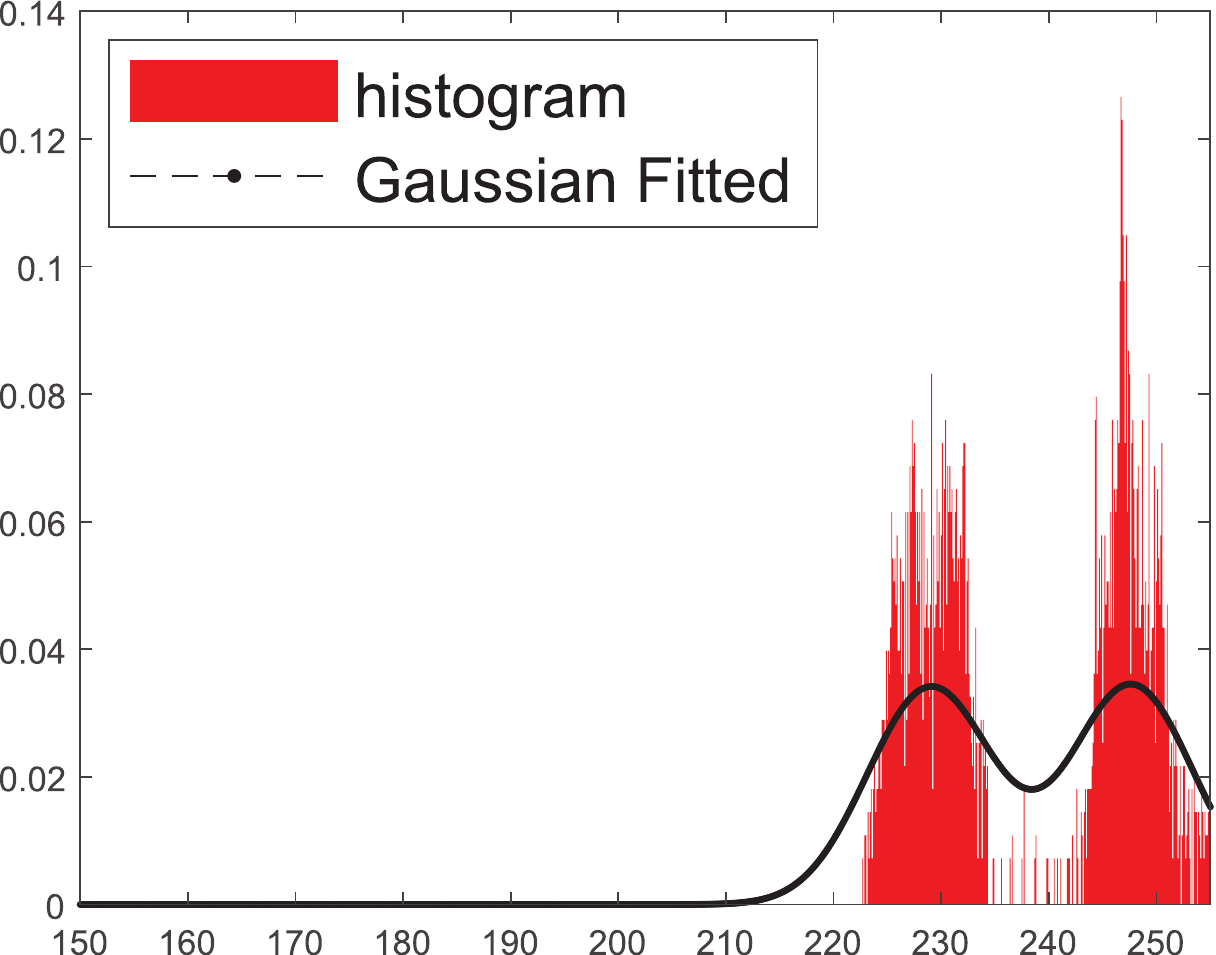}&
   \includegraphics[align = c,height = 0.9in]{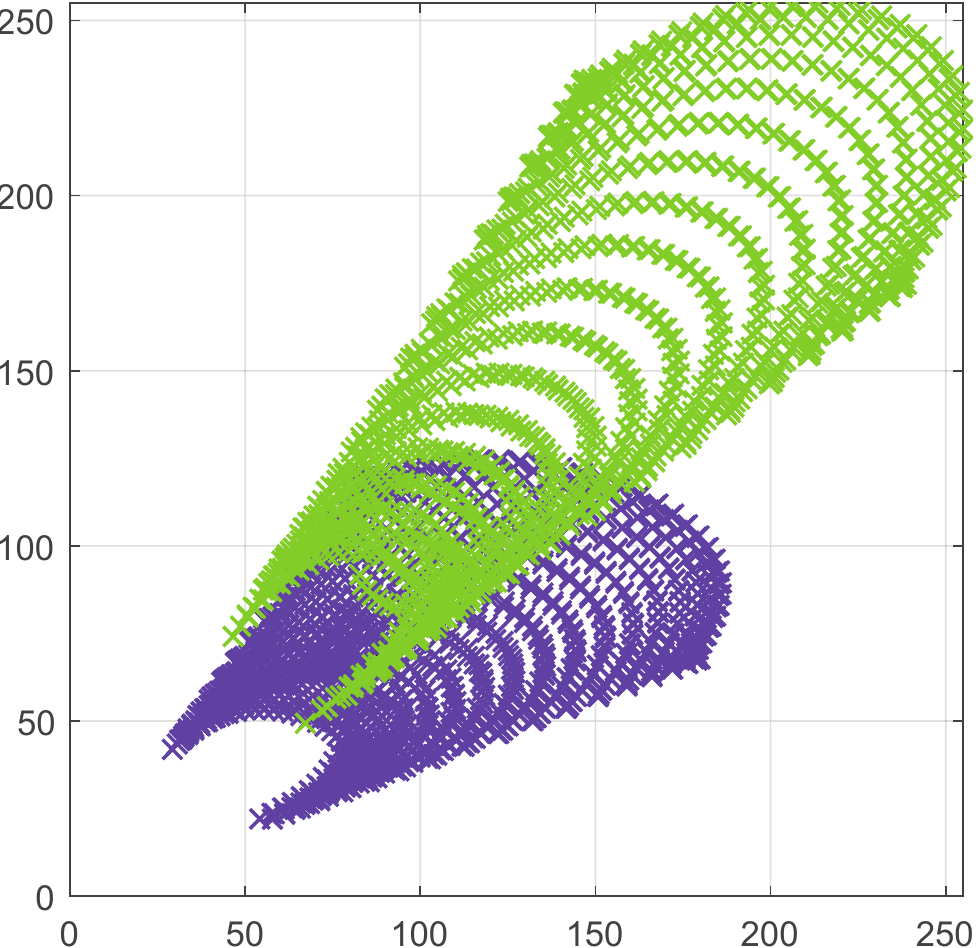}&
\includegraphics[align = c,height = 0.9in]{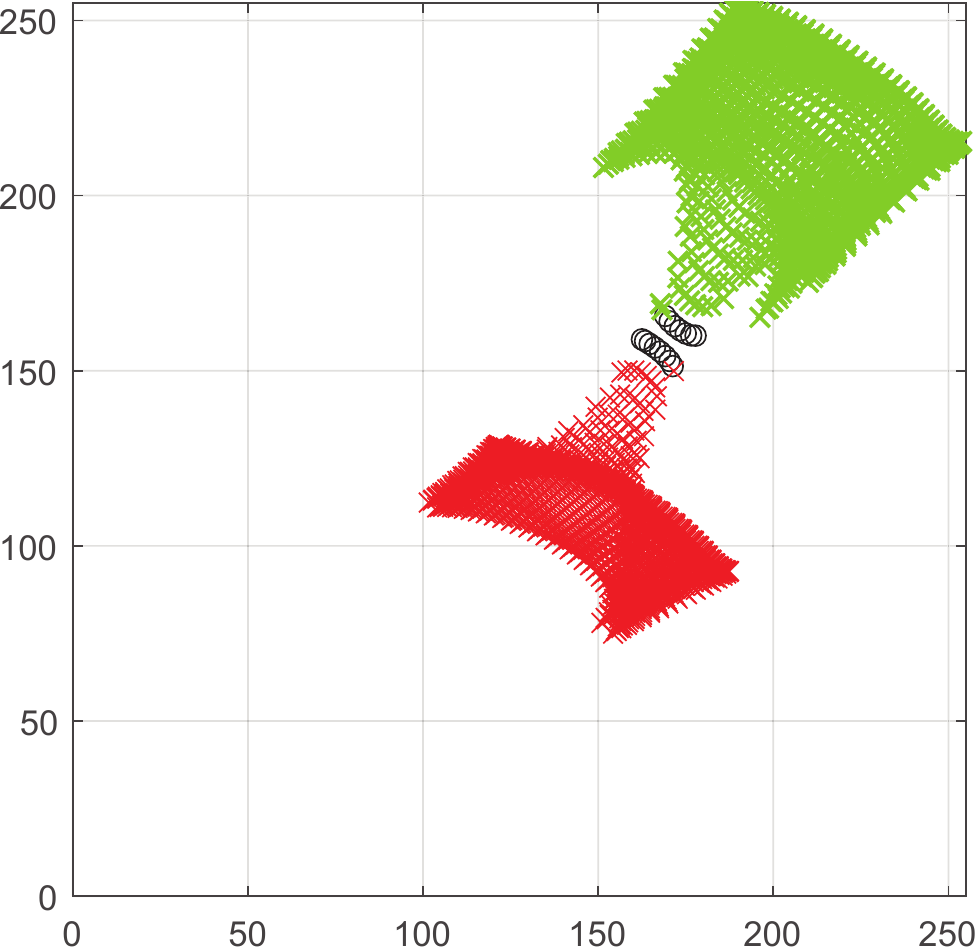}
\\
(c) &    & (c1)  & (c2) & (c3) & (c4)  \\
 \includegraphics[align=c,height= 0.9in]{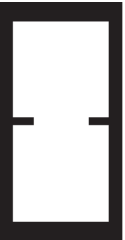}  & 
 & 
 \includegraphics[align=c,height= 0.9in]{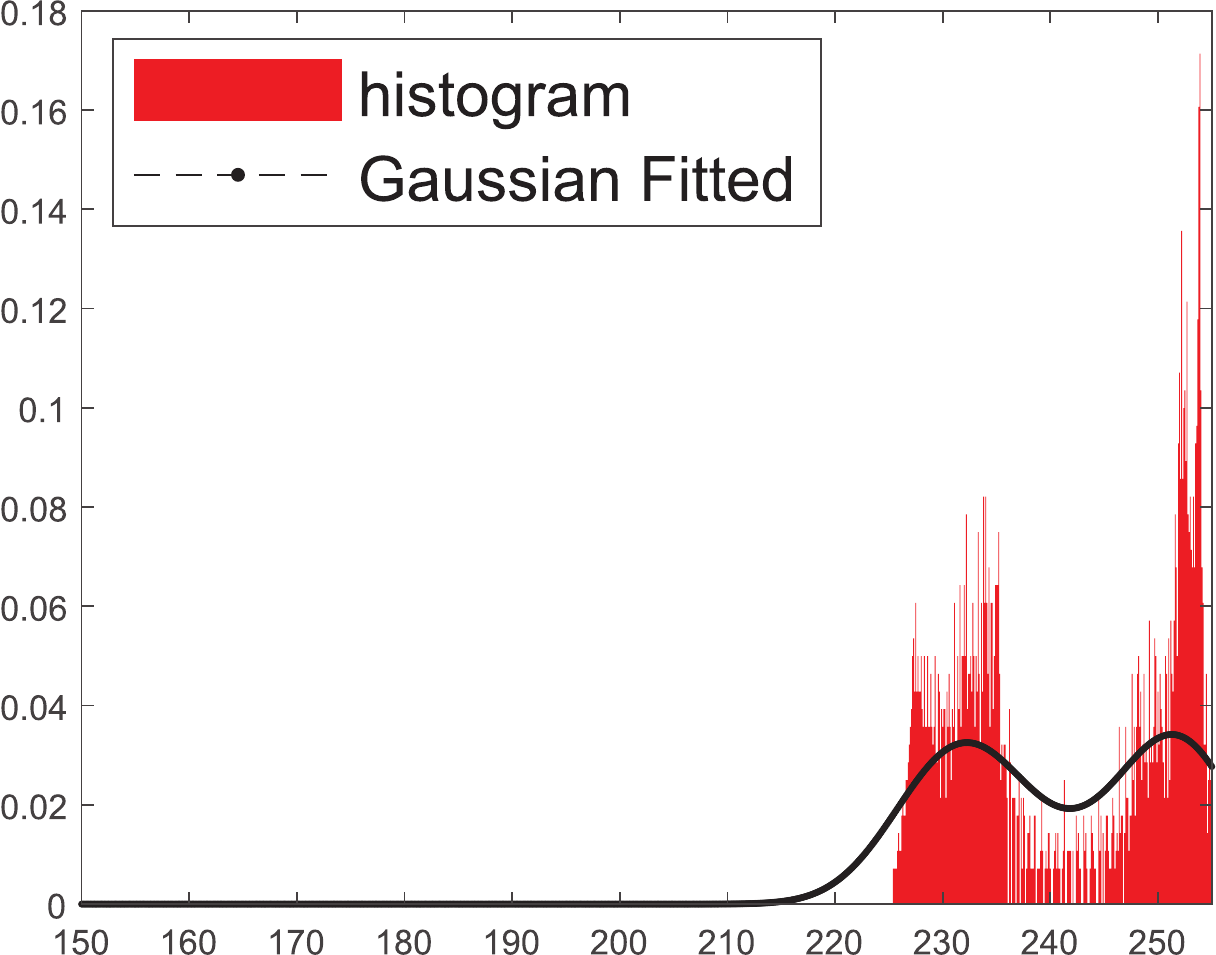} &
\includegraphics[align=c,height= 0.9in]{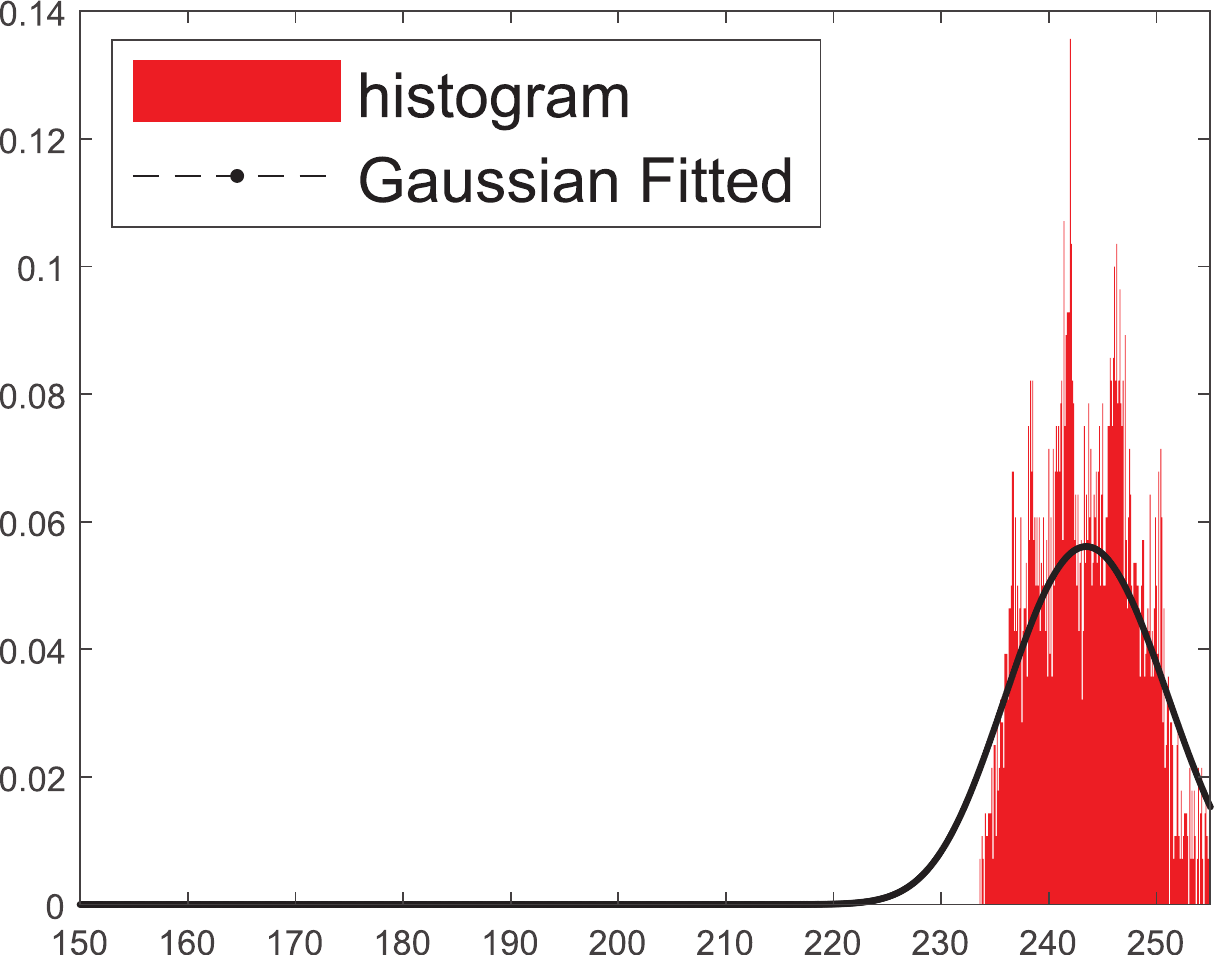}&
 \includegraphics[align = c,height = 0.9in] {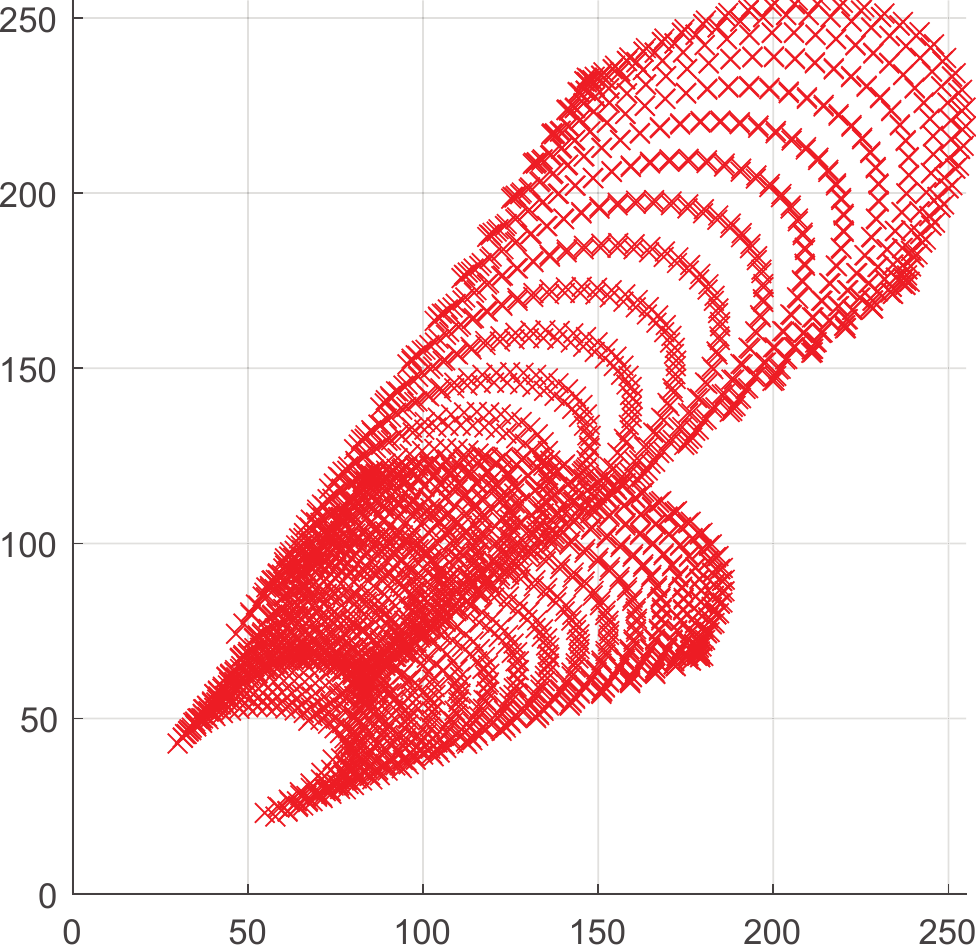} &
 \includegraphics[align = c,height = 0.9in] {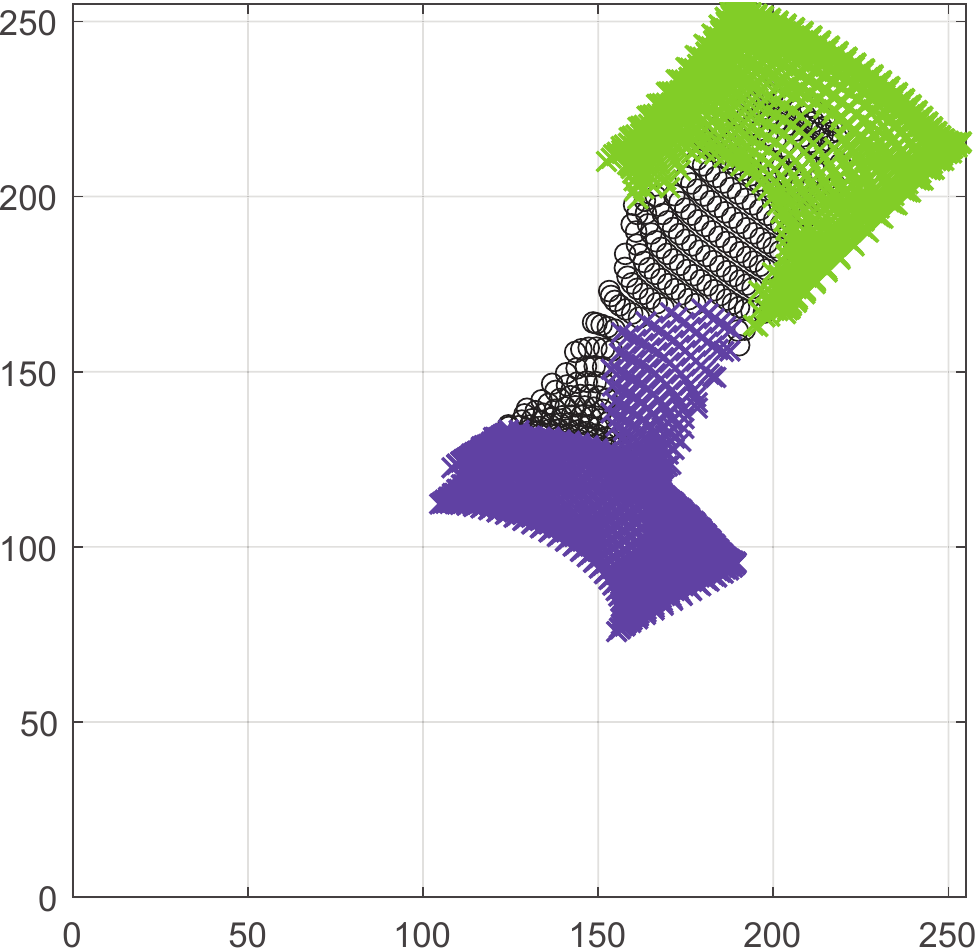}\\
\end{tabular}
\end{center}
\caption{[Open boundary experiments] (a), (b) and (c)  show three synthetic images where two square 
objects are separated with various size of gaps.  An identical seed image $U^1_0$ is  used for CODI-S and the first dimension of CODI-M.  $U^2_0$ is used for second, and two random seed images for third and forth dimensions.   The third and forth columns show CODI-S, and the fifth and sixth columns CODI-M after $40$ and 80 iterations respectively.  When the gaps between objects are wide and thin, it is helpful to have diffusion iteration small for CODI-S and large for CODI-M. }\label{F: Open boundaries MI-40iter}
\end{figure}

\subsection{Further grouping counts of similar sized objects}

\begin{figure}
\centering
\begin{tabular}{cc}
(a) &  \\
\includegraphics[align = c, width = 0.55\textwidth]{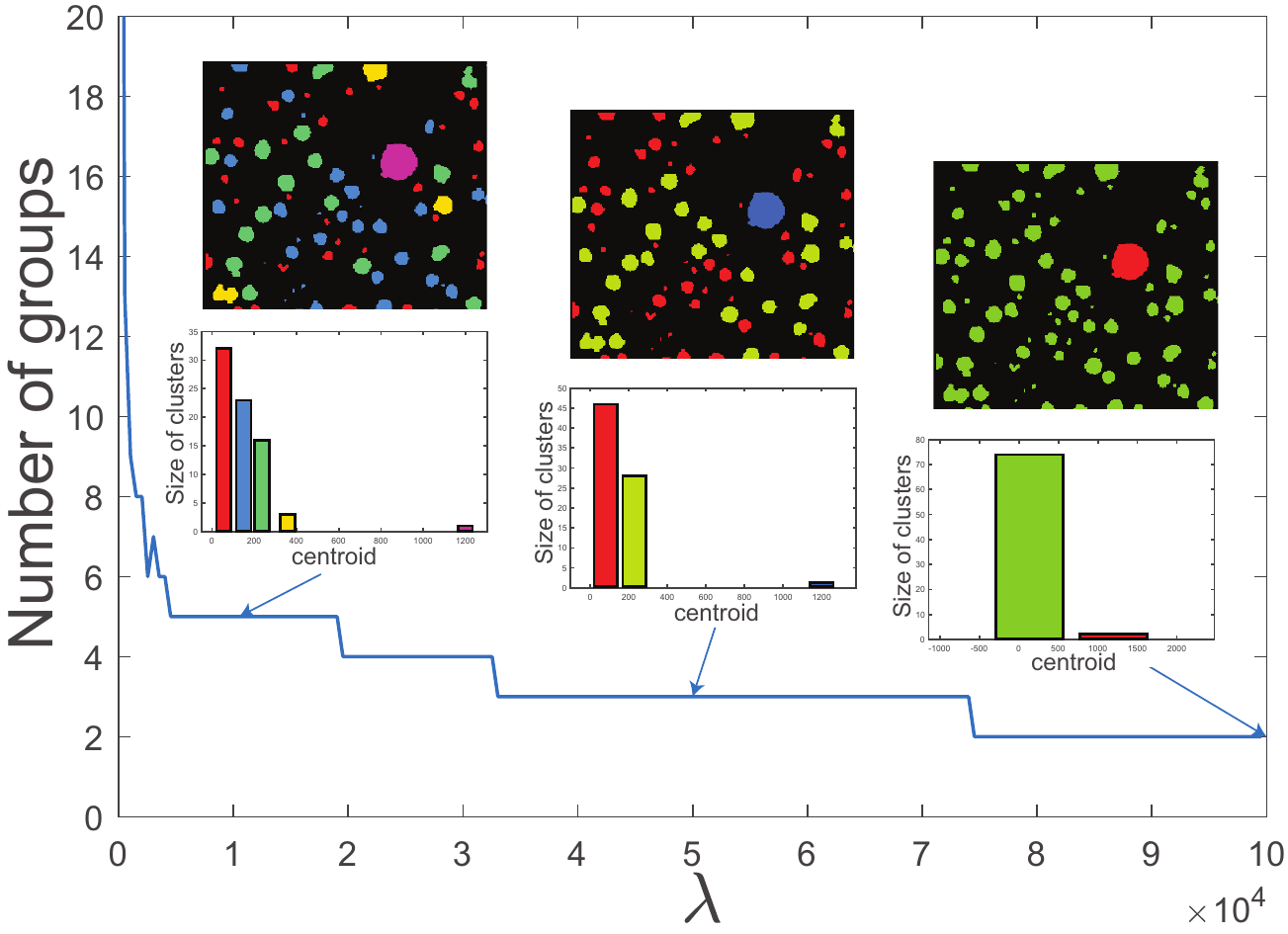} &
\begin{tabular}{cc}
(a1) & (a2) $\lambda= 1 \times 10^4$ \\
\includegraphics[align = c,  width = 1.1in]{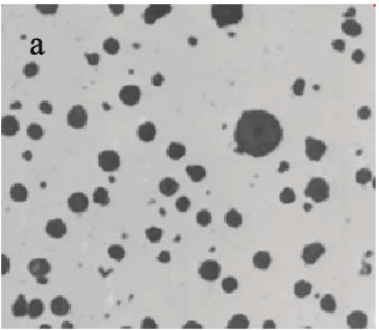} 
&
\small 
 \begin{tabular}{ccc}
 \midrule
$I_i$ &  $c_i$ & $|I_i|$ \\
\midrule
 1  &    56.18    &       33\\
 2  &    156.74   &        23\\
 3  &    242.14   &        14\\
 4  &    357.33    &        3\\
 5  &      1199    &        1\\
\midrule
\end{tabular}\\
& \\
(a3) $\lambda= 5 \times 10^4$ & (a4)  $\lambda= 1 \times 10^5$ \\
\small
 \begin{tabular}{ccc}
 \midrule
$I_i$ & $c_i$ & $|I_i|$ \\
\midrule
 1   &    78.82   &        45\\
2   &    229.25     &      28\\
 3   &     1199.00     &       1\\
\midrule
\end{tabular}
&
\small
 \begin{tabular}{ccc}
 \midrule
$I_i$ & $c_i$ & $|I_i|$ \\
\midrule
1   &    135.9     &      73\\
 2   &     1199     &       1\\
\midrule
\end{tabular}
\end{tabular} \\
(b) &  \\
\includegraphics[align = c, width = 0.55\textwidth]{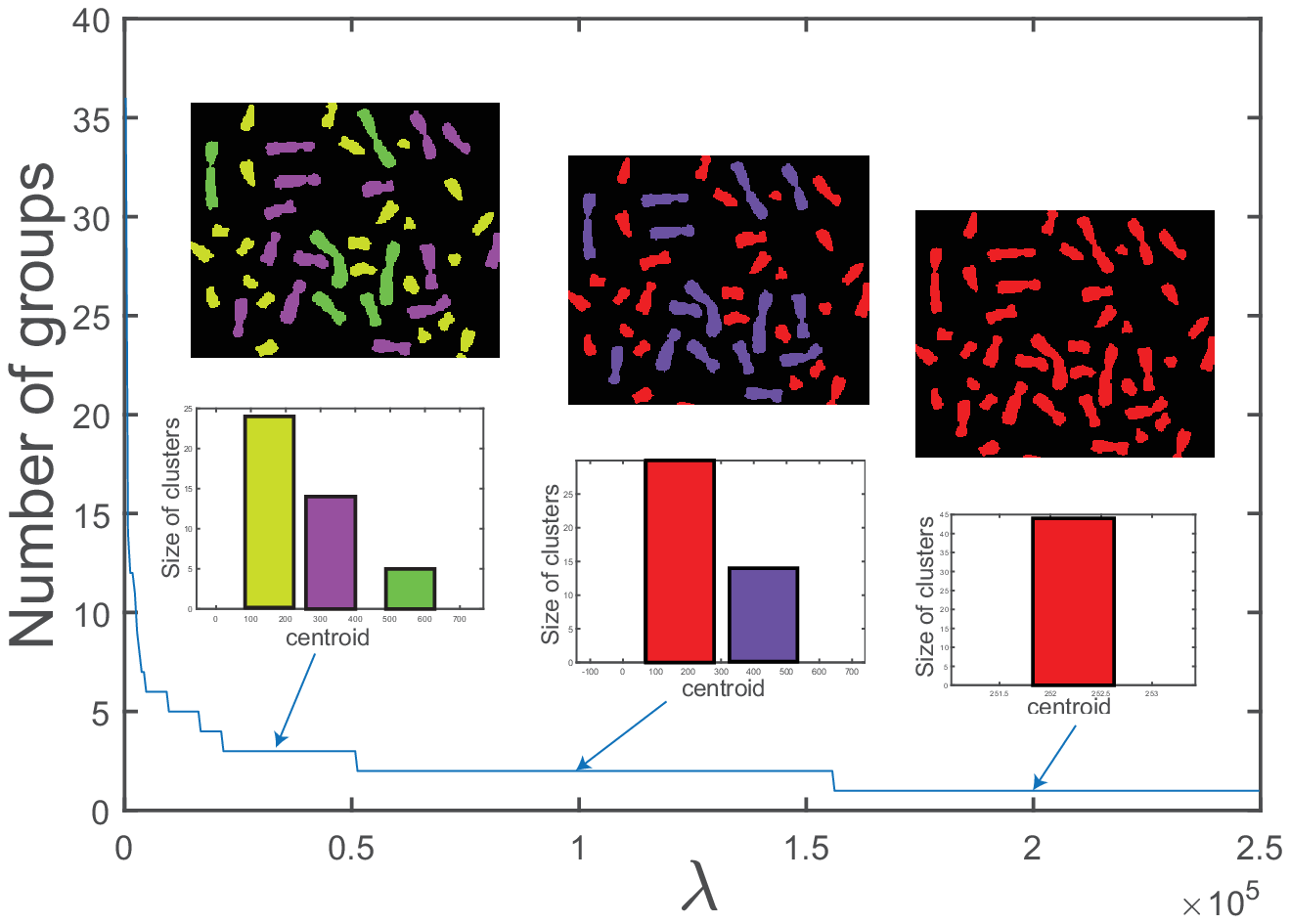} &
\begin{tabular}{cc}
(b1) & (b2) $\lambda= 3 \times 10^4$ \\
\includegraphics[align = c,  width = 1.1in]{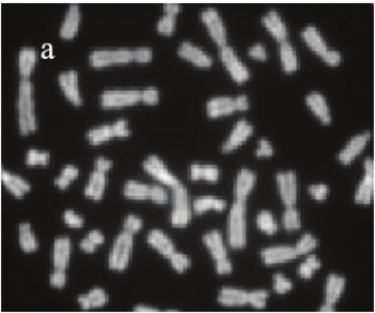}
&\small 
 \begin{tabular}{ccc}
 \midrule
$I_i$ &  $c_i$ & $|I_i|$ \\
\midrule
 1   &   153.17     &      24\\
 2   &   327.71       &    14\\
 3  &    558.00      &     5\\
\midrule
\end{tabular}\\
& \\
(b3) $\lambda= 1 \times 10^5$ & (b4)  $\lambda= 2 \times 10^5$ \\
\small
 \begin{tabular}{ccc}
 \midrule
$I_i$ &  $c_i$ & $|I_i|$ \\
\midrule
1&       173.90     &      29\\
  2&     429.36       &    14\\
\midrule
\end{tabular}
&
 \small
 \begin{tabular}{ccc}
 \midrule
$I_i$ &  $c_i$ & $|I_i|$ \\
\midrule
1&       257.07    &      43\\
\midrule
\end{tabular}
\end{tabular}
\end{tabular}
\caption {[Counting similar size objects] (a1) and (b1) are given images from \cite{guo2013method}, and CODI-M found $K=74$ and $K=43$ cells respectively. 
(a) and (b) are graphs of  $\lambda$ vs. the number of groups.  Three $\lambda$ values for Regularized $k$-mean (\ref{E:regK}) is picked from plateaus $\lambda = 1 \times 10^4,  5 \times 10^4, $$ 1 \times 10^5$ for (a1) and $\lambda = 3 \times 10^4,  5 \times 10^4, $ $ 1 \times 10^5$  for (b1).   Each $\lambda$ shows different grouping depending on the size of objects  from $S$.  (a3) shows grouping to three different sizes, while (a4) shows grouping to two groups: one with one big object and another with all others. (b3)-(b5) also show different grouping possibilities. }
\label{F: two examples in counting cells-cluster} 
\end{figure}

Since the proposed method utilize clustering of $\boldsymbol{H(U)}$, we can further distinguish different sizes after the clusters $C_i$s are found.   The clustering of $\boldsymbol{H(U)}$ gives data $x \in \Omega$ in the form of  
\[ \{(x, U(x), C_i ) :  \; \; U(x) \in C_i, \;\; i = 0,1,\dots, K\},\]
where $C_i$ denotes the $i$-th cluster in the multidimensional histogram domain, and $K$ denotes the counting result.    Now, we consider the size of each clusters $S= \{ |C_i| | i= 1, \dots,K \}$ and use the Regularized $k$-means algorithm  \cite{kang2011regularized} to further cluster this set $S$.  The Regularized $k$-mean energy is given by 
\begin{equation}\label{E:regK}
E[k, \{I_i\}, \{c_i\} |S ] = \lambda \left( \sum_{i=1}^k \frac{1}{n_i} \right) + \sum_{i=1}^k \sum_{|C_j| \in \{I_i\}} ||C_j| - c_i|^2,
\end{equation}
which is minimized for given size of each cluster $|C_i|$.  Here 
$k$ is the number of groups found in the grouping process, $n_i = |I_i|$ is the number of objects that are contained in the group $I_i$, and $c_i =\{\frac{1}{k} \sum_{j=1}^k |C_j| : |C_j| \in I_i \}$  is the average object size in the group $I_i$.  This  $l_i$  represents the group with similar size objects, and this similarity of the sizes are determined by the $\lambda$.   This model  automatically picks a reasonable number of cluster $k$ with a parameter $\lambda$. A large $\lambda$ gives fewer clusters while a small $\lambda$ gives more number of clusters.
 
In Figure \ref{F: two examples in counting cells-cluster}, (a1) is the given image from  \cite{guo2013method} where we used the edge function as 
 $g(t) = \chi_{t < 130}$ and $g(t) = \chi_{t > 125}$, with  $\displaystyle{
    \chi_{t \in \Omega}(t) = \begin{cases}
    1 & t \in \Omega \\
    0 & o.w.
    \end{cases}}$ to threshold the given image.
As a counting result, CODI-M identifies $K= 74$ objects.  From the given image (a1) and its counting result $ \{(x, U(x), C_i ) :  \; \; U(x) \in C_i, \;\; i = 0,1,\dots, K\}$,  (a) shows a graph of  $\lambda$ vs. the number of groups.  Notice that the Regularized $k$-mean (\ref{E:regK}) has large plateaus showing the clustering result (the number of $k$) is not very sensitive against the choice of $\lambda$.   We picked three $\lambda$ values around different plateaus  $\lambda = 1 \times 10^4,  5 \times 10^4, $ $ 1 \times 10^5$ for (a1) and $\lambda = 3 \times 10^4,  5 \times 10^4, $$ 1 \times 10^5$ for (b1).    The colored image shows different size objects identified by different colors, and the histogram of $S$ and tables below show more details.  
In each histogram, each bar denotes a group of different size objects.  The horizontal axis -- centroid size of each group -- is the average size of objects in each group. The height of bars denote the number of objects that belongs to the same group.   In the table (a2)-(a4), $I_i$ shows how many different kinds of sizes are identified, $c_i$ represents the average size in that group, and $|I_i|$ represents how many  of such object exists in each group.  For example in (a2), the table shows there are 5 different size of objects in image (a), with 32 number of the size around 54 objects, 23 of bigger objects of size 150, 16 of bigger ones of size 236, 3 of bigger ones of size 357, and one very big one of size 1199.    As $\lambda$ gets bigger the grouping gets simplified: (a3) separates objects to three, two of smaller sizes (46 of size around 78, and 28 number of size around 230), and one big one of size 1199.  (a4) shows it can be also separated to two different sizes one big one and all other smaller of average size 135. 

The sizes of cells in (b1) are similar size. Table (b2) shows that when using $3 \times \lambda = 10^4$, 3 groups are formed, where the  largest group has 24 objects of average size $327.71$ pixels, and the smallest group contains 5 objects of size 558 pixels.  
As shown in table (b3), as $3\lambda$ increases to as large as $10^5$, 2 groups are formed, where the  larger group has 30 objects with averages size to be 169.57 pixels, which distinguishes the longer cells and shorter cells. 
When $\lambda = 2 \times 10^5$, all objects move  into one single group of average size to be 252.33 as shown in table (b4). 
In conclusion,  a smaller $\lambda$ gives more groups  and the plateaus of $k-\lambda$ curves in Regularized K-means  provide  meaningful justification about the number of groups of objects with respect to the distribution of size of objects in a given image.

\section{Numerical Experiments and Comparisons} \label{sec:experiment}
In this section, we demonstrate the effectiveness and efficiency of the proposed  
methods on various examples.
All experiments are performed on MATLAB using Intel\textregistered Core i5-9600K processor with 3.7GHz 6Core CPU and 16 GB of RAM. 
In all experiments, we fix $\mu= 5 \times 10^{-5}$,
$\theta=1$, and $\eta=0.0001$ in Diffusion Algorithm. 
In some cases, downsampling of original image is used for computational efficiency.  An artificial outline is added  on the boundary of the image domain $\Omega$ to prevent merging of objects near the boundary due to the effect of Fast Fourier transform.
For CODI-S, we use a horizontal seed and for CODI-M we use a 4-dimensional seed involving one vertical, one horizontal, and two random seeds, as shown in Figure \ref{F: seed distribution}.  Due to the two dimensions with random seeds, multiple tests are performed.  

{\bf Cell counting:}  We experiment on cells images in Figure \ref{F: two examples in counting cells-cluster} (a1) and (b1). The counting results are illustrated in Figure \ref{F: two examples in counting cells-cluster-codis}. (a5) and (b5) show the results from \cite{guo2013method}. (a6)-(a7) and (b6)-(b7) are results by CODI-S and CODI-M respectively.  The method in \cite{guo2013method} counted 74 cells in (a5) and 43 cells in (b5).  The CODI-S  count 70 cells in (a6) and 45 in (b6). The CODI-M  count 73 cells in (a7) and 43 cells in (b7).  For CODI-M, experiments are performed 20 times, and the counting results varies between $[72,74]$ for (a1) where 74 cells are found in 18 out of 20 trials.  For (b1), the counting result locates between 42 and 46 among 16 out of 20 trials.  The average cpu time is 0.82 second and 0.77 second for (a1) and (b1) respectively.  This shows that the CODI-M and CODI-S are both comparable to \cite{guo2013method}, and geometry-independent, and able to count cells of various sizes and shapes very efficiently.
\begin{figure}
 \begin{center}
\begin{tabular}{cccc} 
 (a1) & (a5) & (a6) &(a7)\\
\includegraphics[align = c,  width = .18\textwidth]{figures/ct_by_size_1_original_image}  &
 \includegraphics[align = c,  height = 0.16\textwidth]{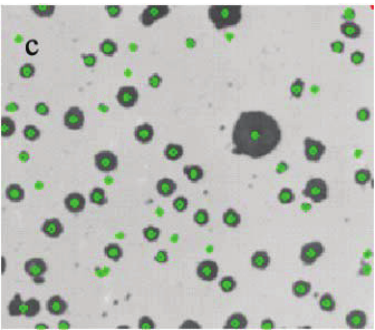} &
 \includegraphics[align = c, height = 0.17\textwidth]{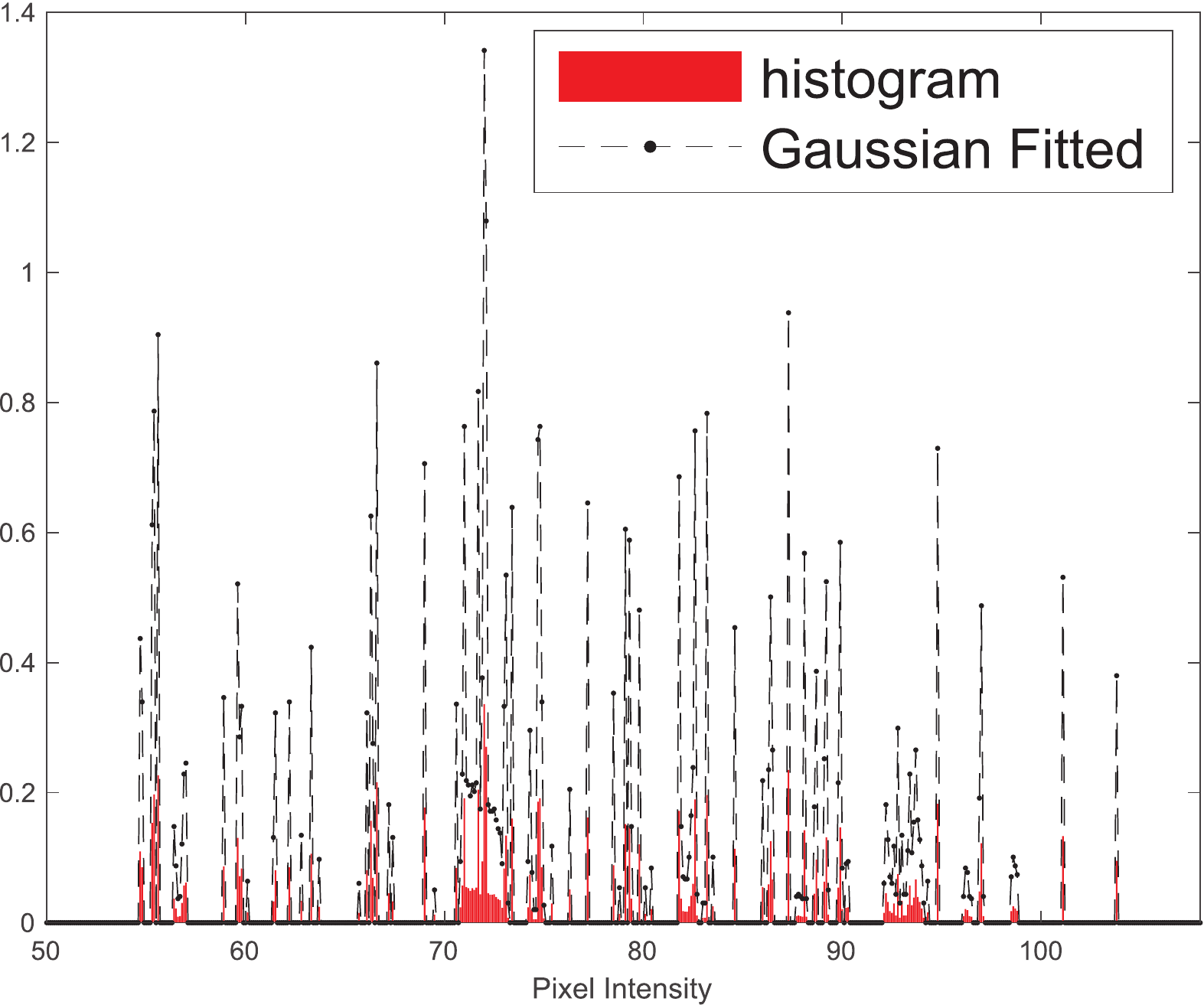}
& \includegraphics[align = c, width = 0.18\textwidth]{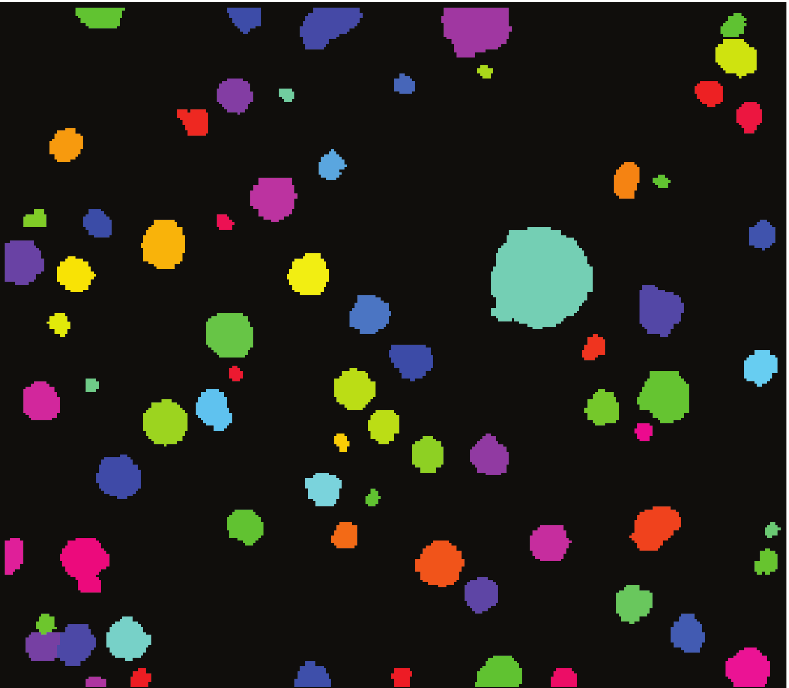} \\ 
 (b1) & (b5) & (b6) &(b7)\\
 \includegraphics[align = c,  width = .18\textwidth]{figures/ct_by_size_2_original_image}  & 
 \includegraphics[align = c,  height = 0.16\textwidth]{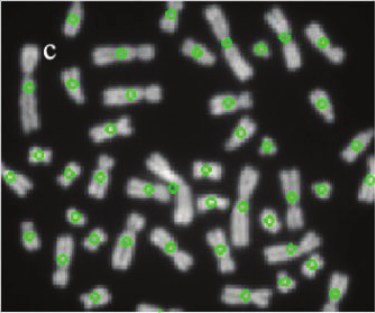} &
 \includegraphics[align = c, height = 0.17\textwidth]{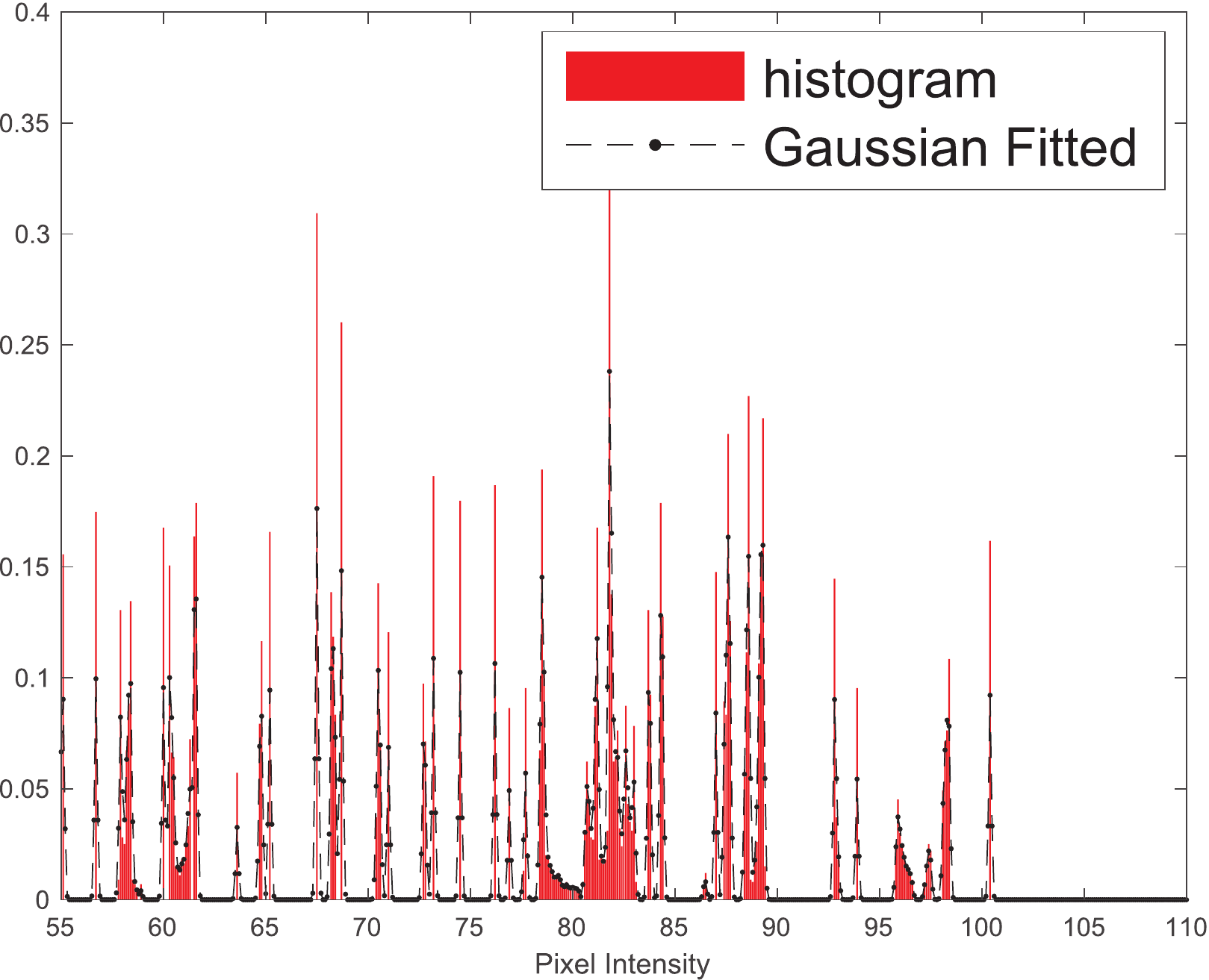}
& \includegraphics[align = c, width = 0.18\textwidth]{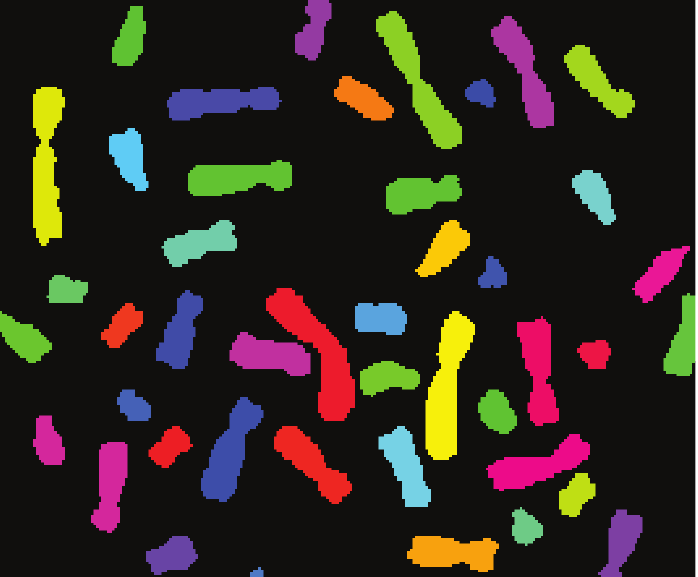} \\  & \\
\end{tabular}
\begin{tabular}{cccc} 
\toprule
Given image& CODI-S & CODI-M     & Existing result    \\
\midrule
(a1) & 70 (1.68s) & 74  (0.82s) & 74 from \cite{guo2013method}\\
(b1)    & 45 (1.87s)& 43 (0.77s) & 43 from \cite{guo2013method}\\
\bottomrule
 \end{tabular}
\end{center}
\caption {[Cell counting] (a1) and (b1) are two cells images in Figure \ref{F: two examples in counting cells-cluster} from \cite{guo2013method}.  (a5) and (b5) are results from \cite{guo2013method}. (a6) and (b6) are results of CODI-S.  (a7) and (b7) are results of CODI-M.   CODI-M experiments are performed 20 times, with the counting results between $[72,74]$ for (a1) where 74 cells are found in 18 out of 20 trials.  For (b1), the counting between 42 and 46 among 16 out of 20 trials. The average cpu time is 0.82 second and 0.77 second for (a1) and (b1) respectively. CODI is geometry-independent, and able to count cells of various sizes and shapes, very efficiently.  }
\label{F: two examples in counting cells-cluster-codis} 
\end{figure}

{\bf Counting Hela Cells:} In Figure \ref{F: examples in hela}, we present  three cell images from the Hela Cells Data set introduced in \cite{arteta2012learning}. These images have a low percentage of overlapping cells where cells are separated by the bright edge boundaries.
The  results obtained by CODI-M  and CODI-S  methods are compared to Class Agnostic method 
\cite{lu2018class} and  Singletons \cite{arteta2016detecting} methods.  In \cite{arteta2016detecting}, a tree-structured discrete graphical model is used to classify non-overlapping regions by optimizing of a classification score. The detection is learned within the structured output SVM framework through dynamic programming on a tree structured region graph. In \cite{lu2018class}, the problem is formulated as a matching problem and the image self similarity property is used. Then a Generic Matching Network is trained using a few labeled examples. Figure \ref{F: examples in hela} shows that both
CODI-S  and CODI-M  methods are comparable to \cite{arteta2016detecting} and  \cite{lu2018class} without any need of learning/training.

\begin{figure}
\begin{center}
\begin{tabular}{cccc} 
 (a) & (a1) \cite{lu2018class} & (a2) CODI-S  & (a3) CODI-M    \\
\includegraphics[align = c,height =1.1 in]{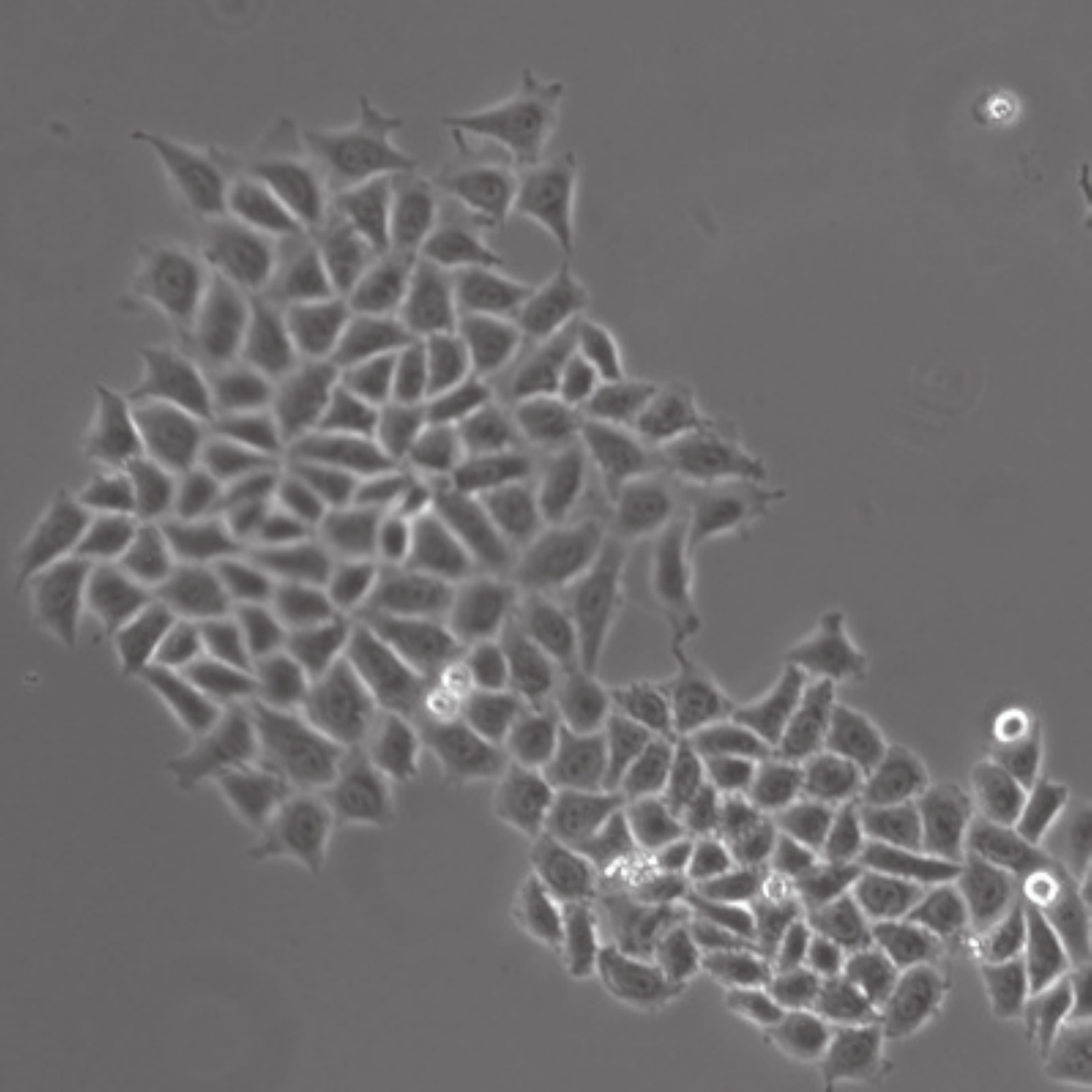}&
\includegraphics[align = c,height = 1.1in]{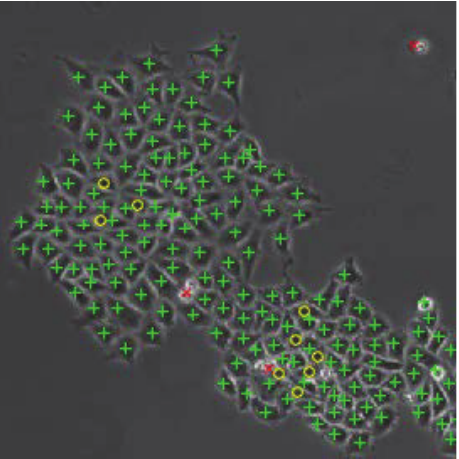} &
\includegraphics[align = c,height = 1.2in]{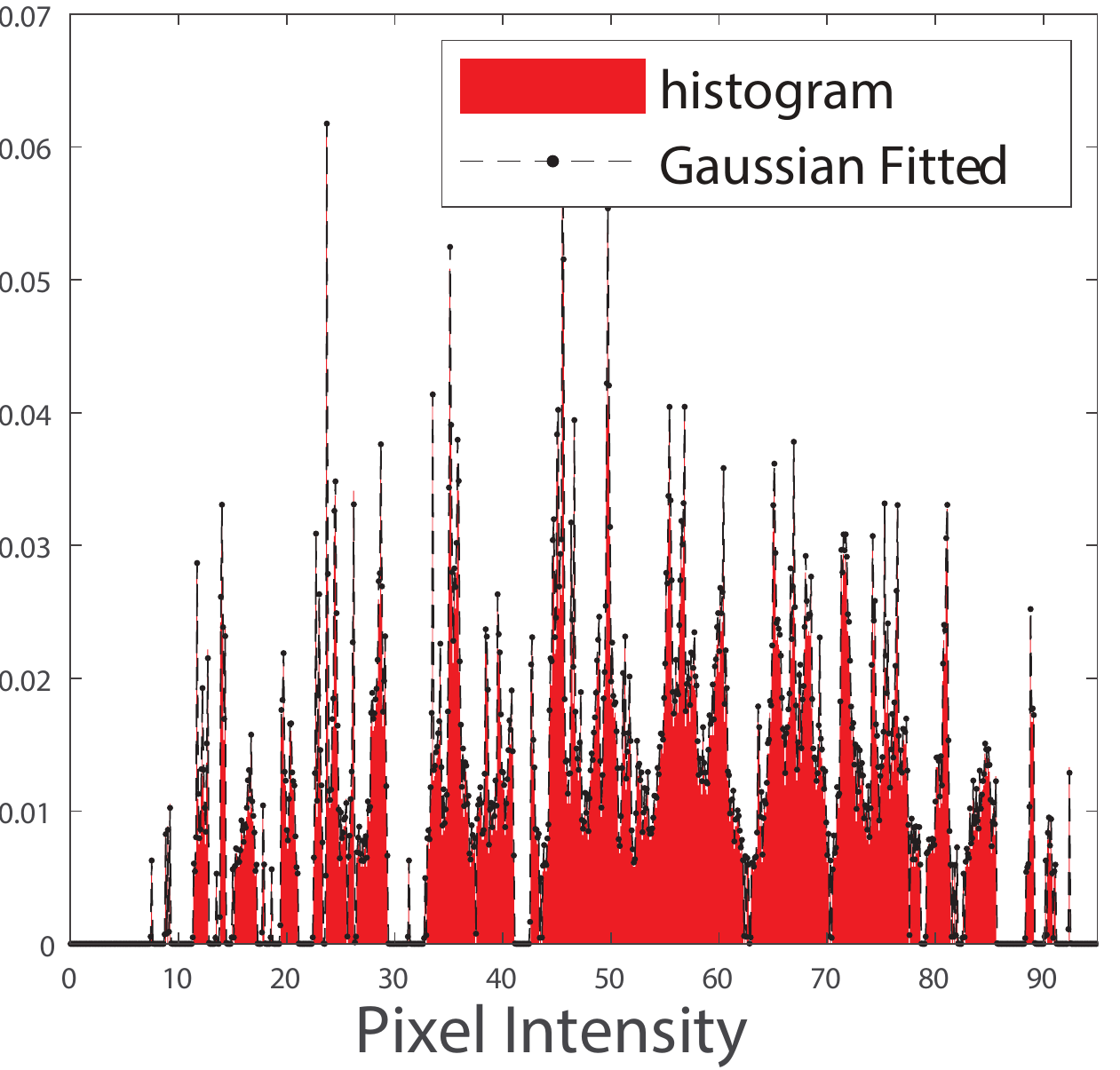}  &
\includegraphics[align = c,height = 1.1in]{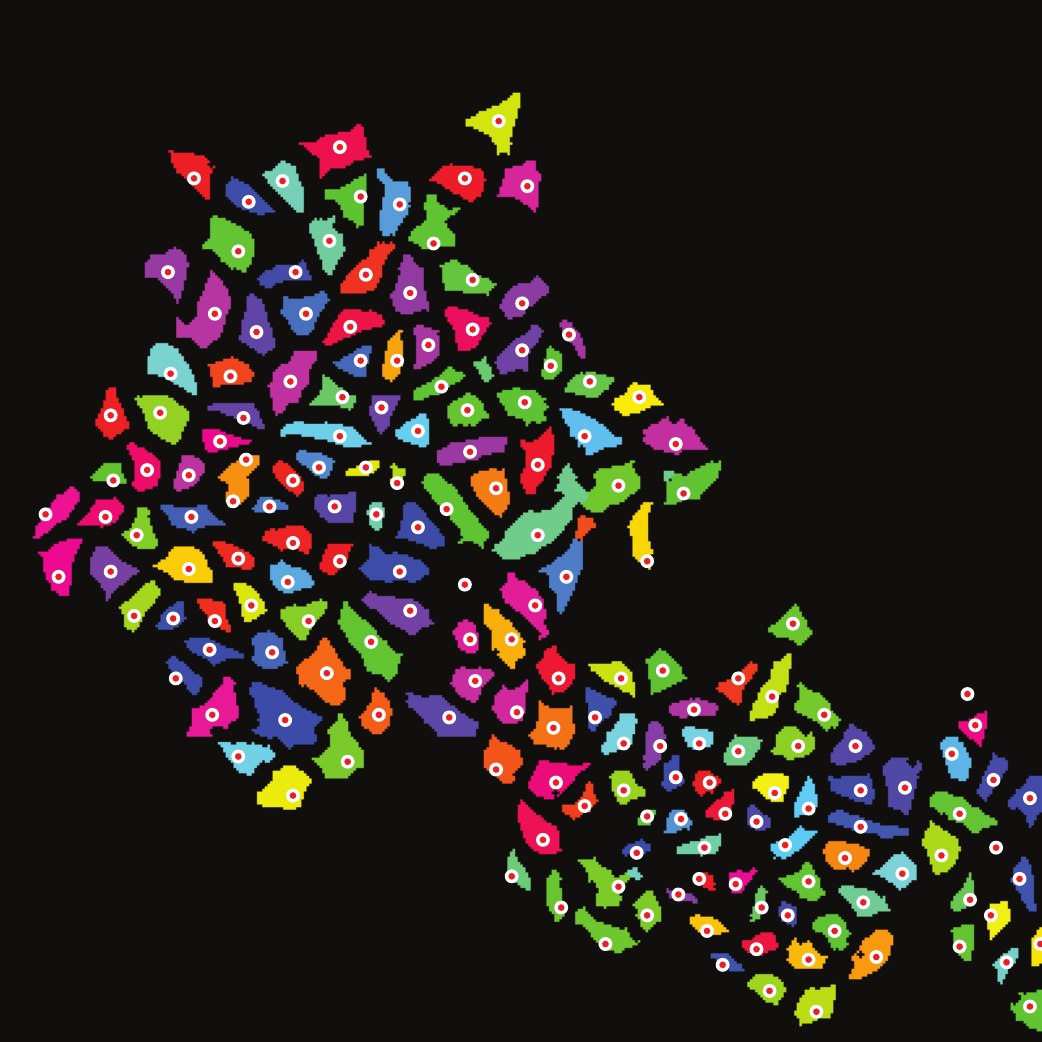} \\
  (b) & (b1)  \cite{lu2018class} & (b2) CODI-S  & (b3) CODI-M  \\
\includegraphics[align = c,height = 1.1 in]{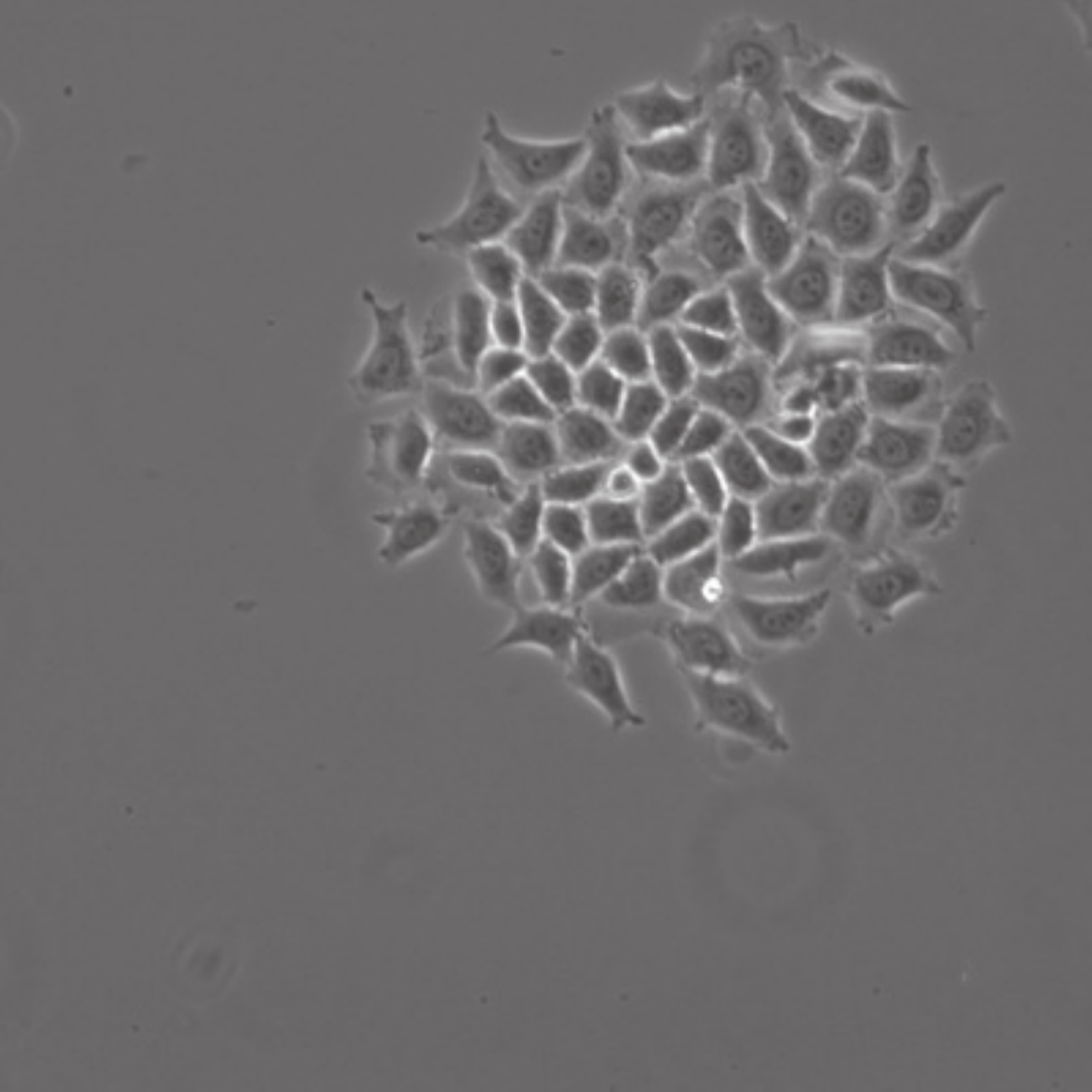}&
\includegraphics[align = c,height = 1.1in]{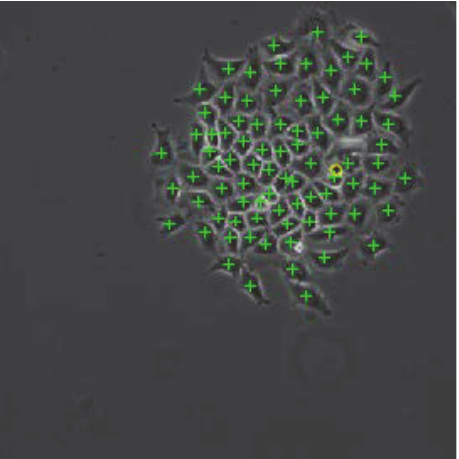} &
\includegraphics[align = c,height = 1.2in]{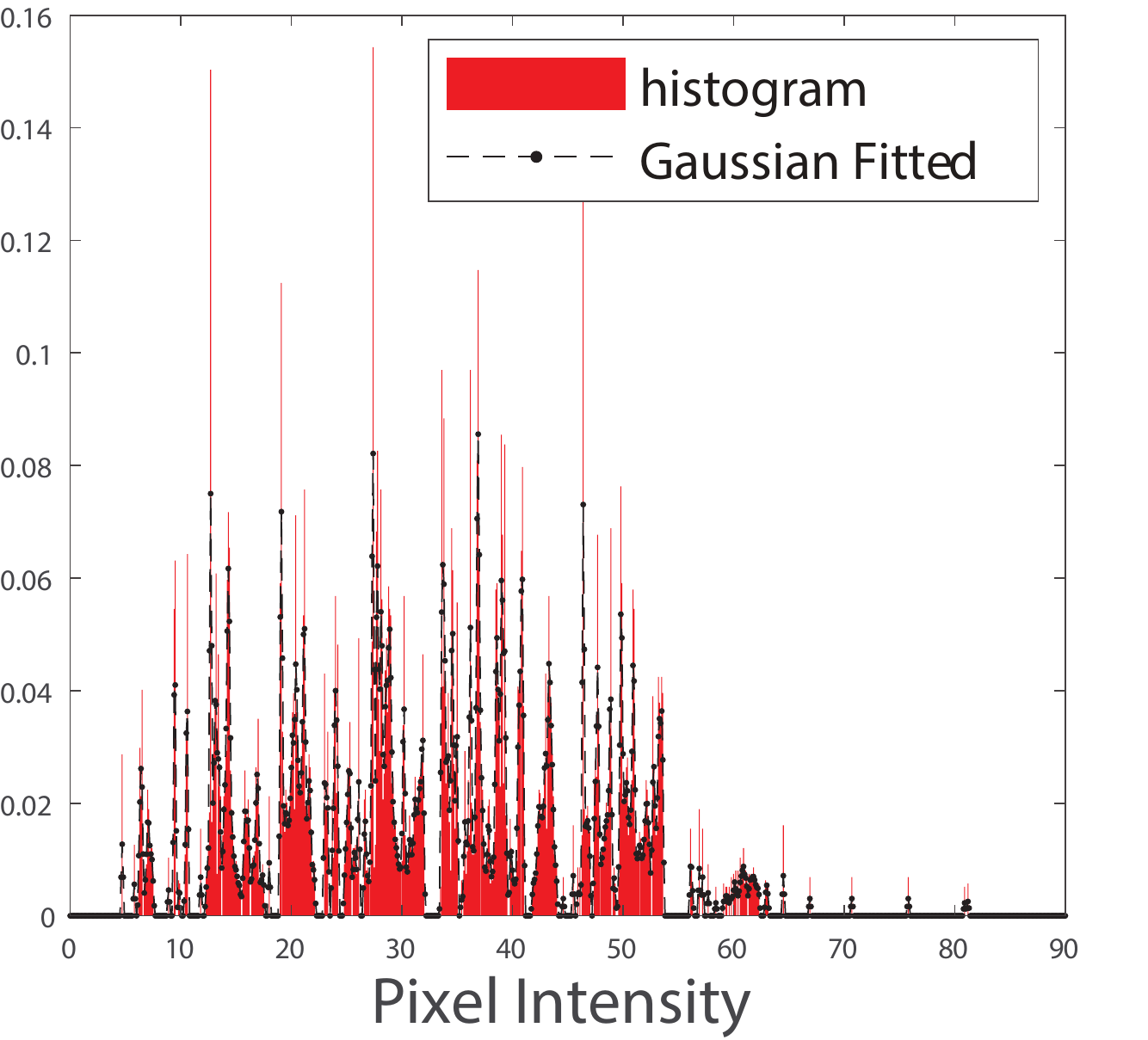}  &
 \includegraphics[align = c,height = 1.1in]{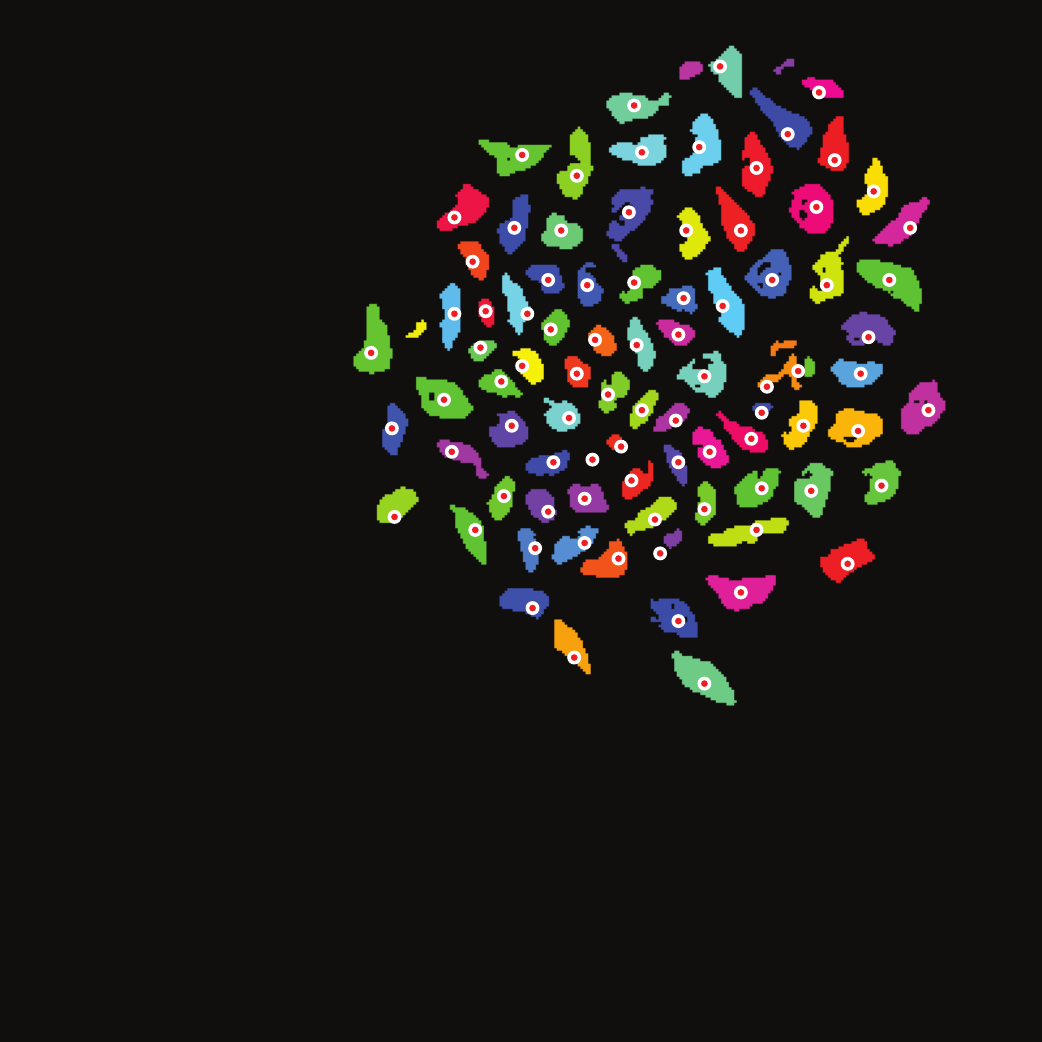}\\
  (c) &  (c1)  \cite{arteta2016detecting}  & (c2) CODI-S & (c3) CODI-M  \\
\includegraphics[align = c,height =1.1in]{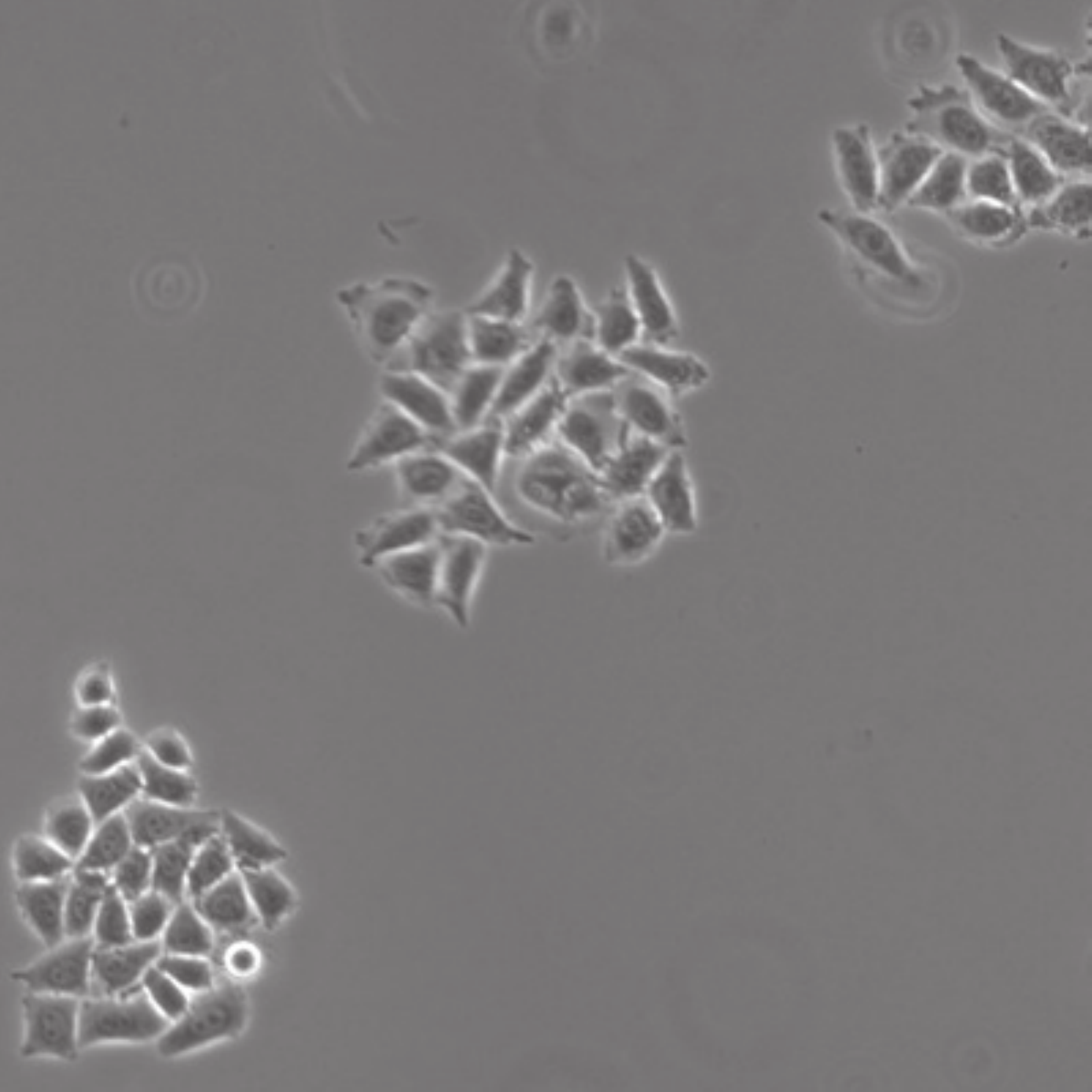}&
 \includegraphics[align = c,height = 1.1in]{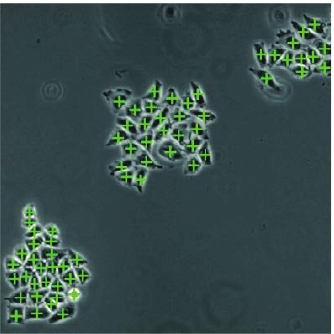}&
 \includegraphics[align = c,height = 1.2in]{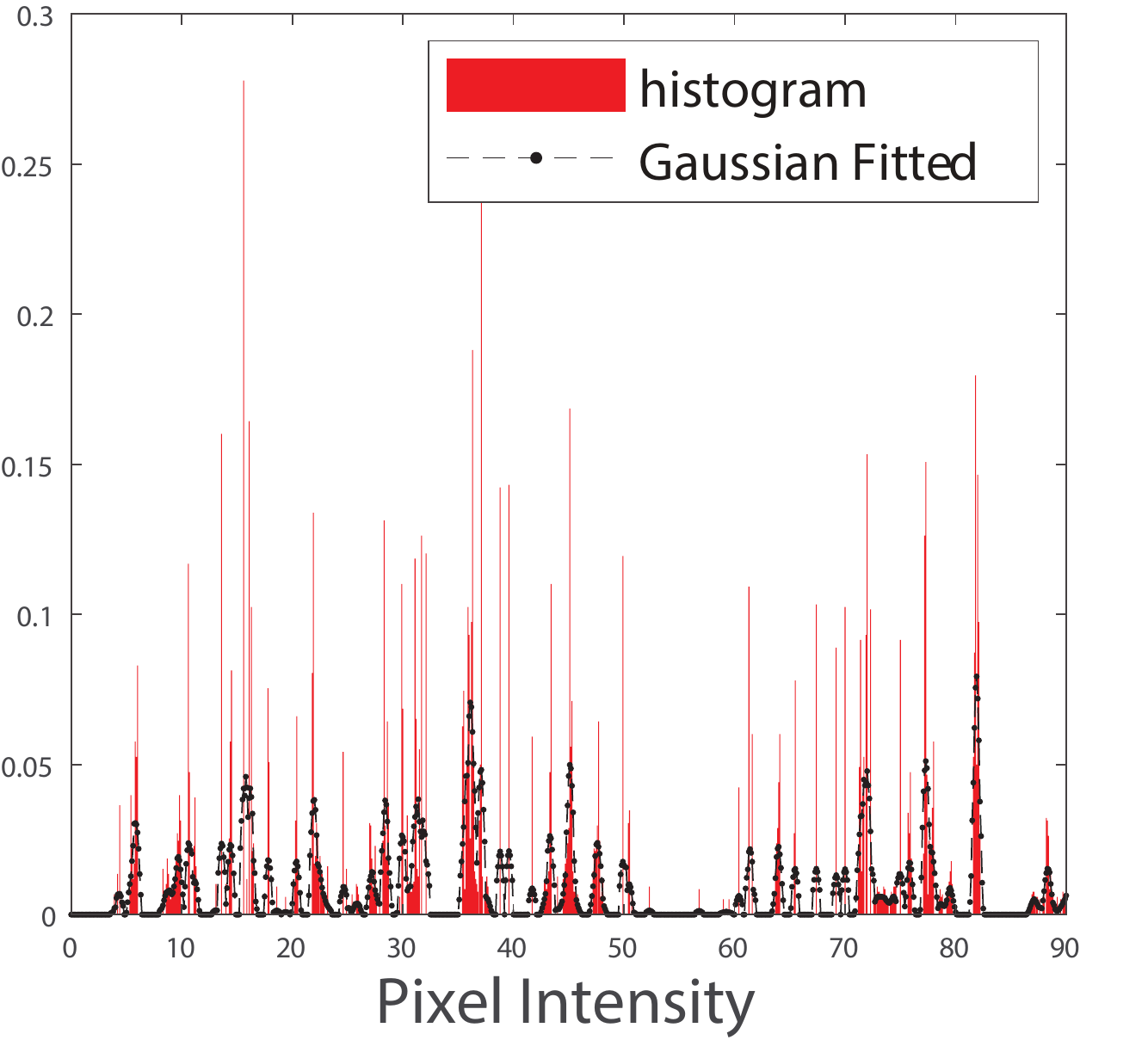}&
  \includegraphics[align = c,height = 1.1in]{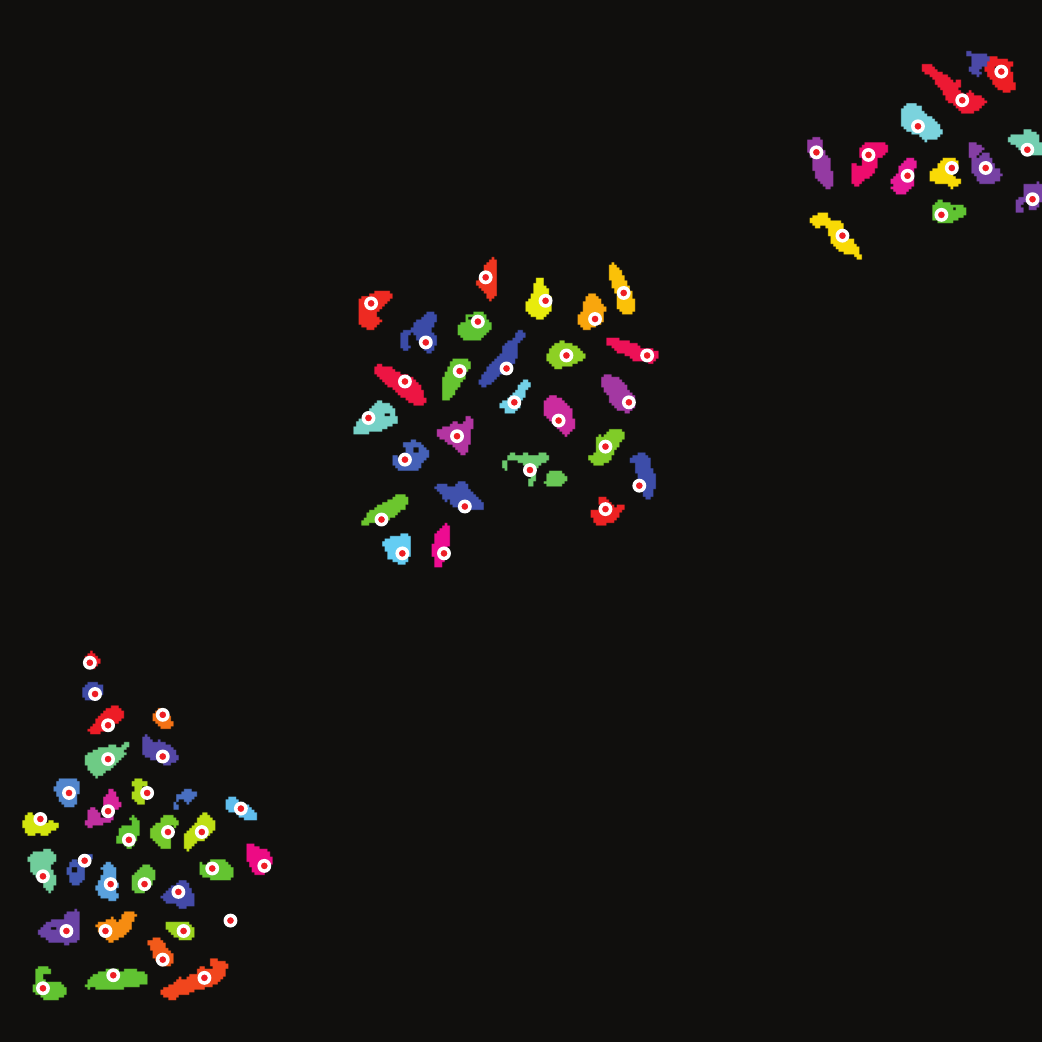}\\ & & & \\
\end{tabular}
\begin{tabular}{cccrrr} 
\toprule
Hela Cells & Size & Ground truth     & CODI-S   & CODI-M  & Others \\
\midrule
 (a) & $400\times 400$ & 177 &  177 (0.95s) & 175 (5.34s)  &171 \cite{lu2018class}  \\
(b) & $400\times 400$  & 85 &  85 (2.06s)  &  88 (3.07s) & 84  \cite{lu2018class} \\
(c)& $400\times 400$ & 67  & 65 (1.95s)  &68 (1.39s) & 67  \cite{arteta2016detecting} \\ 
      \bottomrule
    \end{tabular}
    \end{center}
\caption{[Hela cell counting] Hela cell images from \cite{arteta2012learning}.  CODI-S and CODI-M give comparable counting results to \cite{lu2018class} and \cite{arteta2016detecting}.   In the parenthesis, we show the CPU times for each computation.  CODI-M experiments are performed 5 times, and the best results are presented here, while all comparisons are given in Table \ref{F: statistical results on Hela data set}.}
\label{F: examples in hela}
\end{figure}

The statistical counting results on the whole Hela Data set, containing 11 test images, are given in Table \ref{F: statistical results on Hela data set}.  The comparisons are made with 
Singletons \cite{arteta2016detecting}, 
Full system w/o surface \cite{arteta2016detecting}, 
and Class Agnostic method \cite{lu2018class},
which their data are taken from \cite{lu2018class}. 
All these methods require training and learning procedure. 
The CODI-S  and CODI-M  do not require any training thus to obtain some statistics,
we exploit CODI-S  once and CODI-M  five times on each image in the training data set,
containing 11 images. For $g$, we implemented a threshold with $\chi_{t \leq 100}$, contrast enhancement \cite{zuiderveld1994contrast}, a threshold with $\chi_{t \leq 70}$, and a dilation step  with  structuring element parameters to be (disk,1,4).  We let $0.1\le \sigma\le 2.7$ and $1\le r \le 10$ for CODI-S  and 
$\epsilon =1.5$ and ${\rm MinPts}=20$ for CODI-M  method.

For numerical comparison measures, we use Mean Average Error (MAE) = $\frac{1}{n}\sum_{i=1}^n |y-y^*|$ for the number of objects.
Here $y^*$ represents the ground truth counting number,
$y$ is the computed number, and $n$ is the number of images in the test set.  
Note that a lower MAE is preferable.
In Table \ref{F: statistical results on Hela data set}, we observe that CODI is comparable to the exsiting method, but without any training. CODI-M is also able to track the objects location in the image. 

\begin{table}
\begin{center}
  \begin{tabular}{lccccc}
\multicolumn{6}{c}{Comparison results on Hela Cell Dataset}\\
\midrule
Methods 
 & \multicolumn{1}{c}{Singletons \cite{arteta2016detecting}} 
 & \multicolumn{1}{c}{Full system w/o surface \cite{arteta2016detecting}}
 & \multicolumn{1}{c}{Class Agnostic \cite{lu2018class}} 
  & \multicolumn{1}{c}{CODI-S} 
 & \multicolumn{1}{c}{CODI-M} \\
 \midrule
MAE  & $2.35 \pm 0.67 $ &  $3.84\pm 1.44$ &  $3.53 \pm 0.18 $  & 2.36 
 & $3.32 \pm 0.28$ \\
\bottomrule
\end{tabular} 
\end{center}
\caption{[Hela cell counting] 
Comparison results on 11 Hela images.  
We let $0.1\le \sigma\le 2.7$ and $1\le r \le 10$ in CODI-S,
and $\epsilon =1.5$ and ${\rm MinPts}=20$ in CODI-M.
CODI is comparable to the existing methods without any training process.  
CODI-M experiments are performed 5 times, and the mean and standard deviation of MAE are presented in the table.}
\label{F: statistical results on Hela data set}
\end{table}

{\bf Counting seamless leaf patterns: }
We consider a  seamless leaf image with lace veins patterns in
Figure \ref{F: veins} (a). The manual counting give the number between $[70,72]$.  
For $g$, we use the edge detecting function $\bar{g}$ in (\ref{e:g}) where $\bar{g} > 0.7$  is considered as 1 as the binary output.  In this example,  CODI-S  and CODI-M  find 72 and 70 objects respectively.

\begin{figure}
  \begin{center}
  \begin{tabular}{cccc} 
   (a)  Given image  & (b) Edge function &
 (c) CODI-S & (d) CODI-M\\
\includegraphics[align = c,height = 1in]{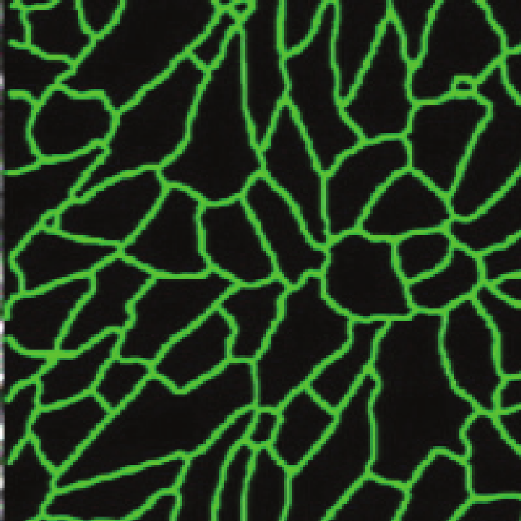}  &
\includegraphics[align = c,height = 1in]{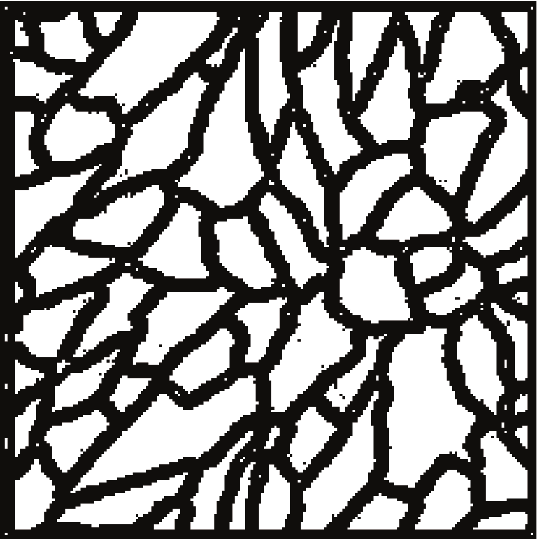} &
\includegraphics[align = c,height = 1.1in]{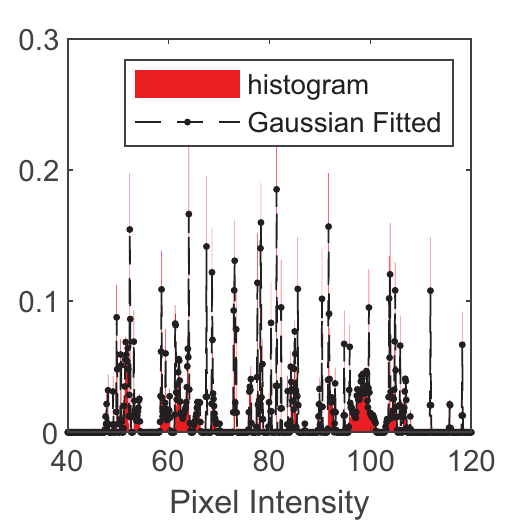}  &
 \includegraphics[align = c,height = 1in]{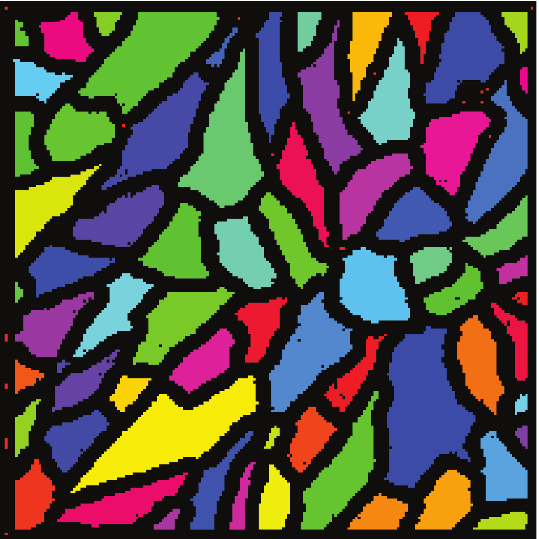} \\ & & & \\
\end{tabular}
 \begin{tabular}{ccccc} 
\toprule
Given image & Manual Count  & CODI-S  & CODI-M  \\
\midrule
(a)  & [70,72] & 72 (2.16s) & 70 (2.81s)  \\
      \bottomrule
    \end{tabular}
\end{center}
\caption {[Counting seamless leaf patterns ]
(a) Given image of   where manual counting is given between 70 and 72.
(b) The edge function $\tilde{g}$.
(c) CODI-S  counts 72 leafs.
For 20 CODI-M experiments, the counts varies between [68,70] and 13 out of 20 trials results in count 70.  The subtle uncertainty comes from the the small objects in the original image. The average cpu time is 2.81 second. Figure (d) shows one representative result of CODI-M. 
}
\label{F: veins} 
\end{figure}


{\bf Arabidopsis plant leafs: }
We consider an image of Arabidopsis plant from  MSU-PID dataset \cite{cruz2016multi} shown in 
Figure \ref{F: leaf} (a).  In the ground truth image in (b) shows 10 leafs. 
We compare our methods with \cite{ayalew2020unsupervised}, a Domain-Adversarial Neural Network (DANN) where the counting is done by the density map estimation shown in Figure \ref{F: leaf}(c). For $g$, we used a threshold with $\chi_{t > 137}$. To further separate the edges between the overlapping leaves, an edge detecting function $\bar{g}$ in (\ref{e:g}) where $\bar{g} > 0.8$ is considered as 1 as the binary output. It finds 8 leafs, while CODI-S  and CODI-M  find 9 and 10 leafs, respectively. 

\begin{figure}
  \begin{center}
  \begin{tabular}{ccccc} 
   (a)  Given image  & (b)  Ground truth &
 (c) \cite{ayalew2020unsupervised} & (d) CODI-S & (e) CODI-M\\
\includegraphics[height = 1in]{figures/leaf_original_image}  &
\includegraphics[height = 1in]{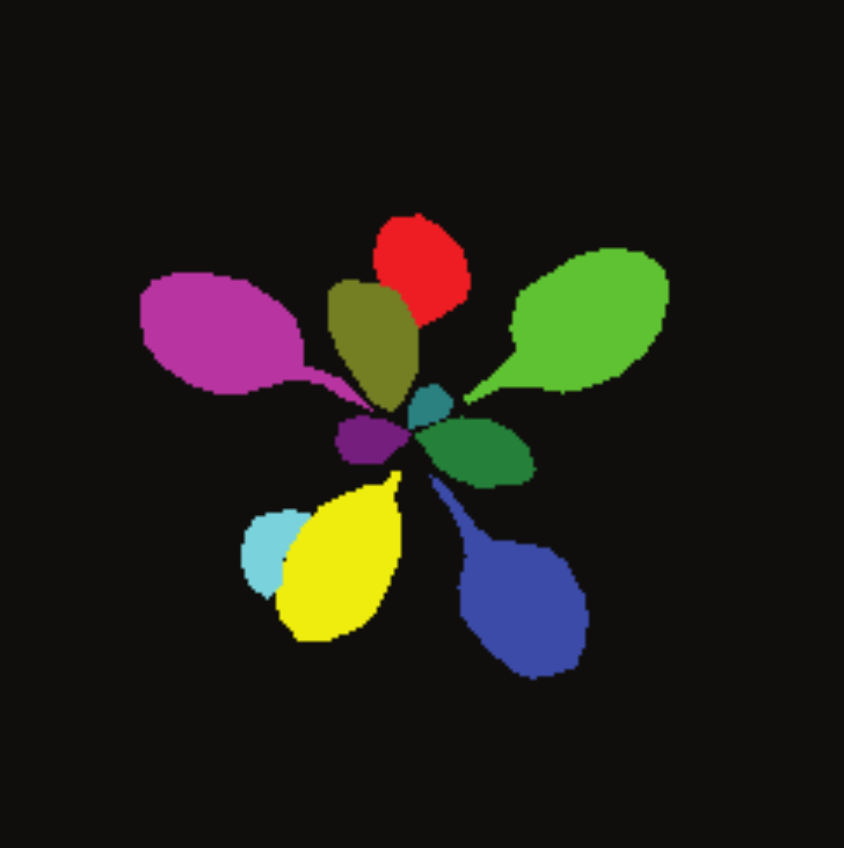} &
\includegraphics[height = 1in]{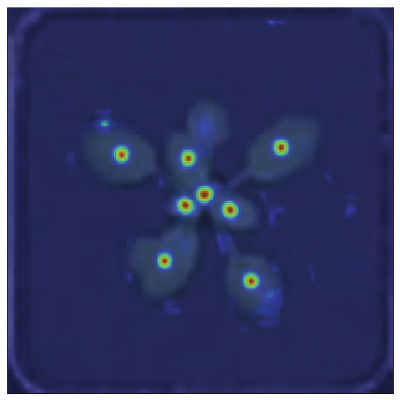}  &
 \includegraphics[height = 1in]{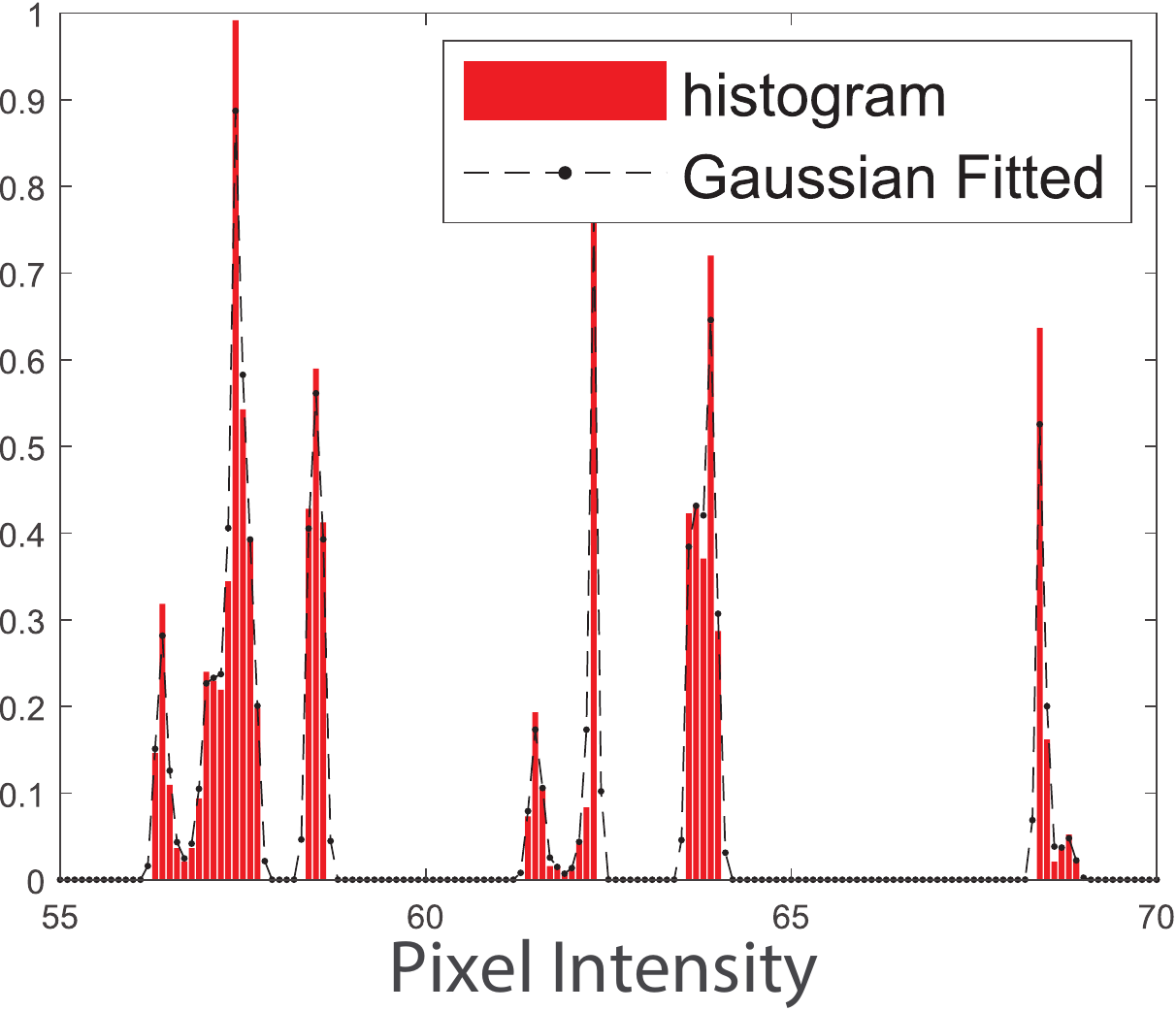}& 
  \includegraphics[height = 1in]{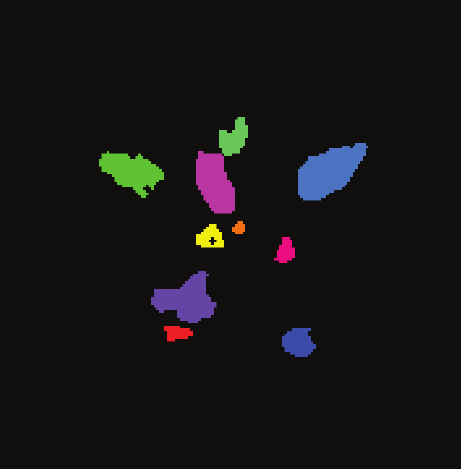} \\ &&&&\\
\end{tabular}
    \begin{tabular}{ccccc} 
\toprule
Given image & Manual Count  & CODI-S  & CODI-M & Other \\
\midrule
from \cite{cruz2016multi}   & 10 & 9 (1.33s) & 10  (0.07s) & 8  \cite{ayalew2020unsupervised}\\
      \bottomrule
    \end{tabular}
\end{center}
\caption {[Arabidopsis plant leaf counting]
(a) Given image of Arabidopsis plant \cite{cruz2016multi}.
(b) The ground truth image in \cite{cruz2016multi} showing $10$ leafs.
(c) The  density map estimation \cite{ayalew2020unsupervised}, showing 8 leafs.
(d) CODI-S  counts 9 leafs.
(e) CODI-M  counts 10 leafs.
 Experiments are performed 20 times on CODI-M, where 10 out of 20 trials results in count between 9 and 11.  The subtle uncertainty comes from the the delicate boundaries between the leaves in the original image.  The average cpu time is 0.07 second. Figure (e) shows the best results among 20 CODI-M experiments. 
}
\label{F: leaf} 
\end{figure}

{\bf Agriculture and fruits: } We consider agriculture images  in Figure \ref{F: fruits}:
(a) an apple tree ($594\times 800$) and (b) a bunch of cherries ($800\times 800$).
These are color images where the fruits are red and the rest of image 
is roughly green. For $g$, we subtracted the green channel (the second dimension)  from the red channel (the first dimenstion) followed by a thresholding $\chi_{t > 80}$, $\chi_{t > 110}$ for (a) and (b) respectively.
Since there are many overlapping objects, a rough estimate of manual counts are provided in form of intervals.
We apply the proposed methods to count the number of apples in (a)
and cherries in (b). 
Figure \ref{F: fruits} shows that the proposed methods are able to find a correct estimation for the 
number of fruits.  

\begin{figure}
  \begin{center}
\begin{tabular}{cc} 
(a) &(b)\\
\includegraphics[align = c, height = 1.5in]{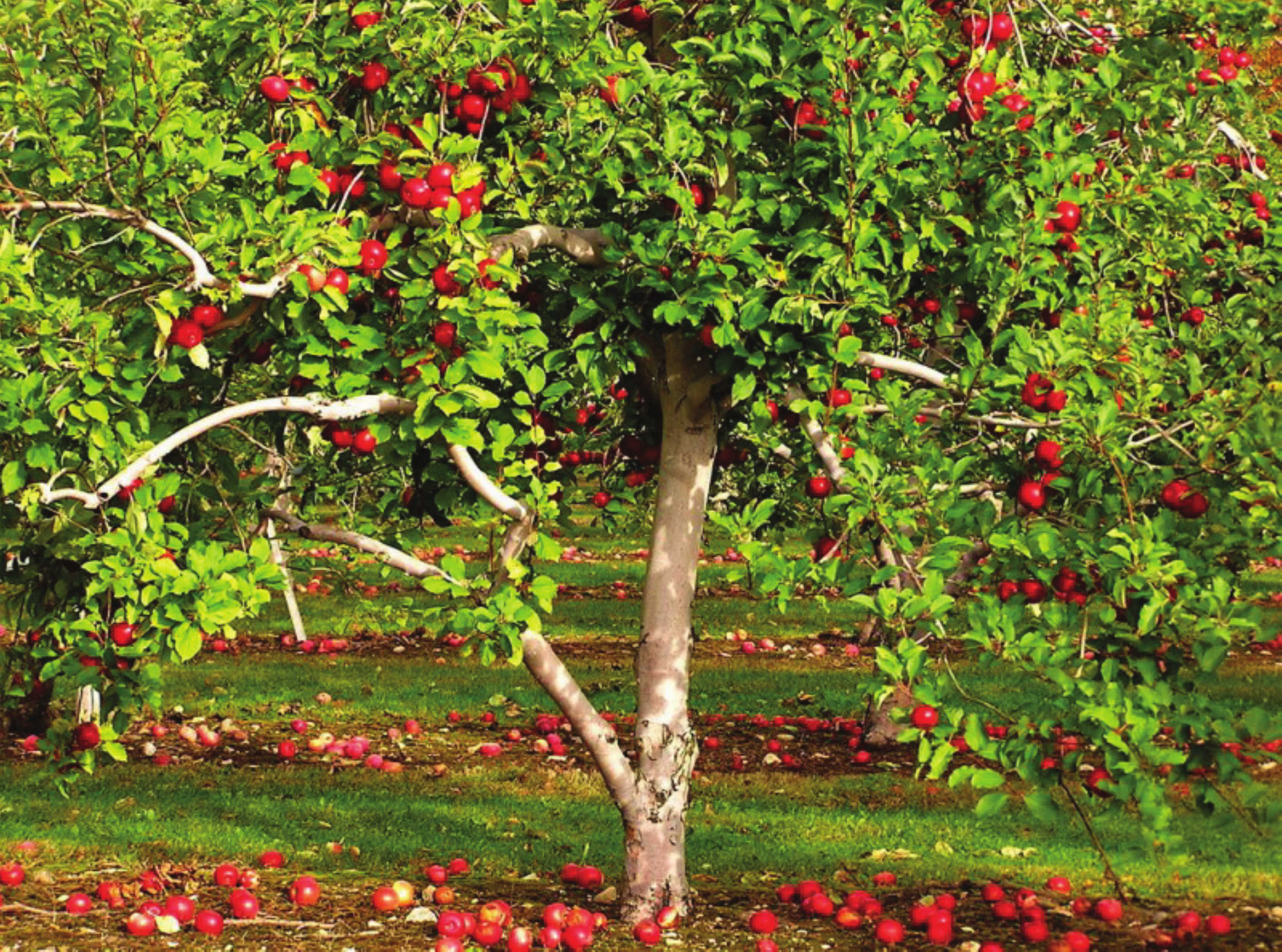}  & 
\includegraphics[align = c, height = 1.5in]{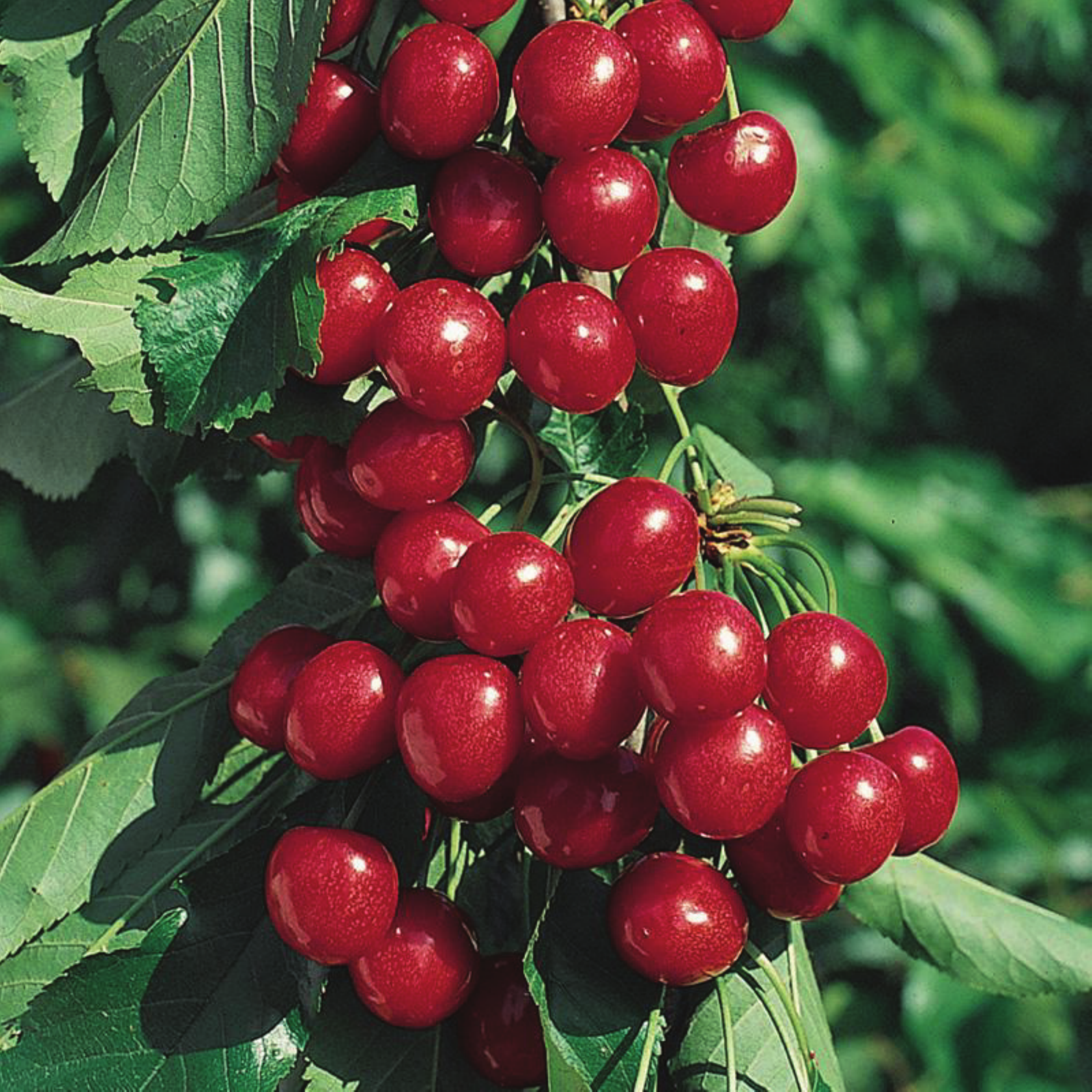}  \\ &\\
\end{tabular}
    \begin{tabular}{cccrr} 
\toprule
Given image & Manual Count &  CODI-S  & CODI-M  \\\midrule
(a)  & [205, 235]  & 202 (7.57s) & $[217,226]$ (4.85s)   \\

(b)   & [27,33]  & 31 (0.25s) & [27,33] (0.07s) \\

 \bottomrule
    \end{tabular}
\end{center}
\caption{[Counting apples and cherries]  (a) An apple tree (b) A bunch of cherries.  
Both CODI-S  and CODI-M find a number within the accepted range.  Experiments are performed 20 times on CODI-M. For (a), the results varies in [208, 226], where 14 out of 20 trials generate results in [217, 226].  For (b) the result varies between [25,40] where 14 out of 20 experiments  results in [27,33]. This result is consistent with  the large quantity of apples in (a) and the unclear boundaries between cherries in (b). The average cpu time is 4.85 and 0.07 for each image respectively.
}
\label{F: fruits} 
\end{figure}


{\bf Objects in the production line: }
The production line images are considered in  Figure \ref{F: industry}: (a) a cart of eggs and (b) a case of soda bottles.
We compare  CODI-M  and CODI-S  with the method in \cite{baygin2018image}.  In \cite{baygin2018image}, the authors considered the segmentation, Gaussian filter, Otsu Thresholding \cite{otsu1979threshold}, Sobel Edge Detection \cite{vincent2009descriptive,shrivakshan2012comparison}, and Hough Circle Transform \cite{cha2016vision,baker2016power}.
For $g$, we used a threshold $\chi_{t > 205}$ for (a), $\chi_{t > 120}$ for (b) and an erosion step on (b)  with  structuring element parameters to be (disk,1,4) to further distinguish the boundary.
Due to the use of Hough Circle Transform, the work \cite{baygin2018image} is geometry dependent. 
Figure \ref{F: industry} shows CODI is comparable while being geometry-free. 

\begin{figure}
  \begin{center}
\begin{tabular}{cccc} 
  (a)& (a1) \cite{baygin2018image}&(a2) CODI-S  &(a3) CODI-M \\
  \includegraphics[align = c, width = 0.18\textwidth ]{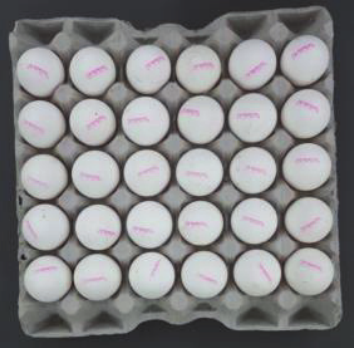}   
  &  \includegraphics[align = c, width = 0.18\textwidth]{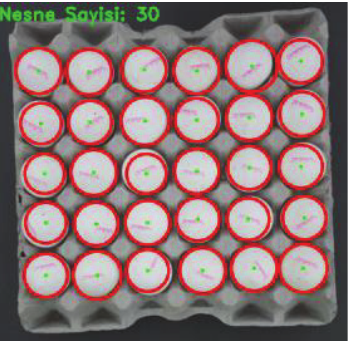}  
    &  \includegraphics[align = c, width = 0.22\textwidth]{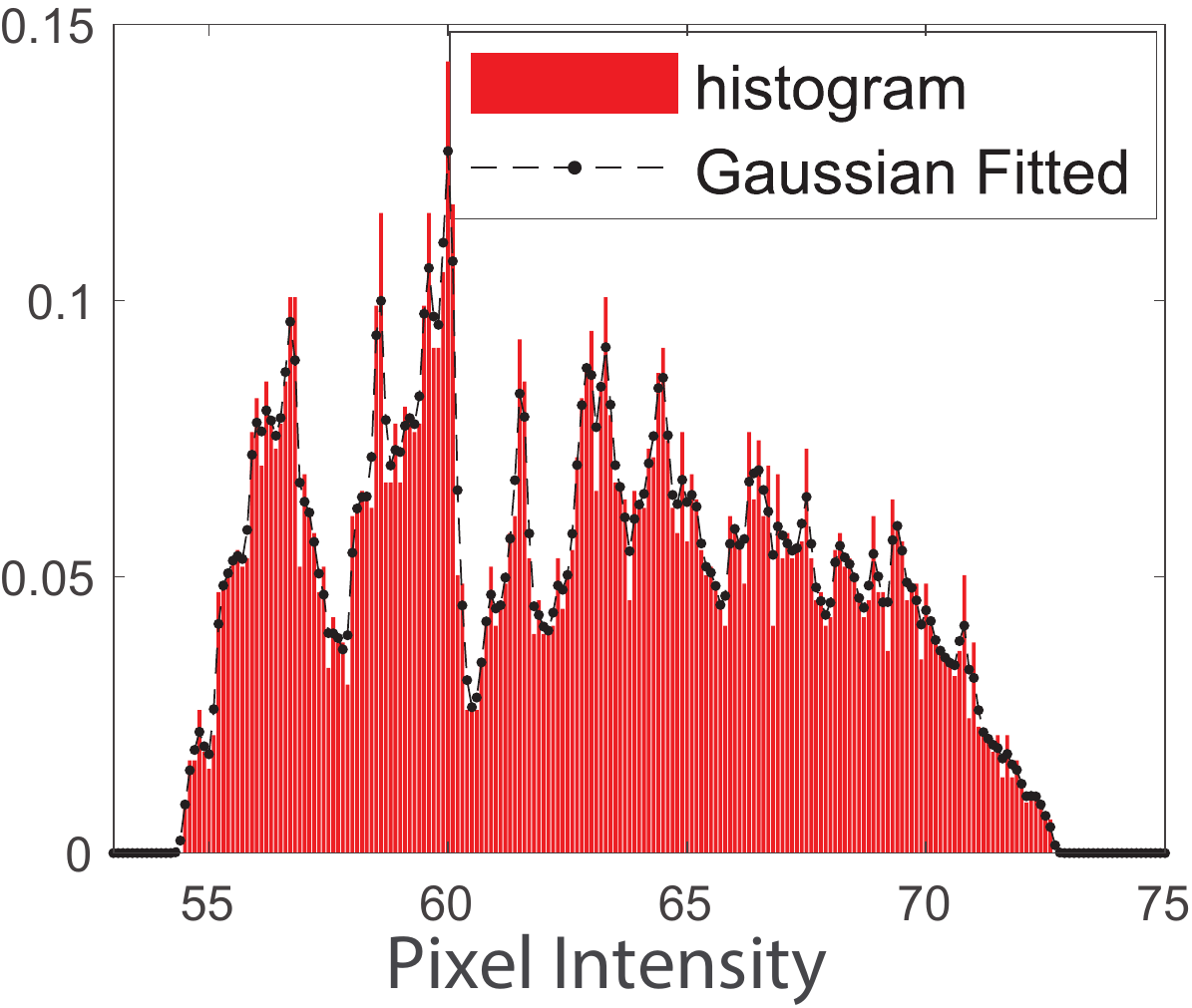}
       & \includegraphics[align = c, width = 0.18\textwidth]{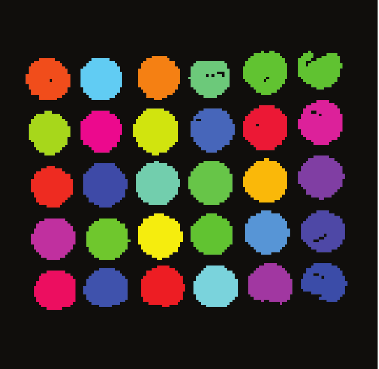}
  \\
    (b)& (b1) \cite{baygin2018image}&(b2) CODI-S  &(b3) CODI-M \\
 \includegraphics[align = c, width = 0.18\textwidth]{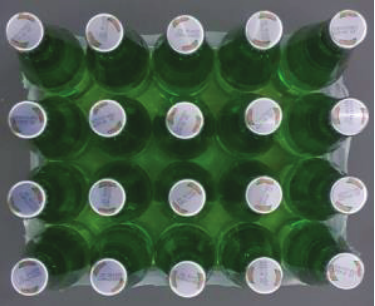}
& \includegraphics[align = c,width = 0.18\textwidth]{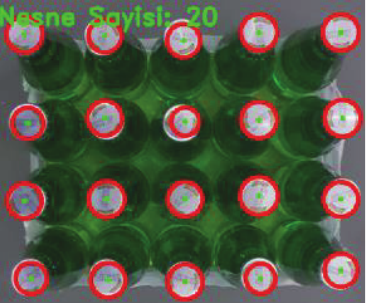}
& \includegraphics[align = c, width = 0.22\textwidth]{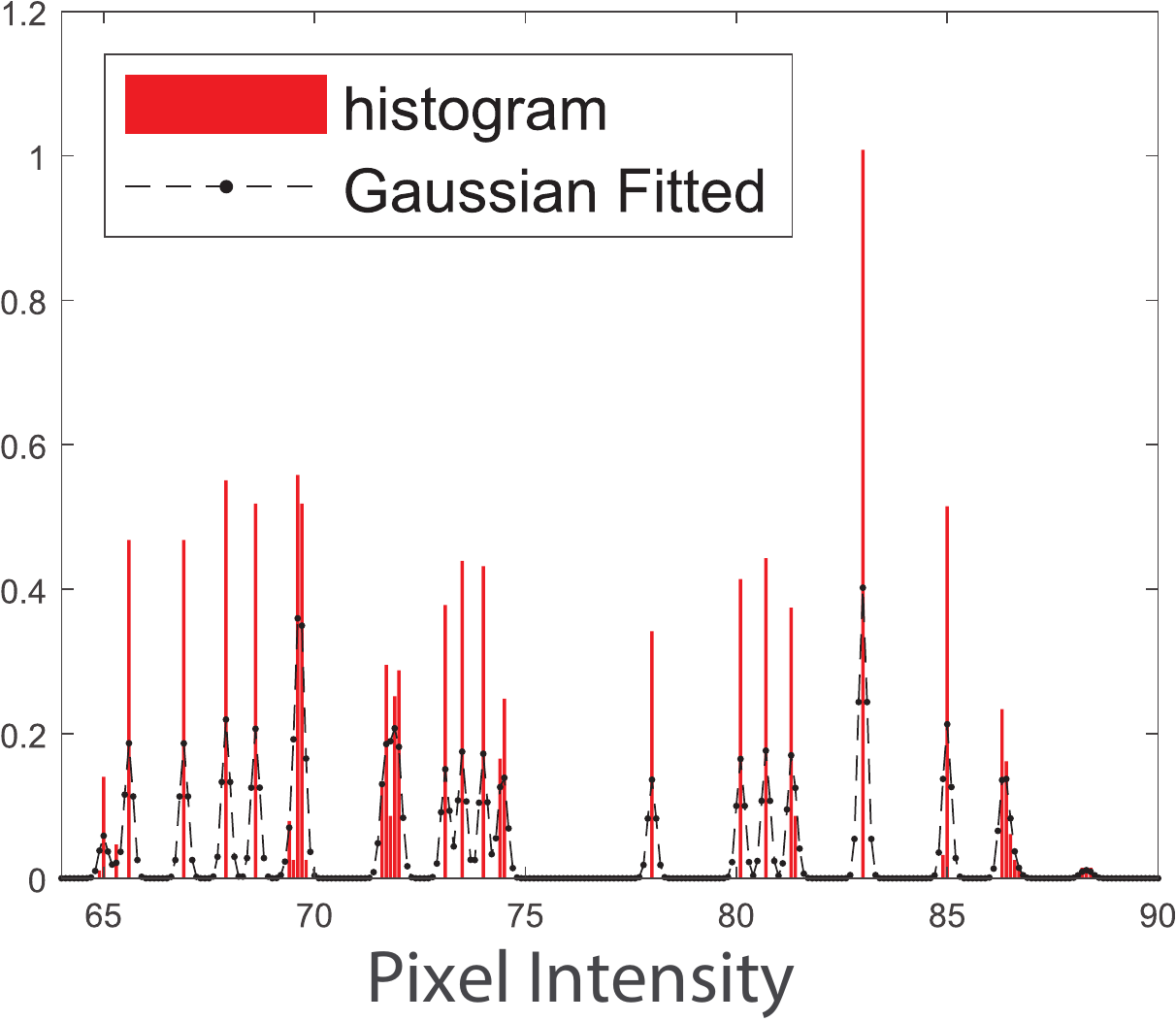}
& \includegraphics[align = c, width = 0.18\textwidth]{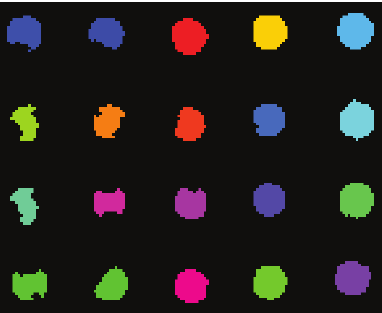}
\\ &&&\\
\end{tabular}
\begin{tabular}{ccccc} 
\toprule
Given image   & CODI-S  & CODI-M  & Others \\
\midrule
(a)   & 29 (0.38s) & $30$ (0.28s) & 30  \cite{baygin2018image}\\
(b)   & 20 (1.08s) &  19 (0.10s)  &   20 \cite{baygin2018image}\\
   \bottomrule
    \end{tabular}
\end{center}
  \caption{[Counting objects in the production lines] (a) A cart of eggs.  (b) A case of soda bottles. 
The second column shows results by \cite{baygin2018image},
the third column by CODI-M, and the forth column by CODI-S.
Experiments are performed 20 times on CODI-M. For (a), the results varies in [29, 31], where 18 out of 20 trials generate 30 as counting result.  For (b) the The result varies in [18, 23] where 18 out of 20 experiments generate results between [18, 20]. CODI gives comparable results to \cite{baygin2018image} without exploiting any geometrical information. }\label{F: industry} 
\end{figure}

{\bf Crowd and Vehicle:}
Figure \ref{F: crowd} (a) displays an image of a concert crowd and  (b) 
shows a GPS image from DOTA  dataset \cite{Ding_2019_CVPR,Xia_2018_CVPR}.  
An estimated number of people and vehicles are obtained by manual counts given in Figure \ref{F: crowd}. For $g$, we used $\chi_{t < 155}$ in (a), $\chi_{t> 220}$ followed by a  dilation step  with  structuring element parameters to be (disk,1,4).
We observe that the proposed CODI-S  and CODI-M  methods give good estimation of the counts.
\begin{figure}
 \begin{center}
\begin{tabular}{ccc} 
\includegraphics[align = c, height = 1.3in]{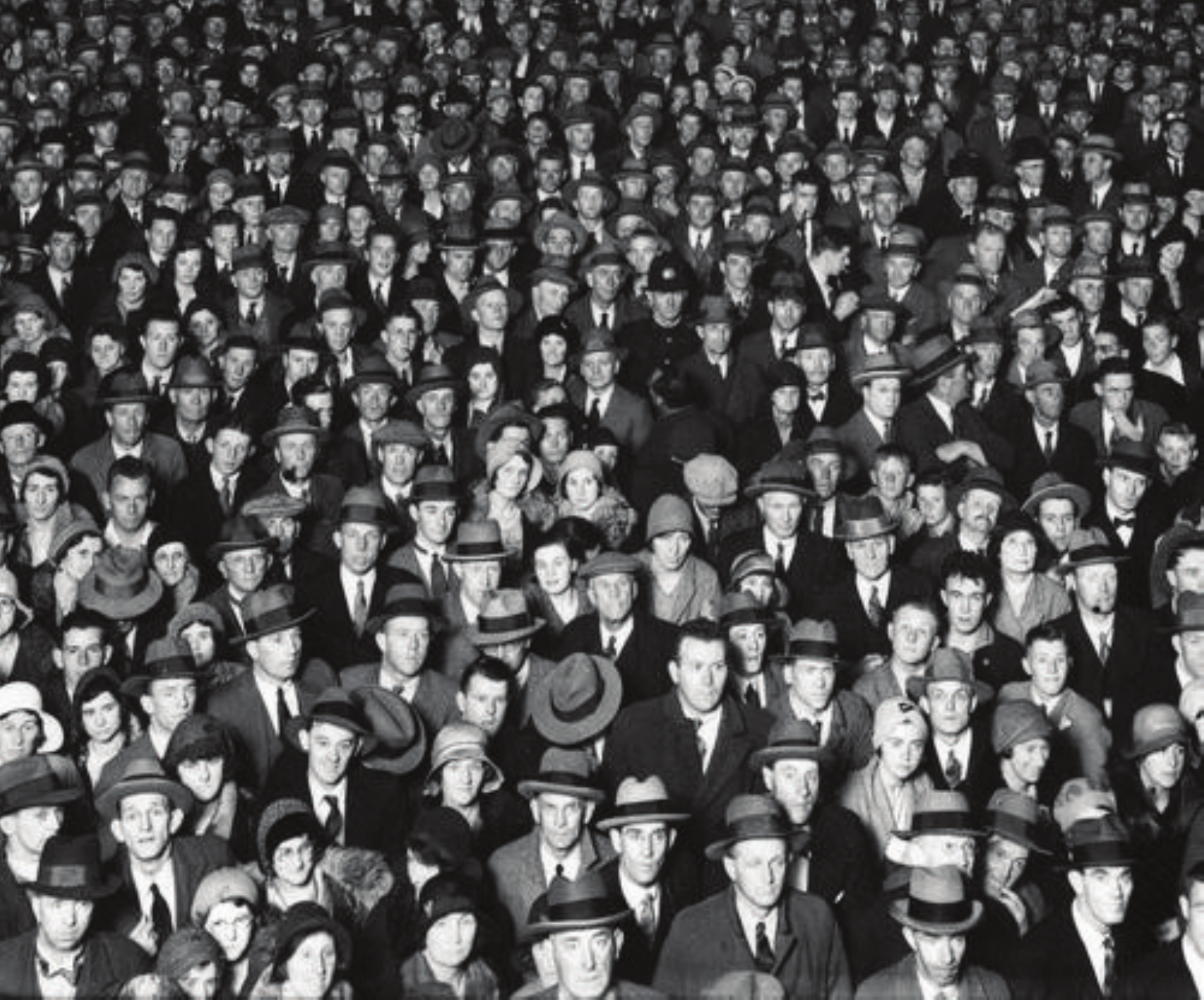} &
\includegraphics[align = c, height = 1.3in]{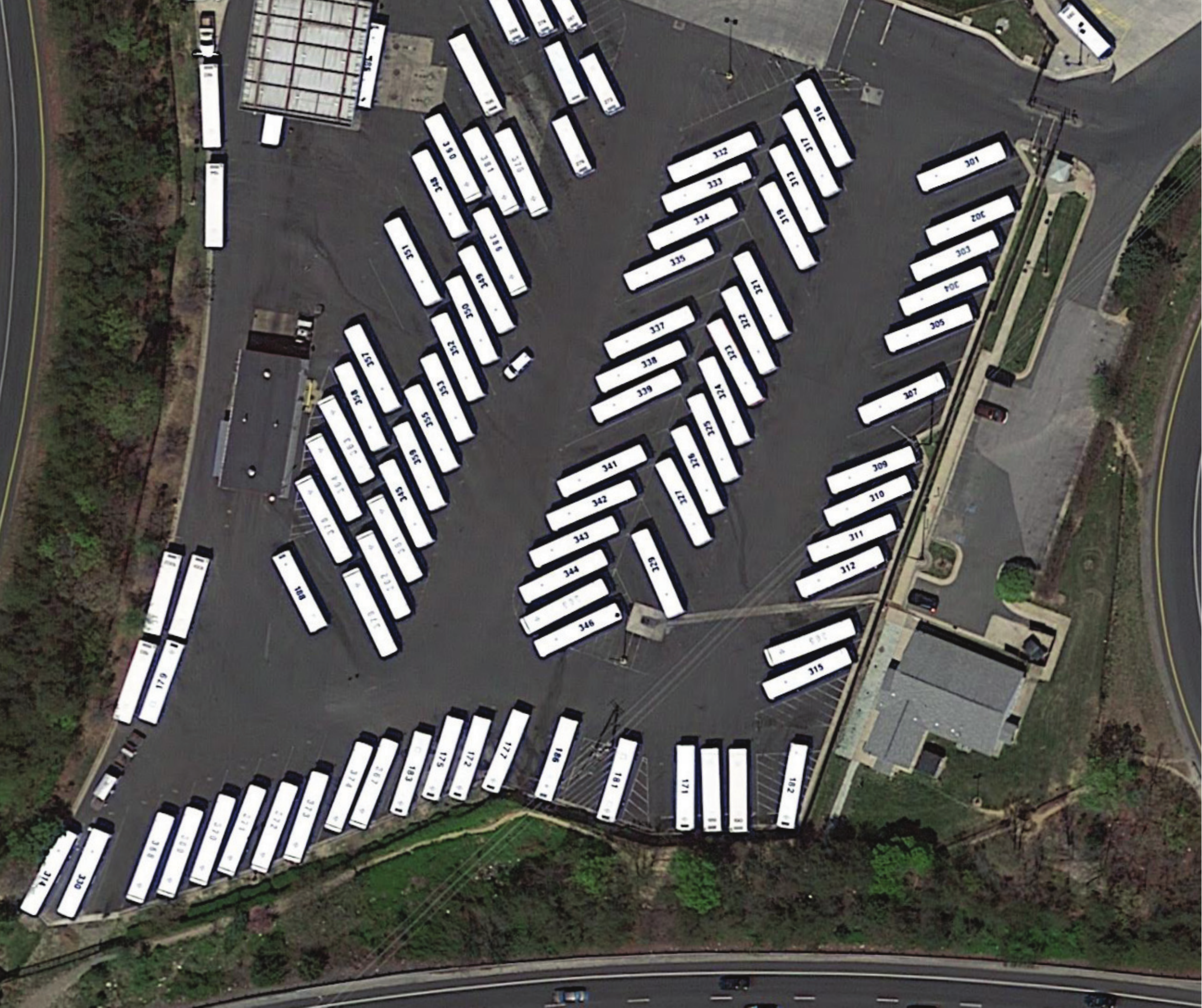}  \\
(a) & (b)\\
\end{tabular}
\end{center}
  \centering
  \begin{minipage}[b]{1.0\linewidth}
    \begin{center}
    \begin{tabular}{cccc} 
\toprule
Given image             &  Manual count & CODI-S & CODI-M \\\midrule
(a) & $292 \pm 10$ & 286 (6.48s) &[285,302] (6.29s)\\
(b) & 94$^\dagger$ &   94 (3.22s)& [93,95] (3.26s)\\
 \bottomrule
    \end{tabular}
    \end{center}
      \end{minipage}
\caption {(a) Concert crowd image (b)  GPS image from DOTA dataset \cite{Ding_2019_CVPR,Xia_2018_CVPR}.
An estimated number of people and vehicles are obtained by manual counting. Experiments are performed 20 times on CODI-M. For (a), the results varies in [283, 315], where 13 out of 20 trials generate result in [285,302].  For (b) the The result varies in [93,96] where 14 out of 20 experiments generate results between [93,95]. The subtle unstable of the result for (a) is due to the large quantity of people in the original image. $^\dagger$ The ground truth of 94 is provided in the dataset.}
\label{F: crowd} 
\end{figure}

\vspace{0.3cm}
In the following, we present a few aspects of CODI.  First, to ensure the quality of diffused index, we present ideas to properly choose the seed location and size. Then, we present the  effect of the downsampling of original image, and finally comment on the choice of parameter in computation.   


\begin{figure}
\begin{center}
\begin{tabular}{ccccc} 
(a) Original Image & (b) & (b1) Count = 2 & (b2) Count = 2 \\
\includegraphics[height = 0.9in]{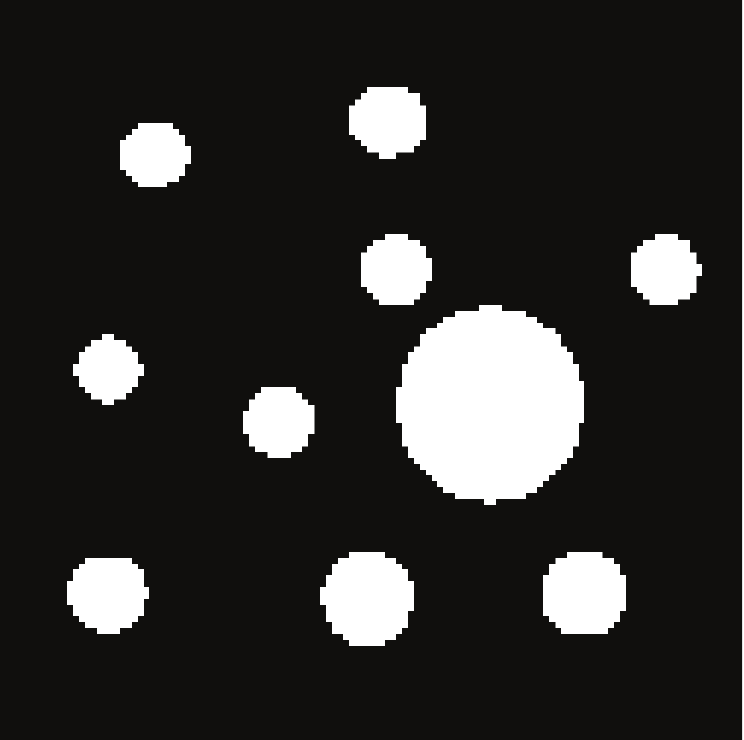}&
\includegraphics[height = 0.9in]{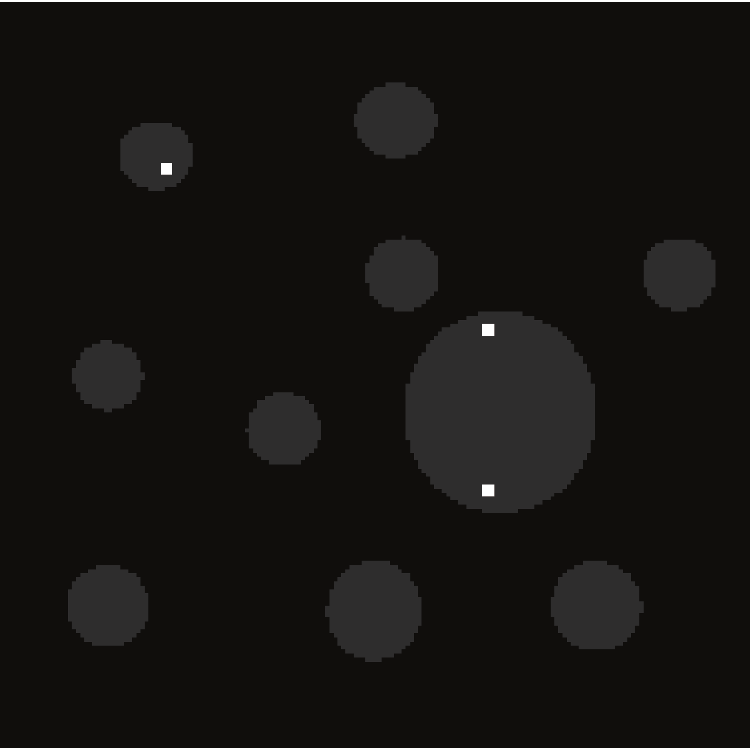}&
\includegraphics[height = 0.9in]{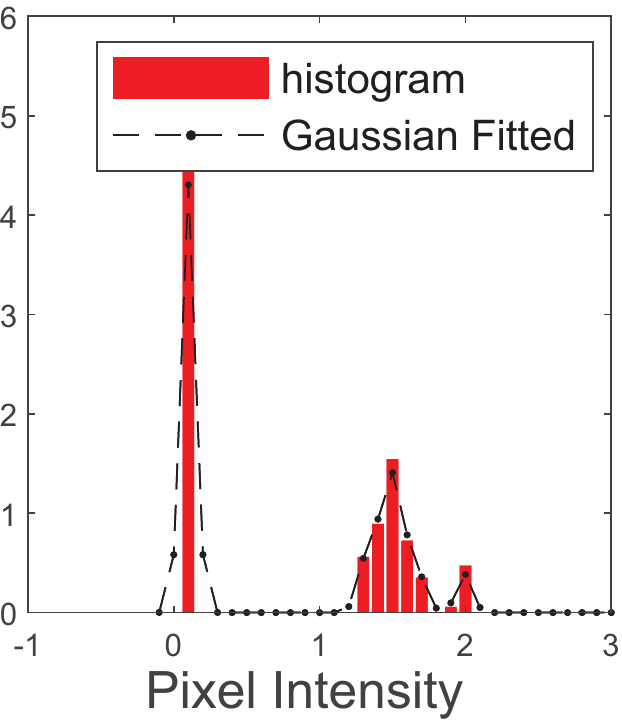}&
\includegraphics[height = 0.9in]{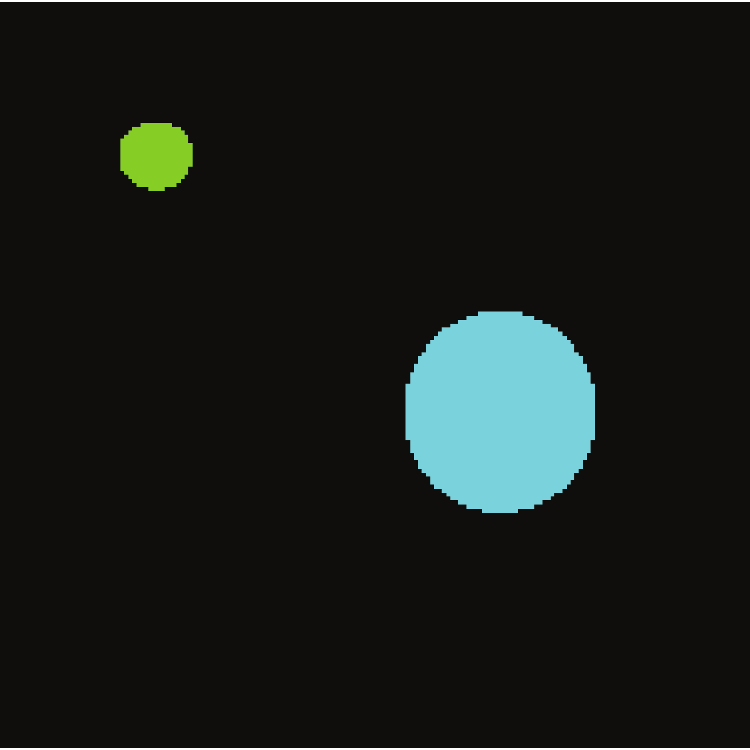}\\
& (c) & (c1) Count = 10 & (c2) Count = 10\\
 &
\includegraphics[height = 0.9in]{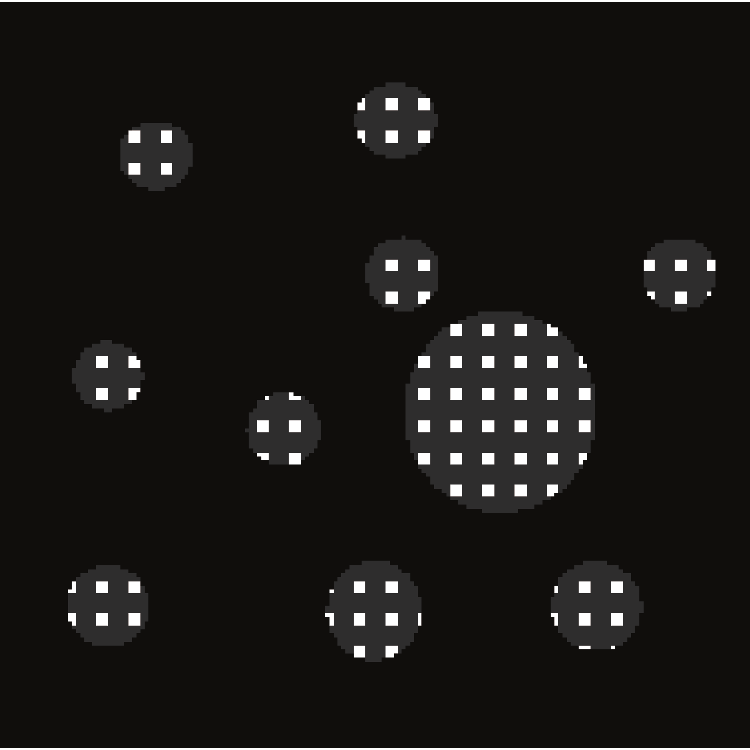}&
\includegraphics[height = 0.9in]{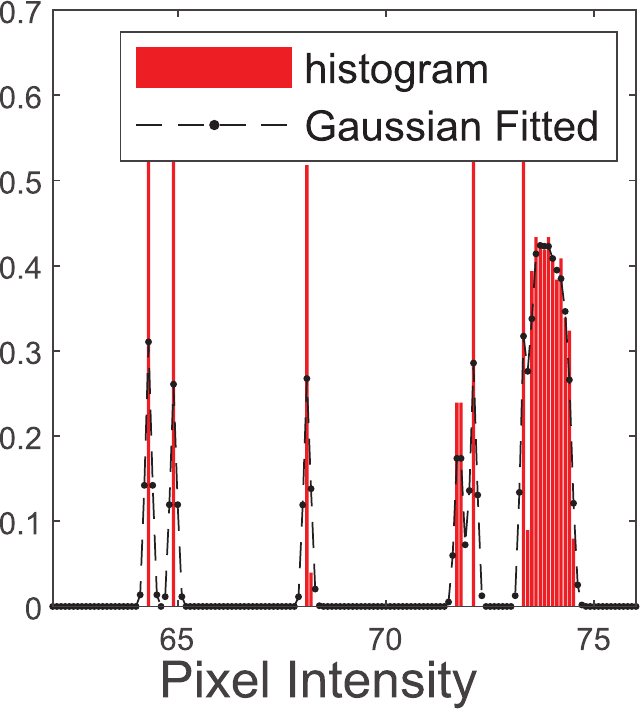}&
\includegraphics[height = 0.9in]{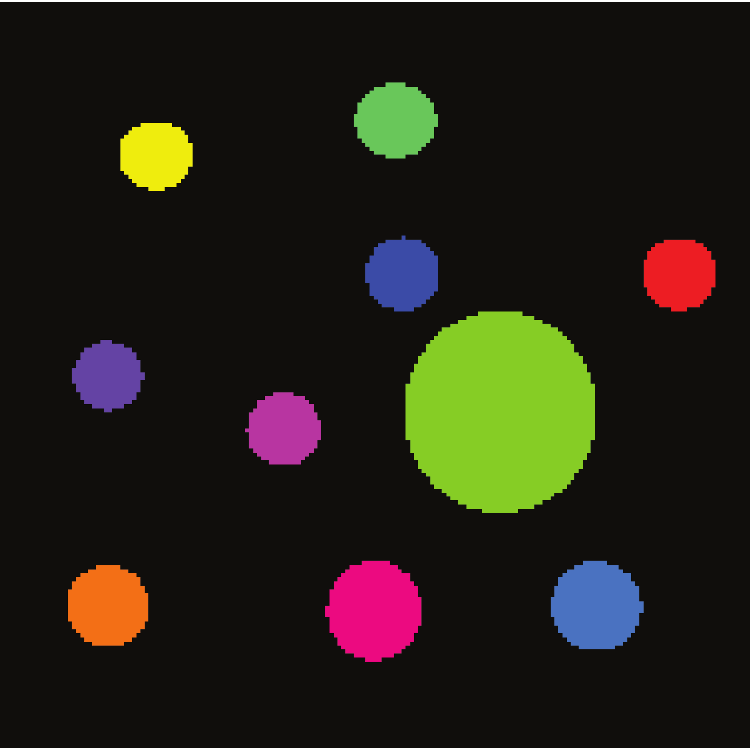}\\
\end{tabular}
\end{center}
\caption{[Seed sparsity/distance]  (a) The given image. (b) and (c) are two different seed images.  If there are objects without any seeds inside, CODI misses counting these objects as expected as in (b1) and (b2).   With multiple seeds within all objects, both method counts correctly as in (c1) and (c2).  This illustrates the importance of having the distance between seeds to be smaller than the minimum distance between objects.}
\label{F: seed distance test}
\end{figure}

Since CODI counts the diffused index, it is helpful to have the indexes to be as separated as possible.   We propose the following simple rules on the distance between seeds and size of seeds, for better performance of CODI: 
\begin{enumerate}
    \item The distance between (the boundary of) seeds should be smaller than the minimum distance between the boundary of objects, that every object has at least one seed inside.
    \item The size of seed itself should be small compared to the minimum size of objects, that no two objects are covered by only one unique seed.  In addition, we found that the convergence is faster with smaller seed size.  \end{enumerate}

Figure \ref{F: seed distance test} shows  the effect of counting results based on different sparsity of seeds. (a) is a synthetic image of size $126 \times 127$, and  experiments are preformed based on two seed images (b) and (c), with two different  distance between seed boundaries  $d=38$ and $d=6$ respectively.  The size of seeds are both $2 \times 2$.  (b1)-(b2) and (c1)-(c2) provide the counting results form CODI-S and CODI-M respectively.   If there are objects without any seeds inside, CODI  misses counting these objects as expected, shown in (b1) and (b2).   As a comparison, both proposed methods count 10 objects in (c1) and (c2), if there are multiple seeds within all objects to be counted.   This  illustrates the importance of Rule 1 that it is important to have the distance between seeds to be smaller than the minimum distance between objects. 

As for the size of the seeds, if multiple objects have only one large seed shared, it will be identified as one object in CODI.  Rule 2 suggest each seeds to be small compared to the minimum size of the objects to ensure separation between different objects.  We further experiment with the size of seed in Figure \ref{F: seed test IV}.    It shows that even if the size of the seed is smaller than the size of the objects to be counted, it is better to have smaller seeds for faster diffusion.   We experiment on a binary image of size $281 \times 87$ with 6 hexagons using seeds  of size $20 \times 20$ and $2 \times 2$ respectively.   The distance between the boundaries of big seeds and small seeds are both 10 pixels, which is smaller than the minimum distance between any two hexagons to be counted.  The first and second rows show CODI-S and CODI-M using big seeds, while the third and fourth rows show CODI-S and CODI-M using small seeds respectively.   
With smaller seeds, less than 40 iteration for both CODI-S and CODI-M give correct counting of 6, while for bigger seeds (top two rows) takes 300 to 400 iteration to find the correct counting.   To demonstrate the relation between seed size and convergence, we set the objective function in (\ref{usub-orig}) at $n$th iteration to be $E_n$, and consider 
\begin{equation}
    R_n= \left| \frac{E_n - E_{n-1} }{E_{n-1}}\right|
    \label{e:R_n}
\end{equation} 
for convergence measure.  If $R_n$ is small, it means the diffusion is converging.  For each experiments, the clustering results are shown in 3 stages: first column: $R_n = 0.09$, second column: $R_n = 0.05$, and third column: $R_n = 0.01$. 
In Figure \ref{F: seed test IV} the third row, after 32 iterations, $R_n = 0.09$ in CODI-S, 6 objects are found.  After another 9 iterations, $R_n$ decreased to 0.05 and 6 objects are found by CODI-S again.   This shows that using relatively small seeds results in good counting results with $R_n = 0.05-0.09$.   For bigger seeds $R_n = 0.01$ is needed, since changing given seed values to become a diffused index for each object takes longer. 

\begin{figure}
\begin{center}
\begin{tabular}{c}
\includegraphics[width = 0.8 \textwidth ]{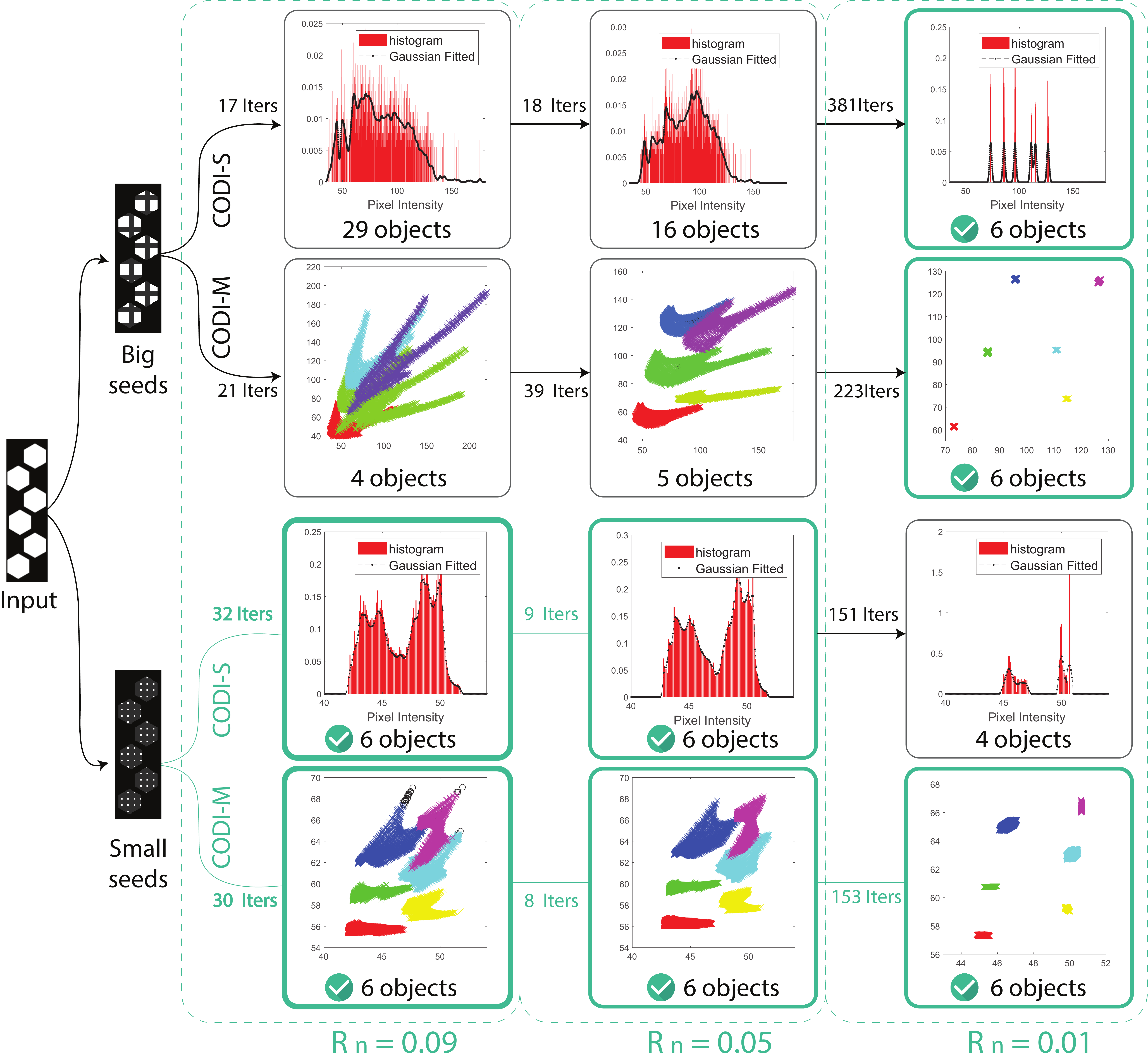} 
\end{tabular}
\end{center}
\caption{[Seed size v.s. Convergence]  From one given image, two different sizes of seeds are used while keeping the distance between the seeds to be the same (smaller than the minimum distance between the objects).    For smaller seeds in third and forth row, CODI gives good counting results with $R_n = 0.05-0.09$.  For bigger seeds $R_n = 0.01$ is needed, since changing given seed values to become a diffused index for each object takes. }
\label{F: seed test IV}
\end{figure}

Given an image of high resolution, reducing the size of image while keeping the boundary information  can significantly reduce the cpu time.  Figure \ref{F: downsample test} shows reduction of size vs the counting result.  (a) is the original image of size $1000 \times 1097$.  With manual counting, there are about $[203,213]$ number of cells, depending on how very small objects are counted. 
This image is reduced to 7 different levels of quality as shown in (b), level of reduction ranging from 76\% to 88\% reduced from the original image.  For example, after 88\% reduction, the given has been reduced to size $140 \times 154$.  For each of the seven reduced image in (b), we perform CODI-S for once and CODI-M for 50 times.  In (c), blue dots are CODI-S, blue bars are CODI-M, and the yellow color bar is a range of correct counting.   Red bar graphs show the CPU time in seconds for CODI-S and CODI-M showing the clear reduction on cpu time.   The $x$-axis shows the downsampling rate.   Notice while the counting results are near the correct range, cpu time clearly reduces with downsampling.  
\begin{figure}
\begin{center}
\begin{tabular}{ccc}
(a) Given image & (b) Visualization & (c) Downsample test (CODI-S) \\
\includegraphics[align=c,height= 1.2in]{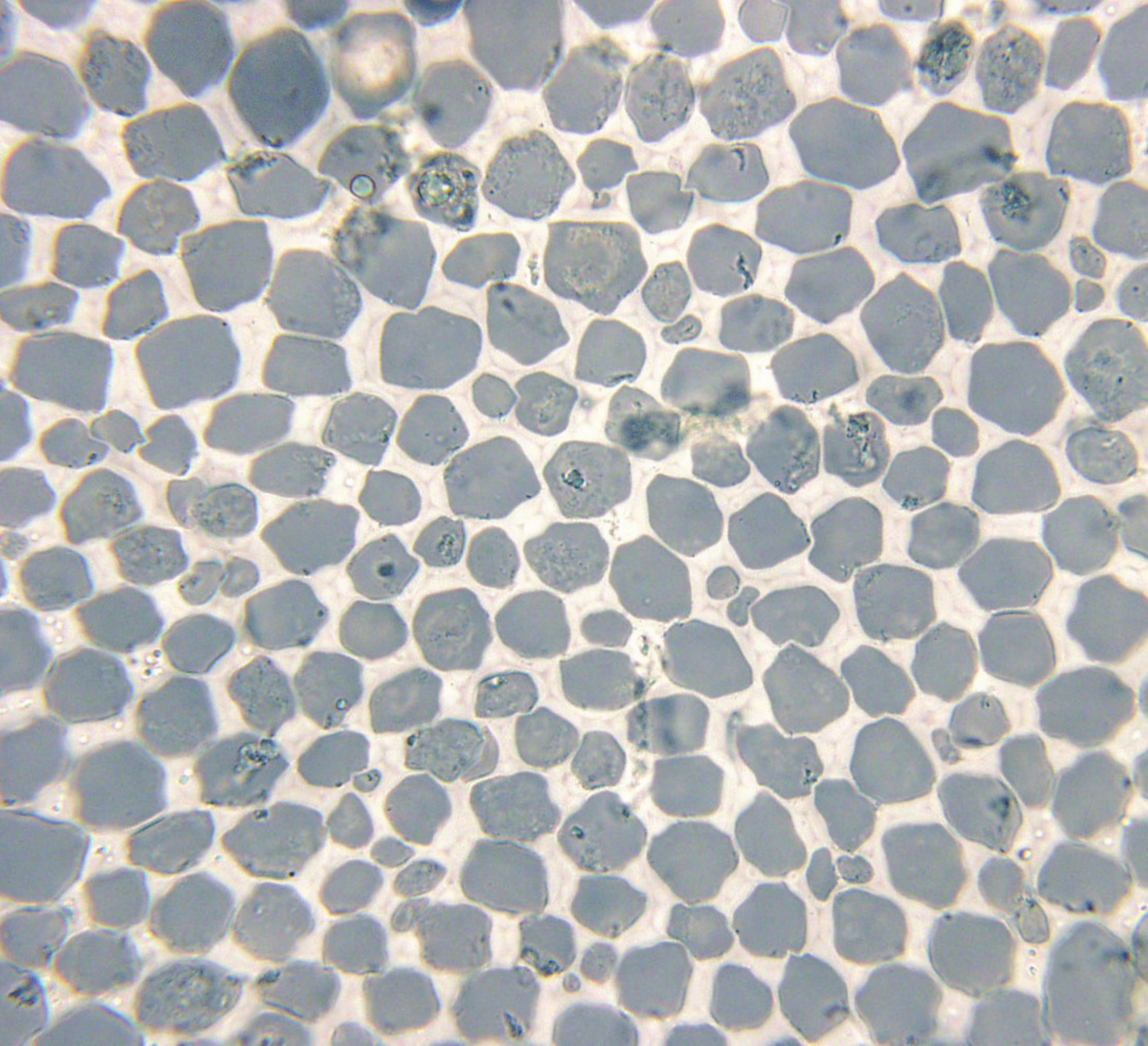} 
& \includegraphics[align=c,height= 1.2in]{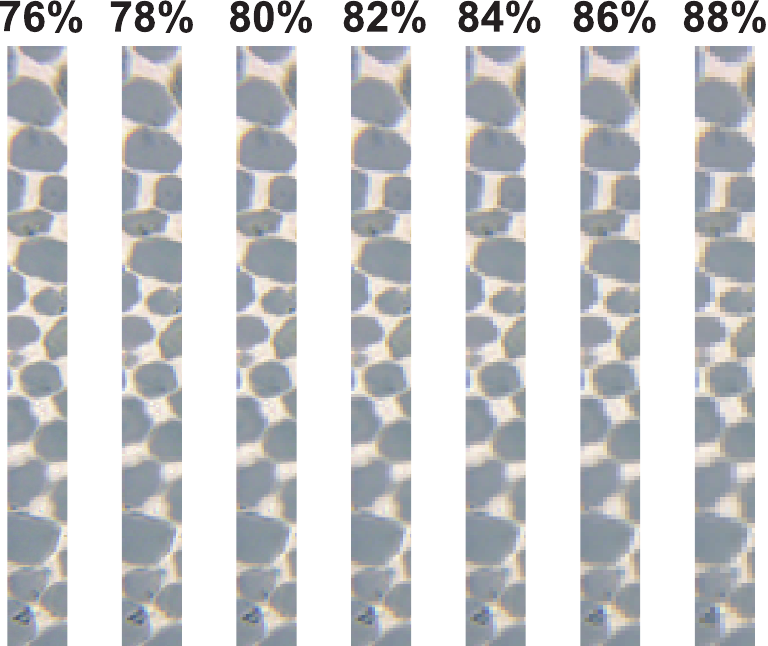} 
&\includegraphics[align=c,height= 2.2in]{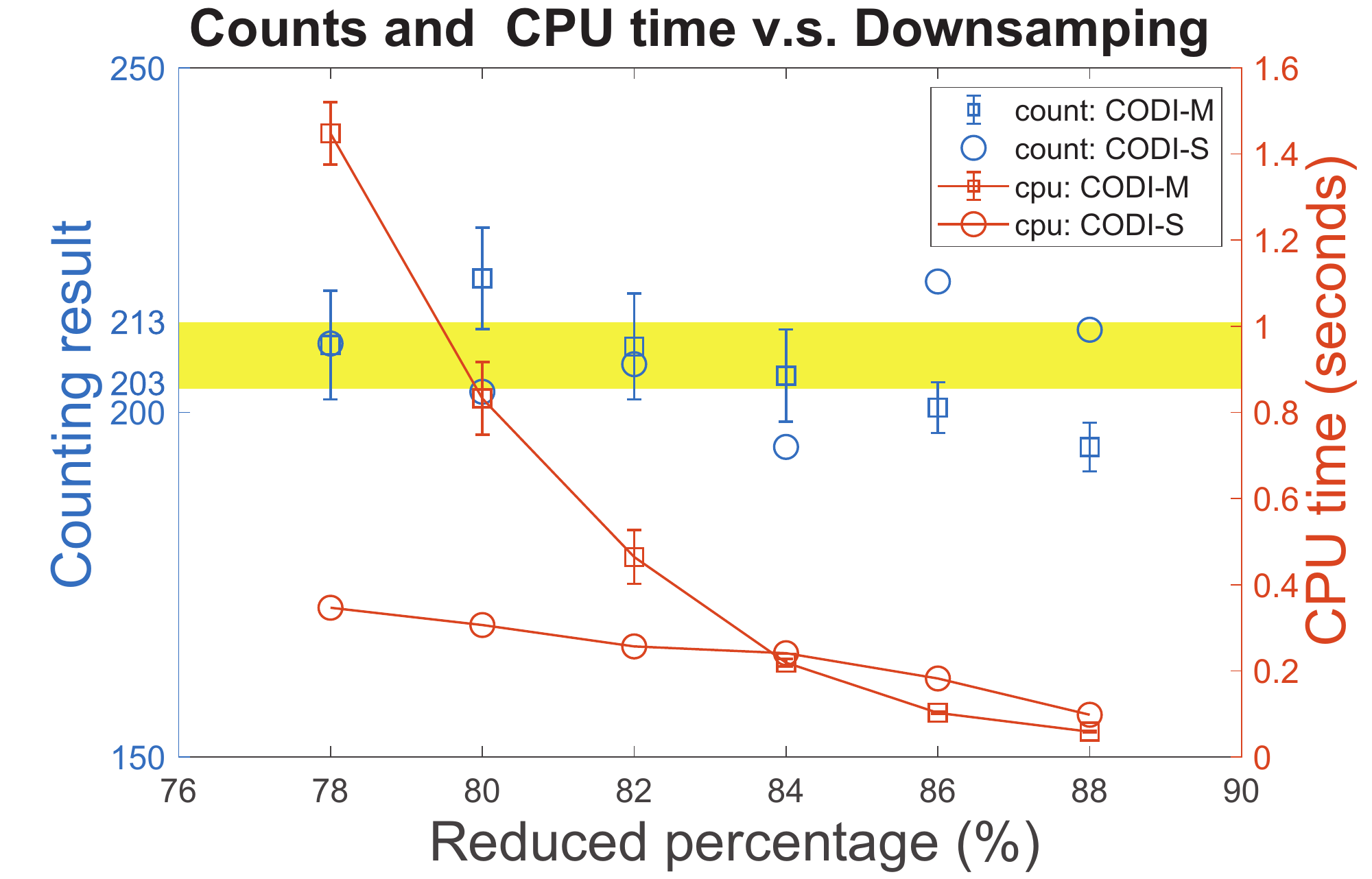}
\end{tabular}
\end{center}
\caption{[Downsample and cpu time]  (a) The given image of $1000 \times 1097$ with manual counting in the range of $[203,213]$ which is shown as the highlighted region in (c).  (b) a visualization of seven different downsampled image, ranging from 76\% to 88\%.  (c)  The blue circles represents CODI-S, the red circles the cpu time.   The blue error bars denotes the mean and standard deviation of 50 experiments of CODI-M, and the red error bars those of cpu times.  Notice while the counting results are near the correct range, cpu time clearly reduces with downsampling.  }
\label{F: downsample test}
\end{figure}

As for the stability of parameters for CODI-S and CODI-M, we consider the parameter space in terms of  $r$ and $\sigma$ for CODI-S, and in terms of MinPts and $\epsilon$ for CODI-M.   We test with Figure \ref{F: industry}(a) image.  
 In Figure \ref{converging test1} and \ref{converging test 2}, the most yellow region denotes the parameter set that produce 100\% correct counting results.  We present the parameter graph as the diffusion algorithm convergence.  We consider $R_n$ in (\ref{e:R_n}) for convergence measure.  If $R_n$ is small, it means the diffusion is converging. 
In Figure \ref{converging test1}, we show five experiments (a)-(e) where $R_n \in \{ 50\%,40\%,30\%,20\%,10\% \}$.  The ground truth of counting result is 30. 
When $R_n = 10\%$, there are larger green region in the parameter space that produces high accuracy.  These graphs also present the relation between the smoothing of histogram, the number of iteration and the counting results.   In general, there are large regions with yellow which represent good counting results. This result is consistent with Figure \ref{F: Open boundaries MI-40iter} where smaller number of iteration is favorable for CODI-S. 
In Figure \ref{converging test 2}, the same experiments are conducted for CODI-M where we have $R_n$ set to be $15\%, 10\%, 5\%$ in (a)-(c).  The red marks denotes two examples of the optimal parameters  we recommend for the experiment in similar cases.  When $R_n\leq 10\% $, the counting result won't be affected by small perturbation of the parameters.  As in the case of open boundary, CODI-M with longer iteration give stable results. 

\begin{figure}
\begin{center}
\begin{tabular}{c} 
\includegraphics[height = 1.8in]{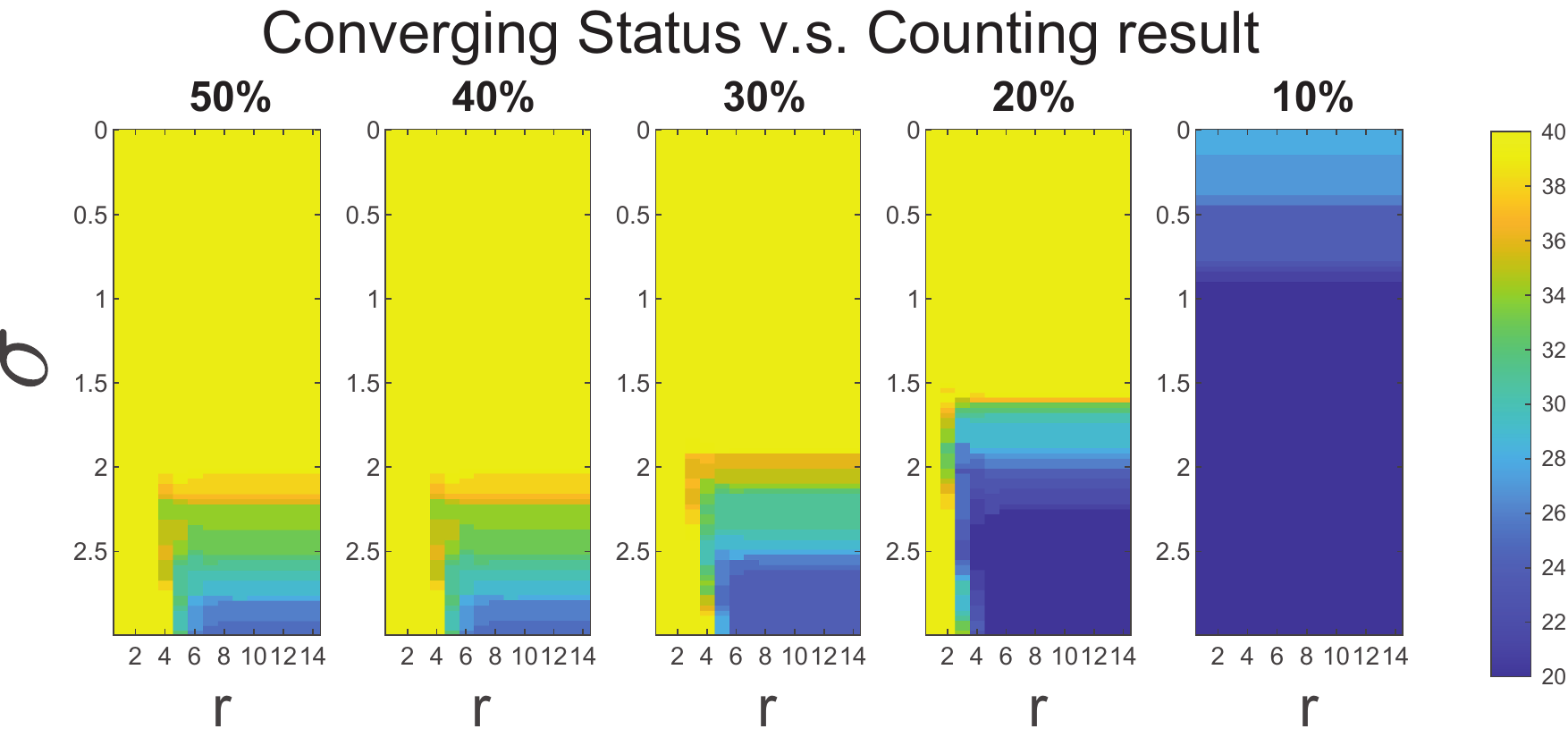} 
\end{tabular}
\end{center}
\caption{[CODI-S parameter space] Visualization of countingresult in the parameter space $(r,\sigma) \in [2,15] \times [0.01, 3]\cup [1.6,3] $ based on different diffusion stage.   
(a)-(e) shows when $R_n = 50\%,40\%,30\%,20\%,10\%$.  
The ground truth of counting result is 30, where the more yellow the color is more accurate the result.  This result is consistent with Figure \ref{F: Open boundaries MI-40iter} where smaller number of iteration is favorable for CODI-S. 
  }
\label{converging test1}
\end{figure}

\begin{figure}
\begin{center}
\begin{tabular}{cccc} 
(a)$R_n=15\%$  &  (b)$R_n=10\%$  & (c)$R_n=5\%$  \\\includegraphics[height = 1in]{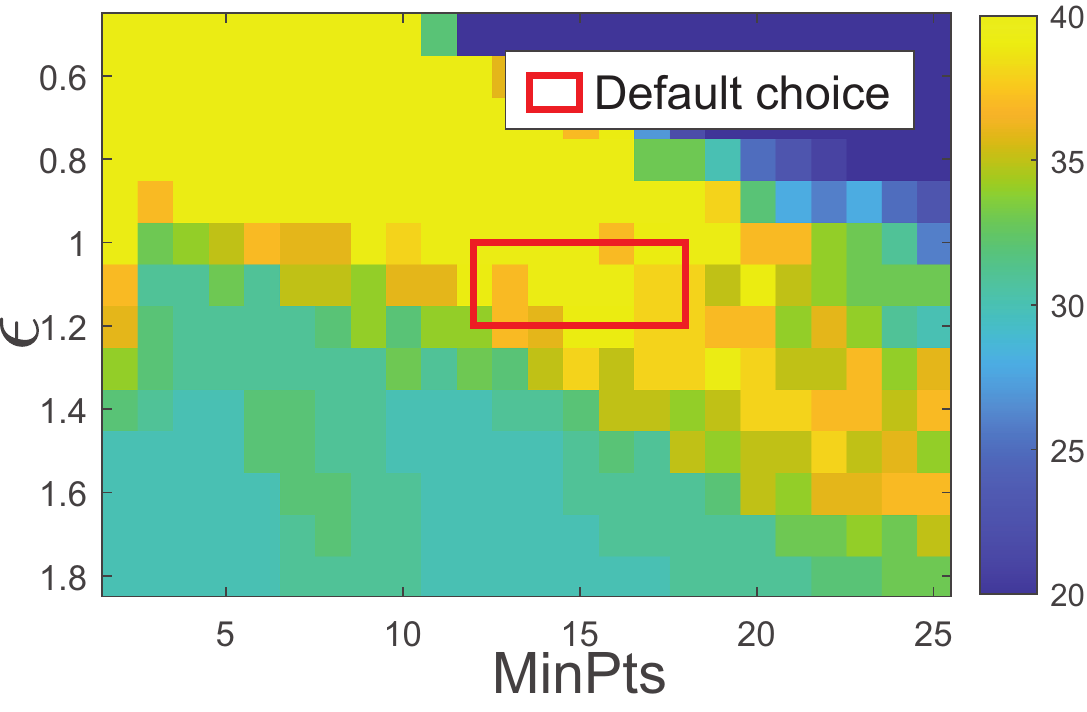}& 
\includegraphics[height = 1in]{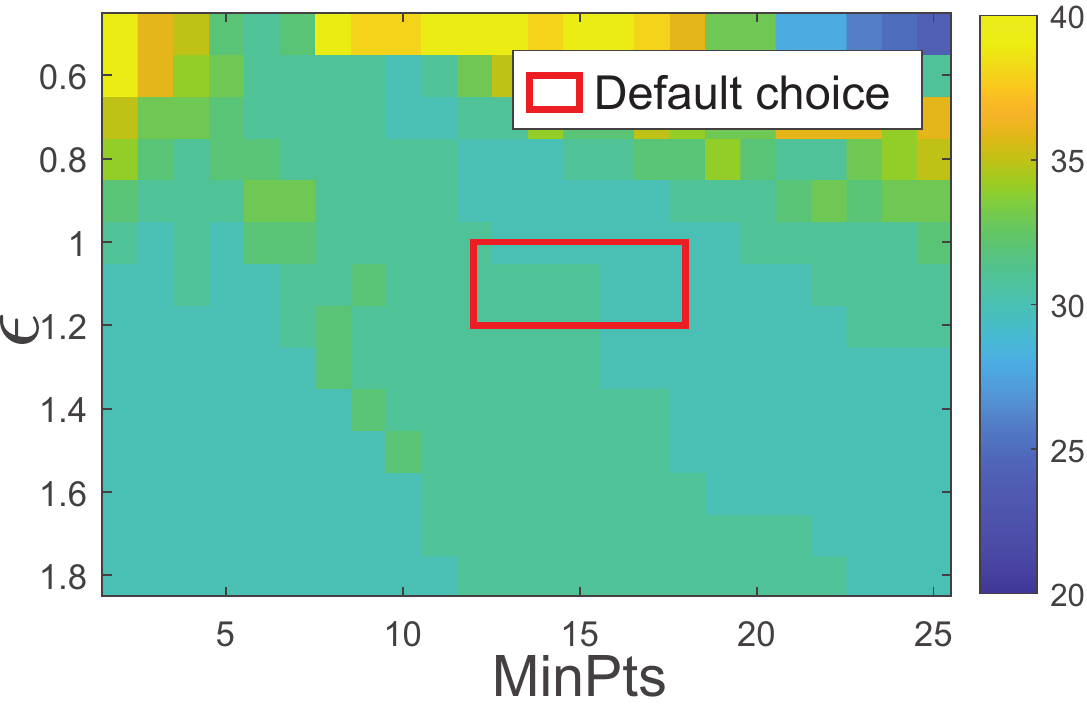}&
\includegraphics[height = 1in]{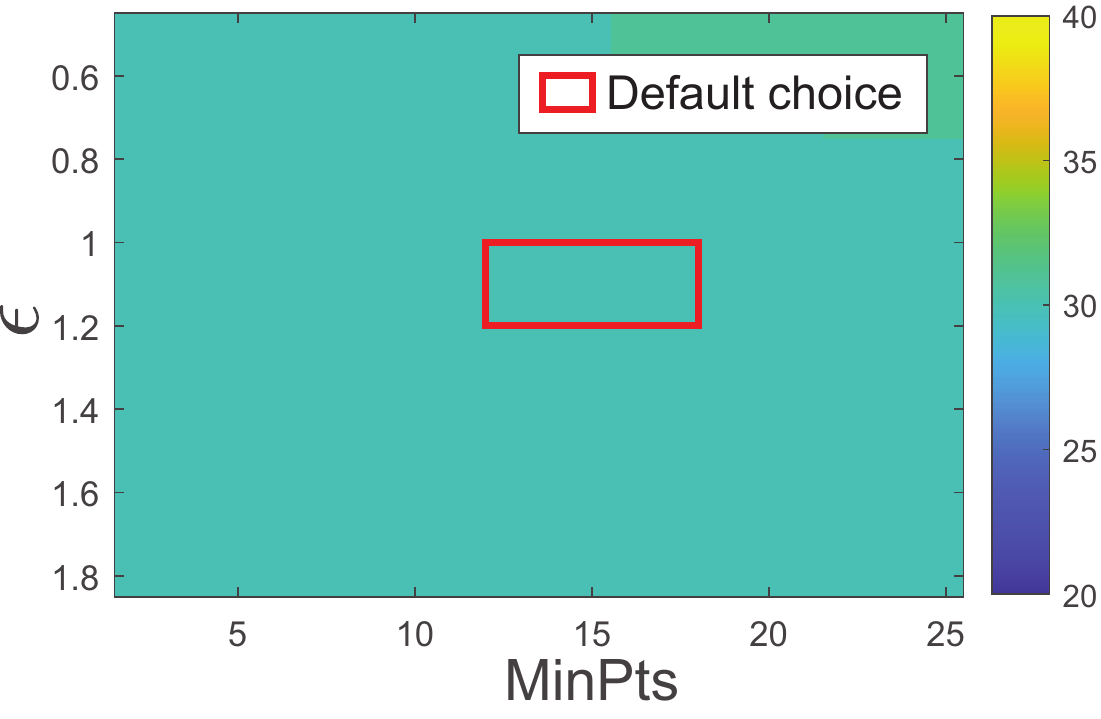} \\
\end{tabular}
\end{center}
\caption{[CODI-M parameter space]  Visualization of counting result in the parameter space $(\text{MinPts}, \epsilon) \in [2,25] \times [0.5, 1.8] $ based on different diffusion stage. The ground truth in this example is 30, i.e., the green area represents good result.  (a)-(c)  shows when $R_n$ to be $15\%, 10\%$, and $5\%$.  The red marks denotes the  parameters we recommend for similar cases.  With enough iterations, the counting result of CODI-M is not affected by a small perturbation of the parameters. }
\label{converging test 2}
\end{figure}

\section{Concluding Remarks} \label{sec:concluding}

We proposed Counting Object by Diffused Index with scalar and multi-dimensional seeds.  This method is  diffusion-based, geometry-free and learning-free method.
The diffusion phase is based on an edge-weighted harmonic optimization model, using the $g$ weight function or mask image and the seed image.
We proposed an efficient algorithm, called Diffusion Algorithm, to obtain the diffused image. 
CODI-S  is based on Gaussian fitted curve to the histogram data of the diffused image, that 
the number of local maximum of this curve gives the number of objects in the image.  For CODI-S, even with a small number of diffusion iteration, there is a large region with 100\% accurate counting in the parameter space.  
CODI-M utilizes more flexible 4-dimensional seeds which can help to distinguish objects better.  Typically, a longer iteration compared to CODI-S helps accurate and stable counting for CODI-M.   CODI-M can also find each object location in the given image for object identification.  This method can further separately count different size object by clustering the set $S$ of cluster size.   
In the numerical section, we experimented the proposed methods on various images  including cells, plants, fruits, and concert crowd.  The results confirm that the proposed methods are geometry-free,  and are able to provide good counts in various cases in a very short amount of cpu time.   We compared with different existing methods, many of which only works for particular types of images considered in their paper.  We also compared with methods which   require learning and training process.  The proposed methods show comparable results in terms of accuracy.

\section*{Acknowledgement}
Authors want to acknowledge Mr. Sayem Hoque who worked on this project a few years ago as an undergraduate student at Georgia Institute of Technology.   We thank his contribution and  valuable discussions. 

\section*{Appendix}
\appendix
\section{Proof of Theorem \ref{thm 2}}
\label{proof of thm}
\begin{proof}
(a). The first order optimality conditions of $U$ subproblem is given by
\[
-2\nabla\cdot(G_0\nabla U^{(k+1)})+2\nabla\cdot ((G_0-g_0)\nabla U^{(k)})+
\theta(U^{(k+1)}-U^{(k)})
+\mu (U^{(k+1)}-V^{(k)})-\lambda^{(k)}=0.
\] 
We express this equation in terms of error to get
\[
-2 \nabla \cdot G_0 \nabla U^{(k+1)}_e+
2\nabla\cdot ((G_0-g_0)\nabla U_e^{(k)})+
\theta(U_e^{(k+1)}-U_e^{(k)})
+\mu (U_e^{(k+1)}-V_e^{(k)})-\lambda_e^{(k)}=0.
\]
We multiply this equation by $U^{(k+1)}_e$, this gives
\begin{eqnarray*}
&2G_0 \|\nabla U^{(k+1)}_e\|^2
-2(G_0-g_0)\langle \nabla U_e^{(k)}, \nabla U_e^{(k+1)}\rangle+
\theta\langle U_e^{(k+1)}-U_e^{(k)}, U_e^{(k+1)}\rangle&\\
&+\mu \|U_e^{(k+1)}\|^2
-\mu \langle V_e^{(k)}, U_e^{(k+1)}\rangle
-\langle \lambda_e^{(k)},U_e^{(k+1)}\rangle=0.&
\end{eqnarray*}
For any vectors $x,y\in\mathbb R^n$, the following inequalities hold
\begin{eqnarray}\label{identity}
2 \langle x,x-y \rangle &=& \|x\|^2- \|y\|^2 + \|x-y\|^2
\\
2 \langle x,y\rangle &=&  \|x\|^2+ \|y\|^2- \|x-y\|^2.
\end{eqnarray}
We exploit these identities into the account to get 
\begin{eqnarray}\label{U-equ}
&(G_0-g)  \|\nabla U^{(k+1)}_e\|^2
+(\mu+\frac{\theta}{2})\|U_e^{(k+1)}\|^2
+\frac{\theta}{2}\|U_e^{(k+1)}-U_e^{(k)}\|^2
+(G_0-g_0)\|\nabla U_e^{(k+1)}-\nabla U_e^{(k)}\|^2&
\nonumber\\[.1in]
&=(G_0-g)  \|\nabla U^{(k)}_e\|^2
+\frac{\theta}{2} \|U_e^{(k)}\|^2
+ \mu \langle V_e^{(k)}, U_e^{(k+1)}\rangle
+\langle \lambda_e^{(k)},U_e^{(k+1)}\rangle&
\end{eqnarray}

The optimality conditions of $V$ subproblem is given by
\[
\Big\langle \eta_{D} (V^{(k+1)}-U_0)+\mu (V^{(k+1)} - U^{(k+1)} ) +\lambda^{k}, V-V^{(k+1)}
\Big\rangle \ge 0
\]
We set $V=V^*$ in the latter inequality and $V=V^{(k+1)}$ in the middle inequality in (\ref{kkt}),
and add the results, then we express it in terms of error to obtain
\begin{eqnarray}\label{V-equ}
(\eta_D+\mu)\|V_e^{(k+1)}\|^2 
\le \mu \langle U^{(k+1)}_e, V_e^{(k+1)}\rangle 
-\langle \lambda_e^{(k)}, V_e^{(k+1)}\rangle. 
\end{eqnarray}
By the algorithm we have 
\[
\lambda^{(k+1)}=\lambda^{(k)}+\mu (V^{(k+1)}-U^{(k+1)}),
\]
which in terms of error it is given by 
\[
\lambda_e^{(k+1)}=\lambda_e^{(k)}+\mu (V_e^{(k+1)}-U_e^{(k+1)}).
\]
We multiply this equation by $\lambda_e^{(k)}$ and use the identity (\ref{identity}) to get
\begin{eqnarray}\label{lambda-equ}
\frac{1}{\mu}\|\lambda_e^{(k+1)}\|^2
+\mu \|V_e^{(k+1)}-U_e^{(k+1)}\|^2
= \frac{1}{\mu}\|\lambda_e^{(k)}\|^2
+2\langle V_e^{(k+1)}, \lambda_e^{(k)}\rangle
-2 \langle U_e^{(k+1)}, \lambda_e^{(k)}\rangle.
\end{eqnarray}

We add (\ref{U-equ}),  (\ref{V-equ}), and  (\ref{lambda-equ}) to get
\begin{eqnarray}\label{sum}
&
(\mu+\frac{\theta}{2})\|U_e^{(k+1)}\|^2
+(\eta_D+\mu)\|V_e^{(k+1)}\|^2 
+\frac{1}{\mu}\|\lambda_e^{(k+1)}\|^2
+(G_0-g)  \|\nabla U^{(k+1)}_e\|^2
\nonumber&\\[.1in]&
+\frac{\theta}{2}\|U_e^{(k+1)}-U_e^{(k)}\|^2
+(G_0-g_0)\|\nabla U_e^{(k+1)}-\nabla U_e^{(k)}\|^2
+\mu \|V_e^{(k+1)}-U_e^{(k+1)}\|^2
\nonumber&\\[.1in]&\le
\frac{\theta}{2} \|U_e^{(k)}\|^2
+\frac{1}{\mu}\|\lambda_e^{(k)}\|^2
+ \mu \langle V_e^{(k)}, U_e^{(k+1)}\rangle
+\mu \langle U^{(k+1)}_e, V_e^{(k+1)}\rangle 
+\langle \lambda_e^{(k)}, V_e^{(k+1)}-U_e^{(k+1)}\rangle. 
\end{eqnarray}

By  (\ref{identity}) we  then have
\begin{eqnarray*}
\mu \langle V_e^{(k)}, U^{(k+1)}_e \rangle &=&
\frac{\mu}{2}\|V_e^{(k)}\|^2 + \frac{\mu}{2}\|U_e^{(k+1)}\|^2 -\frac{\mu}{2}\|V_e^{(k)}-U_e^{(k+1)}\|^2, 
\\[.1in]
\mu \langle V_e^{(k+1)}, U^{(k+1)}_e \rangle &=&
\frac{\mu}{2}\|V_e^{(k+1)}\|^2 + \frac{\mu}{2}\|U_e^{(k+1)}\|^2 -\frac{\mu}{2}\|V_e^{(k+1)}-U_e^{(k+1)}\|^2, 
\\[.1in]
\langle \lambda_e^{(k)}, V_e^{(k+1)}-U_e^{(k+1)}\rangle 
&=&
-\frac{1}{2\mu} \| \lambda_e^{(k)} \|^2
-\frac{\mu}{2} \| V_e^{(k+1)}-U_e^{(k+1)} \|^2+
\frac{1}{2\mu} \| \lambda_e^{(k+1)}\|^2.
\end{eqnarray*}
We replace these equations in the right hand side of (\ref{sum}) to get
\begin{eqnarray}\label{sum-rep}
&
\frac{\theta}{2}\|U_e^{(k+1)}\|^2
+(\eta_D+\frac{\mu}{2})\|V_e^{(k+1)}\|^2 
+\frac{1}{2\mu}\|\lambda_e^{(k+1)}\|^2
+(G_0-g)  \|\nabla U^{(k+1)}_e\|^2
\nonumber&\\[.1in]&
+\frac{\theta}{2}\|U_e^{(k+1)}-U_e^{(k)}\|^2
+(G_0-g_0)\|\nabla U_e^{(k+1)}-\nabla U_e^{(k)}\|^2
+2 \mu \|V_e^{(k+1)}-U_e^{(k+1)}\|^2
\nonumber&\\[.1in]&
+\frac 12 \mu \|V_e^{(k)}-U_e^{(k+1)}\|^2
\le
\frac{\theta}{2} \|U_e^{(k)}\|^2
+\frac{\mu}{2}\|V_e^{(k)}\|^2 
+\frac{1}{2\mu}\|\lambda_e^{(k)}\|^2.
\end{eqnarray}
We drop some positive terms on the left hand sides to get 

\[
E_{k+1}=\frac{\theta}{2}\|U_e^{(k+1)}\|^2
+\frac{\mu}{2}\|V_e^{(k+1)}\|^2 
+\frac{1}{2\mu}\|\lambda_e^{(k+1)}\|^2
\le
\frac{\theta}{2} \|U_e^{(k)}\|^2
+\frac{\mu}{2}\|V_e^{(k)}\|^2 
+\frac{1}{2\mu}\|\lambda_e^{(k)}\|^2=E_k.
\]
This shows that $\{E_k\}_{k\in\mathbb N}$
is a monotonically nonincreasing sequence.

Proof of (b). By Part 1, we have 
\[
\frac{\theta}{2}\|U_e^{(k+1)}-U_e^{(k)}\|^2
+2 \mu \|V_e^{(k+1)}-U_e^{(k+1)}\|^2
+\frac 12 \mu \|V_e^{(k)}-U_e^{(k+1)}\|^2
\le
E_k-E_{k+1}
\]
We sum this inequality from $k=1$ to any positive integer $K>1$ to obtain
\[
\frac{\theta}{2}\sum_{k=1}^K
\|U_e^{(k+1)}-U_e^{(k)}\|^2
+2  \mu \sum_{k=1}^K\frac{\theta}{2} \|V_e^{(k+1)}-U_e^{(k+1)}\|^2
+\frac {\mu}{2} \sum_{k=1}^K\frac{\theta}{2} \|V_e^{(k)}-U_e^{(k+1)}\|^2
\le
E_1-E_{K+1} \le E_1
\]
The latter is due to the fact that $\{E_k\}$ is decreasing.
We let $K$ approach to infinity, 
\[
\frac{\theta}{2}\sum_{k=1}^{\infty}
\|U_e^{(k+1)}-U_e^{(k)}\|^2
+2  \mu \sum_{k=1}^{\infty}\frac{\theta}{2} \|V_e^{(k+1)}-U_e^{(k+1)}\|^2
+\frac {\mu}{2}\sum_{k=1}^{\infty}\frac{\theta}{2} \|V_e^{(k)}-U_e^{(k+1)}\|^2
 \le E_1<\infty
\]
Thus, we have
\begin{eqnarray*}
\lim_{k\to\infty} \|U_e^{(k+1)}-U_e^{(k)}\|=0,\quad
\lim_{k\to\infty} \|V_e^{(k+1)}-U_e^{(k+1)}\|=0,\quad
\lim_{k\to\infty} \|V_e^{(k)}-U_e^{(k+1)}\|=0.
\end{eqnarray*}
Since $U^*=V^*$, then these results are equivalent to 
\begin{eqnarray}\label{inf}
\lim_{k\to\infty} \|U^{(k+1)}-U^{(k)}\|=0,\quad
\lim_{k\to\infty} \|V^{(k+1)}-U^{(k+1)}\|=0,\quad
\lim_{k\to\infty} \|V^{(k)}-U^{(k+1)}\|=0.
\end{eqnarray}
Moreover, by the triangle inequality we have
\[
\|V^{(k+1)}-V^{(k)}\|\le  \|V^{(k+1)}-U^{(k+1)}\|+ \|V^{(k)}-U^{(k+1)}\|.
\]
By (\ref{inf}) we then have 
\[
\lim_{k\to\infty}\|V^{(k+1)}-V^{(k)}\|=0.
\]
Moreover, as $\lambda^{(k+1)}-\lambda^{(k)}=\mu(V^{(k+1)}-U^{(k+1)})$, by (\ref{inf})
we also have 
\[
\lim_{k\to\infty} \|\lambda^{(k+1)}-\lambda^{(k)}\|=\lim_{k\to\infty}  \mu \|V^{(k+1)}-U^{(k+1)}\| =0.
\]

Proof of (c). 
By part (a), the sequence $\{E_k\}$ is monotonically decreasing and bounded below by zero,
hence it approaches a limit. This follows that the sequence 
$\{(U^{(k)}, V^{(k)}, \lambda^{(k)})\}_{k\in\mathbb N}$ is uniformly bounded.
Thus a convergence subsequence $(U^{(k_l)}, V^{(k_l)}, \lambda^{(k_l)})$, $l\ge 1$ exists,
 that approaches a limit, say $(U^{\infty}, V^{\infty}, \lambda^{\infty})$.  
 For the subsequence $\{(U^{(k_l)}, V^{(k_l)}, \lambda^{(k_l)})\}_{l\in\mathbb N}$ 
it holds
\[
-2\nabla\cdot(G_0(\nabla U^{(k_l+1)}-\nabla U^{(k_l)}-2\nabla\cdot g_0\nabla U^{(k_l)})+
\theta(U^{(k_l+1)}-U^{(k_l)})
+\mu (U^{(k_l+1)}-V^{(k_l)})-\lambda^{(k_l)}=0.
\] 
By part (b), 
$\lim_{l\to\infty} \|U^{(k_l+1)}-U^{(k_l)}\|=
\lim_{l\to\infty} \|U^{(k_l+1)}-U^{(k_l)}\|=
\lim_{l\to\infty} \|U^{(k_l+1)}-V^{(k_l)}\| =0.$
Hence by letting $l$ approach to infinity we obtain 
\begin{eqnarray}\label{s1}
-2\nabla\cdot (g_0\nabla U^{\infty})-\lambda^{\infty}=0.
\end{eqnarray}
The subsequence $\{(U^{(k_l)}, V^{(k_l)}, \lambda^{(k_l)})\}_{l\in\mathbb N}$ 
satisfies in the optimality conditions of $V$ subproblem
\[
\langle \eta_{D} (V^{(k_l+1)}-U_0)+\mu (V^{(k_l+1)} - U^{(k_l+1)} ) +\lambda^{k_l}, V-V^{(k_l+1)}
\rangle \ge 0
\]
for all $V$. By part (b) again, as $l$ approaches to infinity we have 
$\lim_{l\to\infty} \|V^{(k_l+1)} - U^{(k_l+1)} \|=0,$
hence  we obtain
\begin{eqnarray}\label{s2}
\langle \eta_{D} (V^{\infty}-U_0)+\lambda^{\infty}, V-V^{\infty}\rangle \ge 0.
\end{eqnarray}
By part (b) again, $\lim_{l\to\infty} \|\lambda^{(k+1)}-\lambda^{(k)}\|=0$. 
By the fact that $\lambda^{(k+1)}-\lambda^{(k)} =\mu (V^{(k_l+1)} - U^{(k_l+1)})$, we obtain
$U^{\infty}-V^{\infty}=0$. By this, (\ref{s1}), and (\ref{s2}), any limit point is a stationary point. 

Proof of (d). 
The proof of the theorem started with an arbitrary extreme point $(U^*,V^*,\lambda^*)$.
Let us consider the specific extreme point $(U^{\infty}, V^{\infty}, \lambda^{\infty})$ that is the 
limit of convergent subsequence $(U^{(k_l)}, V^{(k_l)}, \lambda^{(k_l)})$, $l\ge 1$.
Since the subsequence $(U_e^{(k_l)}, V_e^{(k_l)}, \lambda_e^{(k_l)})$, $l\ge 1$ converges to $0$,
it follows that $E_{k_l}$ tends to zero. Since $E_{k}$ is a monotone decreasing sequence 
it follows that $\{E_k\}_{k\in\mathbb N}$ tends to zero. We conclude that the whole sequence
$(U^{(k)}, V^{(k)}, \lambda^{(k)})$ converges to $(U^*,V^*,\lambda^*)$.

\end{proof}

\bibliographystyle{plain}
\bibliography{counting_ref_apr24}

\end{document}